\documentclass[10pt,twocolumn,letterpaper]{article}

\usepackage[pagenumbers]{cvpr} 

\usepackage{mathtools}

\definecolor{cvprblue}{rgb}{0.21,0.49,0.74}

\usepackage{xcolor,amsmath}
\usepackage[linesnumbered,ruled,vlined]{algorithm2e}

\usepackage{algpseudocode}
\usepackage{placeins}
\usepackage{graphicx}

\usepackage{rotating} 
\usepackage{lipsum}   
\usepackage{booktabs} 
\DontPrintSemicolon
\usepackage{svg}

\SetKwComment{Comment}{\color{green!50!black}// }{}

\newcommand{\assign}{\leftarrow}

\SetKwProg{Function}{function}{}{}

\RestyleAlgo{ruled}
\usepackage[utf8]{inputenc} 
\usepackage[T1]{fontenc}    
\usepackage{url}            
\usepackage{booktabs}       
\usepackage{amsfonts}       
\usepackage{nicefrac}       
\usepackage{microtype}      

\usepackage[dvipsnames]{xcolor}

\usepackage{amssymb}
\usepackage{mathtools}
\usepackage{amsthm}
\usepackage{dsfont}
\usepackage{listings}
\usepackage[accsupp]{axessibility}

\usepackage[pagebackref,breaklinks,colorlinks,allcolors=blue]{hyperref}
\theoremstyle{plain}
\newtheorem{theorem}{Theorem}[section]

\newtheorem{lemma}[theorem]{Lemma}

\theoremstyle{definition}

\newtheorem{assumption}[theorem]{Assumption}

\theoremstyle{remark}

\def\BE{\begin{equation}}
\def\EE{\end{equation}}
\def\BEA{\begin{eqnarray}}
\def\EEA{\end{eqnarray}}

\newcommand{\R}{\mathbb{R}}

\newcommand{\E}{\mathbb{E}}

\DeclareMathOperator*{\argmin}{arg\,min}
\newcommand{\X}{\mathcal{X}}
\newcommand{\Y}{\mathcal{Y}}

\newcommand{\norm}[1]{\left\lVert#1\right\rVert}
\newcommand{\sqrbr}[1]{\left[#1\right]}

\newcommand{\parent}[1]{\left(#1\right)}

\newenvironment{proofsketch}[1][\proofname]{%
 \proof[Proof Sketch]%
}{\endproof}
\usepackage{chngcntr}
\newcommand{\layoutmode}{2} 

\newcommand{\figwidth}{%
  \ifnum\layoutmode=3
    0.32
  \else
    0.48
  \fi
}

\newcommand{\ifthree}[1]{%
  \ifnum\layoutmode=3
    #1%
  \fi
}

\def\paperID{14339} 
\def\confName{CVPR}
\def\confYear{2026}

\title{Is the Modality Gap a Bug or a Feature? A Robustness Perspective}

\author{%
  \begin{tabular}[t]{c}
    Rhea Chowers \\
    Hebrew University\\
    Jerusalem, Israel \\
    \texttt{rhea.chowers@mail.huji.ac.il}\\\\
    Udi Barzelay \\
    IBM Research\\
    Haifa, Israel \\
    \texttt{udib@il.ibm.com}
  \end{tabular}
  \hspace{1cm}
  \begin{tabular}[t]{c}
    Oshri Naparstek \\
    IBM Research\\
    Haifa, Israel \\
    \texttt{Oshri.Naparstek@ibm.com}\\\\
    Yair Weiss \\
    Hebrew University\\
    Jerusalem, Israel \\
    \texttt{yair.weiss@mail.huji.ac.il}
  \end{tabular}
}

\begin{document}

\maketitle
\begin{abstract}
Many modern multi-modal models (e.g. CLIP) seek an embedding space in which the two modalities are aligned. Somewhat surprisingly, almost all existing models show a strong modality gap: the distribution of images is well-separated from the distribution of texts in the shared embedding space. Despite a series of recent papers on this topic, it is still not clear why this gap exists nor whether closing the gap in post-processing will lead to better performance on downstream tasks. In this paper we show that under certain conditions, minimizing the contrastive loss yields a representation in which the two modalities are separated by a global gap vector that is orthogonal to their embeddings.  We also show that under these conditions the modality gap is monotonically related to robustness: decreasing the gap does not change the clean accuracy of the models but makes it less likely that a model will change its output when the embeddings are perturbed.  Our experiments show that for many real-world VLMs we can significantly increase robustness by a simple post-processing step that moves one modality towards the mean of the other modality, without any loss of clean accuracy.
\end{abstract}

\section{Introduction}
Foundation models are a common and successful approach to solving a variety of problems - models are trained on extremely large datasets and then adapted to other tasks either by fine tuning or zero shot application, under the assumption that they learned a meaningful embedding space. A specific and popular class of these models are contrastive multi-modal models. These are trained to learn a shared embedding space for different data types such as images and texts by aligning pairs of similar instances from the two modalities via a contrastive loss \citep{ntupletloss,oord, zhang22a}. These shared embedding spaces, specifically that of CLIP \citep{clip}, are commonly used for various tasks such as zero shot classification, text-to-image retrieval, text to image generation, and more.

The success of multi-modal models on a wide range of different applications  (e.g. text-to-image generative models \citep{dalle2}), might imply that they have successfully solved the optimization problem on which they were trained, yet some open questions persist. One of these regards the "modality gap" \citep{liang2022mind} - a phenomenon in which the two modalities are embedded into linearly separable regions of the unit hypersphere. Figure~\ref{fig:intro}a shows projections of CLIP's embedding of the MS-COCO validation set onto its first three principal components. Clearly, the two modalities are well separated in embedding space. This widely observed phenomenon across models and data types contradicts the training objective which aims to make similar texts and images perfectly overlap in embedding space—an objective known as "alignment" \citep{wang2020hypersphere}.

Recent works have tried to explain the cause for the formation of the gap, suggesting either information imbalance between images and texts (i.e. when a single caption is appropriate for multiple images) \cite{trigger} or the use of contrastive loss in cases when the effective dimension of each modality is small ("dimensionality collapse") \cite{zhang2024connect}, but a clear explanation remains elusive. Concurrently, investigation into the downstream effects of the gap has also failed to reach a clear conclusion \cite{liang2022mind,shi2023towards,informationtheorymodalitygap,liao2025multimodal,acceptthegap,zhang2024connect}.  \Cref{fig:bug_or_feature} presents the application of various multi-modal models on common downstream tasks. The performance of these models varies non-monotonously when adjusting the gap by moving the modality means towards each other, implying no clear relation between performance and the size of the gap. 

\begin{figure}[t]
  \centering
  \begin{subfigure}[t]{0.45\linewidth}
   { {\includegraphics[width=1.1\linewidth]{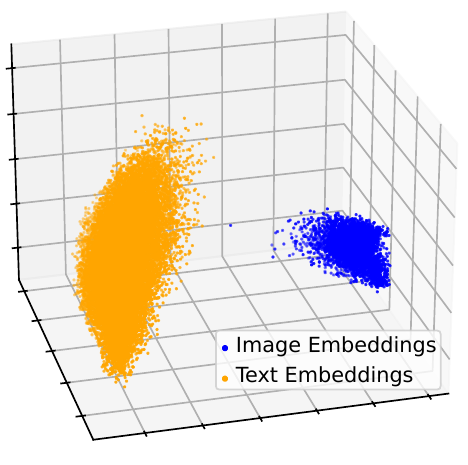}}
    \caption{CLIP's embedding of MS-COCO validation set}}
  \end{subfigure}
\hfill
 \begin{subfigure}[t]{0.45\linewidth}
    {{\includegraphics[width=\linewidth]{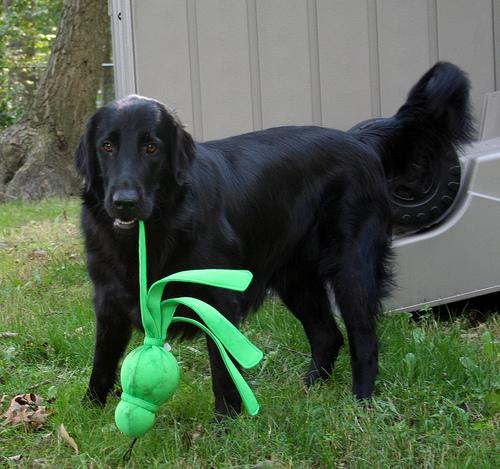}}
    \caption{"a photo of a \textbf{dog} / frog" {\color{green} $\vee$}\\
 "photograph of the dog / \textbf{frog}" {\color{red} $\times$}}
 }
  \end{subfigure}
  
  \caption[]{Left: Projections of CLIPs embedding of the MS-COCO validation set onto its first 3 principal components. A clear separation between images and texts is evident. Right: An image from the Imagenet \citep{imagenet} validation set that's misclassified by CLIP when changing the caption template. Multi-modal models can lose more than $6\%$ of their accuracy when replacing the caption template.}
    \label{fig:intro}
\end{figure}

Another open topic regarding multi-modal models and deep learning in general  focuses on the robustness of models to small variations in their input. Various works have shown that despite their large training corpus and the numerous augmentations used during their training, the performance of these models on downstream tasks is sensitive to small variations in the input such as single pixel shifts \citep{shifman} and caption rephrasing \citep{zhou2022cocoop,kimparaphrasing}. Figure~\ref{fig:intro}b demonstrates this brittleness. When trying to classify an image as a dog or a frog using CLIP, the model  outputs different predictions depending on the the two texts with which the image is compared, even if the change is semantically meaningless. 

In this work we establish a direct link between these two lines of work. In \cref{theory} we expand on \cite{zhang2024connect} and prove that under realistic assumptions, multi-modal models learn an embedding space containing a modality gap due to their initialization and the dynamics induced by the contrastive loss. We continue to prove that the size of the gap is  correlated with the robustness of these models, i.e. their ability to perform consistently under small changes to the embedding space. Our theory, along with empirical verification on real models and datasets, with both controlled (i.e. Gaussian) and real noise (i.e. text rephrasing), leads us to conclude that if robustness is a desired property of the models then the modality gap is indeed a bug caused by an interaction of the multi-modal contrastive loss and the initialization scheme.

Inspired by our theoretical findings, in \cref{section:algo} we present a simple, efficient algorithm to close the gap which can be applied as a post processing step when using the different models' embeddings spaces. In \cref{section:experiments} we demonstrate a use case of this algorithm in dealing with various noise models such as quantization and other real noise settings where the theoretical assumptions are not met. 

\begin{figure}[t]
  \centering
  \begin{subfigure}[t]{\figwidth\linewidth}
   { {\includegraphics[width=\linewidth]{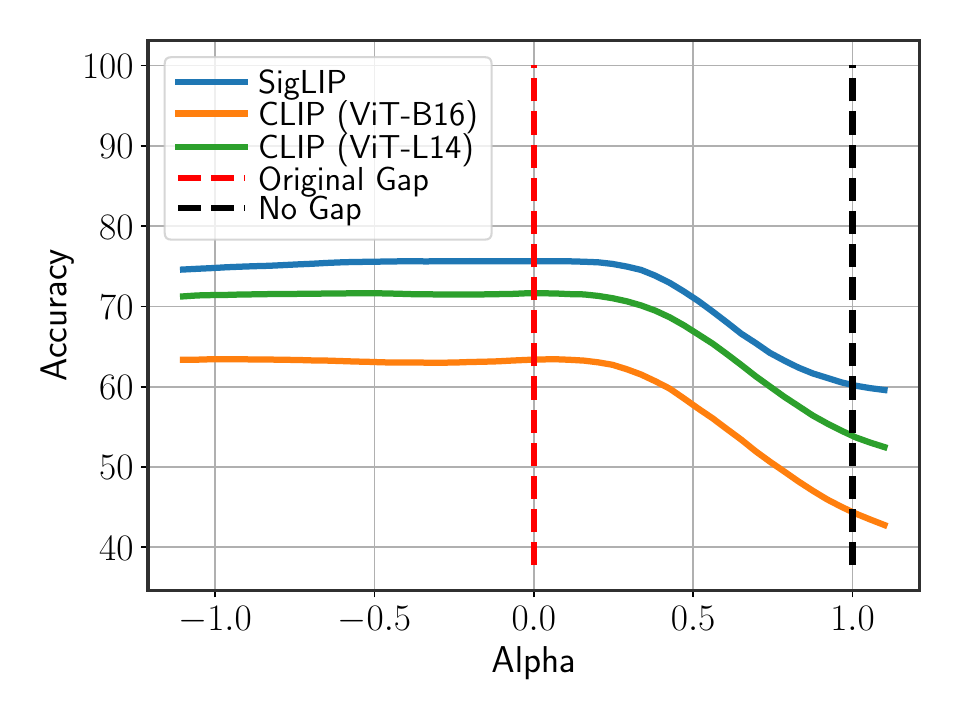}}
    \caption{Zero Shot on Imagenet}
    \label{fig:naive_gapinet}
    }
  \end{subfigure}
\hfill
 \begin{subfigure}[t]{\figwidth\linewidth}
    {\includegraphics[width=\linewidth]{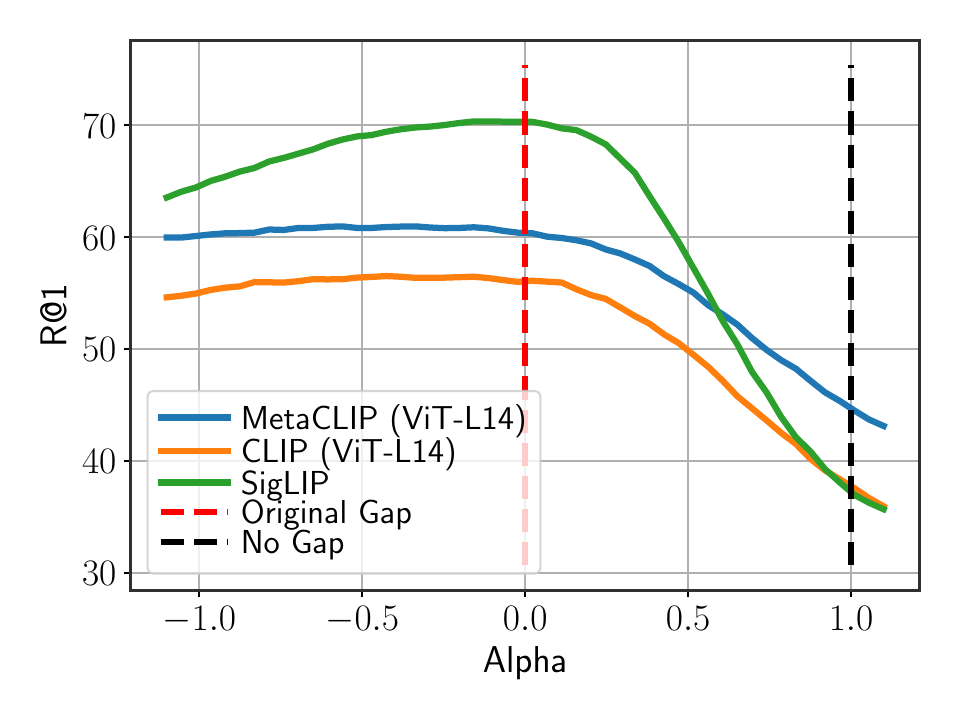}}
    \caption{I2T Retrieval on MS-COCO}
    \label{fig:naive_gapcoco}
  \end{subfigure}
     \ifthree{
     \hfill
     \begin{subfigure}[t]{\figwidth\linewidth}
  \includegraphics[width=\figwidth\linewidth]{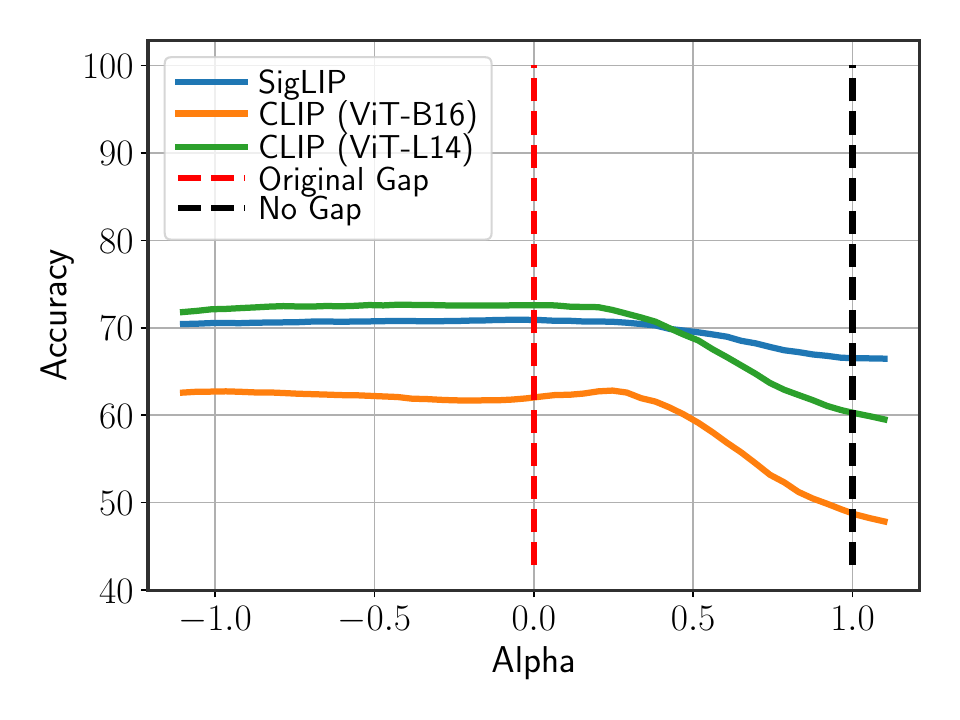}
  \caption{Zero Shot on CIFAR100}
  \label{fig:naive_gapcifar100}
  \end{subfigure}
  }
  
  \caption[]{Is the modality gap a bug or a feature? Changing the gap by moving the text embeddings by $\alpha\cdot\vec{g}$ has an inconsistent effect on downstream performance. The figure follows  \citep{liang2022mind} and shows that for some datasets, models benefit from slightly enlarging the gap (\ref{fig:naive_gapinet}), some from maintaining it (\ref{fig:naive_gapcoco}) \ifthree {and some from closing it (\ref{fig:naive_gapcifar100})}.} 
    \label{fig:bug_or_feature}
\end{figure}

\section{Related Work}

Multi-modal contrastive learning \citep{zhang22a} is a method for learning a shared embedding space for different modalities e.g. images and texts. Models trained in this manner such as CLIP \citep{clip} and others \citep{xu2023metaclip, siglip} show impressive performance on downstream tasks, e.g. zero shot classification. This has led to widespread use of CLIP's embedding space in various tasks, most notably in text to image generation \citep{dalle2,kang2023gigagan}. While the contrastive loss \citep{ntupletloss, oord} of multi-modal models pushes the modality pairs' embeddings to be aligned and uniformly spread on the unit sphere \citep{wang2020hypersphere}, in practice multi-modal models map the different modalities to distinct areas of the shared embedding space, creating a modality gap \citep{liang2022mind}. 

Several explanations as to the formation of the gap were suggested, stemming from either the data used to train these models \citep{liang2022mind,trigger}, the inherent differences between text and images \citep{trigger} or the training procedure \citep{shi2023towards,yaras2024explainingmitigatingmodalitygap}. Recently, Zhang et al.~\cite{zhang2024connect} have suggested that the gap is caused due to a combination of training with the contrastive loss and a dimensional collapse in the initial embeddings of the two modalities \citep{dimcollapse}. In this work we show that dimensionality collapse is neither necessary nor sufficient for the formation of the gap and we provide a more nuanced analysis of the dynamics of learning which cause the gap.

In addition to explaining the formation of the gap, previous works have also attempted to understand its effects on downstream performance. It has previously been shown that a large gap limits the ability to perform cross-modal tasks with uni-modal data \citep{zhang2024connect} and is negatively correlated in specific settings of cross modal retrieval \citep{liao2025multimodal,pmlr-v235-ming24a}. In general, however, recent work concludes that there is at most a weak effect of the gap on downstream performance \citep{informationtheorymodalitygap,liang2022mind,trigger} while other factors such as model and embedding size seem to have a stronger effect \citep{trigger}. Some works even suggest maintaining the modality gap \citep{acceptthegap, informationtheorymodalitygap}, or show that it has no effect in specific downstream settings \citep{zhang2023diagnosing}. We build on these findings and and prove that the gap indeed has negative effect on downstream performance in terms of robustness.

Alongside the attempt to understand the gap and its implications, previous works suggested various methods for closing the gap. These were mainly motivated by failure modes in specific settings \citep{mistretta2025cross,liao2025multimodal}, or by the inherent contradiction of the gap to the contrastive loss objective. Some works suggest closing the gap by altering the training procedure of multi-modal models with specific losses and augmentations \citep{oh2023geodesic,yaras2024explainingmitigatingmodalitygap}, while others include the training of models that implicitly close the gap as part of pipelines performing complex downstream tasks such as text to image generation \citep{dalle2,Patel_2024_CVPR}. Our work falls in line with these, but is applicable in a general setting - for any downstream task that relies on cross modal nearest neighbor retrieval, and with no additional training required.

A different line of work studies multi-modal models in the context of robustness \citep{pmlr-v162-fang22a,tu2023acliprobustness,clipadversarial}. Various works have shown that their performance on downstream tasks is sensitive to both textual variations \citep{zhou2022cocoop,qiu2022benchmarking}, adversarial attacks \citep{Schlarmann_2023_ICCV}, various image perturbations such as augmentations \citep{qiu2022benchmarking}, spurious correlations \citep{counteranimal}, natural variations \citep{shifman} and more \citep{9878622}. These are tackled mainly via finetuning and retraining with different objectives or data \citep{pmlr-v235-schlarmann24a,zhou2022cocoop,Khattak_2023_ICCV,pmlr-v202-shu23a}. Our theoretical work aims to enrich the understanding of these shortcomings by proving that a modality gap causes degrading in robustness.

\section{Theory}\label{theory}
\subsection{Notation and Definitions}
As in previous work, we assume two modalities that are embedded on the $d$-dimensional unit hypersphere. Our embeddings are in the form of sample matrices $\X\in \R^{N\times d}$ and $\Y\in \R^{M\times d}$ with modality means $\mu_x = \frac{1}{N}\sum_{\vec{x}\in \X}\vec{x}$ and $\mu_y = \frac{1}{M}\sum_{\vec{y}\in \Y}\vec{y}$. Many downstream tasks rely on cross modal retrieval - finding the nearest neighbor of sample $\vec{y}\in \Y$ among the samples in the other embedding $\X$ by computing the $\ell_2$ distance between the two: $\text{NN}(\vec{y},\X) =\argmin_{\vec{x}\in \X} \norm{\vec{x} - \vec{y}}^2$.

Since the embeddings lie on the unit hypersphere, the NN can equivalently be computed using cosine similarity.

We focus on models trained with the multi-modal contrastive loss. During training we are given a paired dataset $\{x_i,y_i\}$ (e.g. image and matching caption) and minimize the following loss: 
\begin{equation}
    \label{mmloss}
    \mathcal{L}(\X,\Y) =
    -\frac{1}{N} \sum_{i=1}^N \log \ell(\vec{x}_i,\Y) -\frac{1}{N}\sum_{j=1}^N \log \ell(\vec{y}_j,\X) 
\end{equation}
Where $\ell$ contrasts an  embedding with the other modality:
\begin{equation}
    \ell(\vec{x}_i,\Y)  =\frac{e^{-\norm{\vec{x}_i-\vec{y}_i}/\tau}}{\sum_{\vec{y}_j\in \Y} e^{-\norm{\vec{x}_i-\vec{y}_j}/\tau}} 
\end{equation}
with $\tau>0$ being the temperature parameter.

We follow previous works \citep{zhang2023diagnosing} and define, for a paired dataset the local gap as the vector between a pair from the two modalities:
\begin{equation}\label{localgap}
    \vec{g}_i=\vec{x}_i-\vec{y}_i
\end{equation}
We also define the global gap \citep{liang2022mind,zhang2023diagnosing,shi2023towards} as the vector between the modality means:
\begin{equation}\label{gap}
    \vec{g}=\frac{1}{M}\sum_{\vec{y} \in \Y}\vec{y} - \frac{1}{N}\sum_{\vec{x}\in \X}\vec{x} = \mu_y-\mu_x
\end{equation}
When the two modalities are balanced ($N=M$), the global vector is the average of local vectors. Nevertheless, the global gap vector is well defined for any dataset consisting of two modalities, whether a bijective pairing exists or not.

\begin{figure}[h]
    \centering
    \begin{tabular}{cc}
       
        \includegraphics[width=0.47\linewidth]{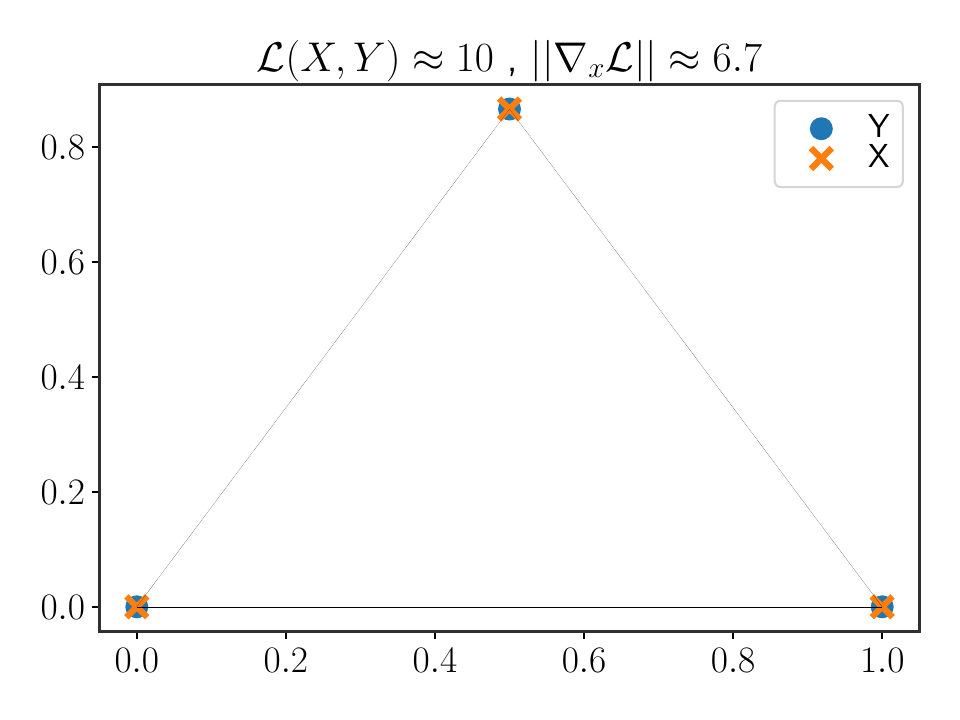}
         & 
          \includegraphics[width=0.47\linewidth]{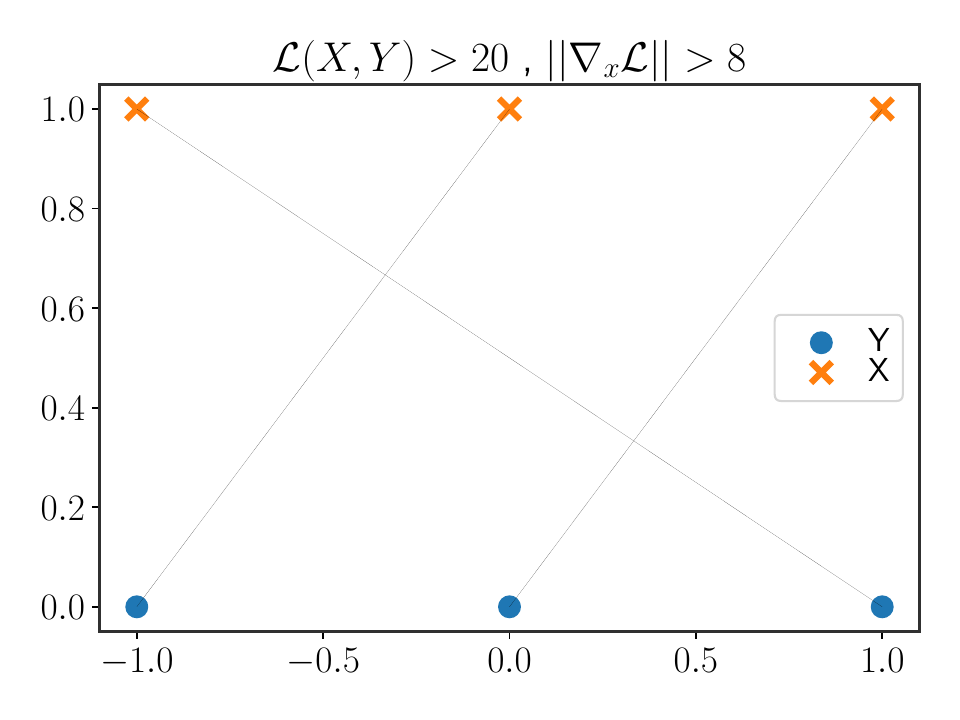}
         \\
          \includegraphics[width=0.47\linewidth]{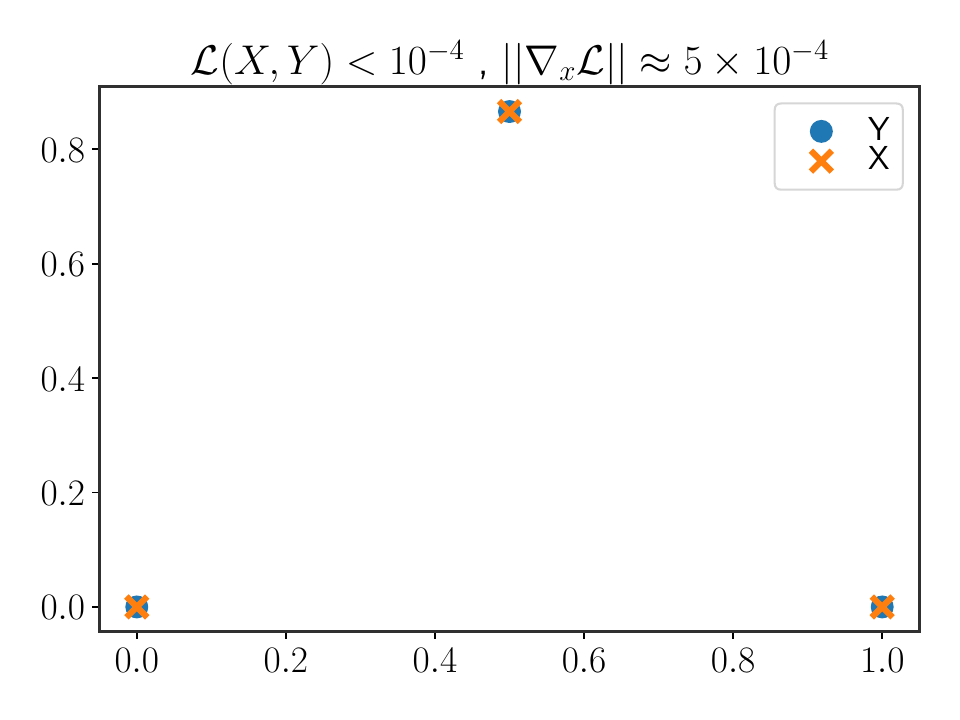}
         & 
          \includegraphics[width=0.47\linewidth]{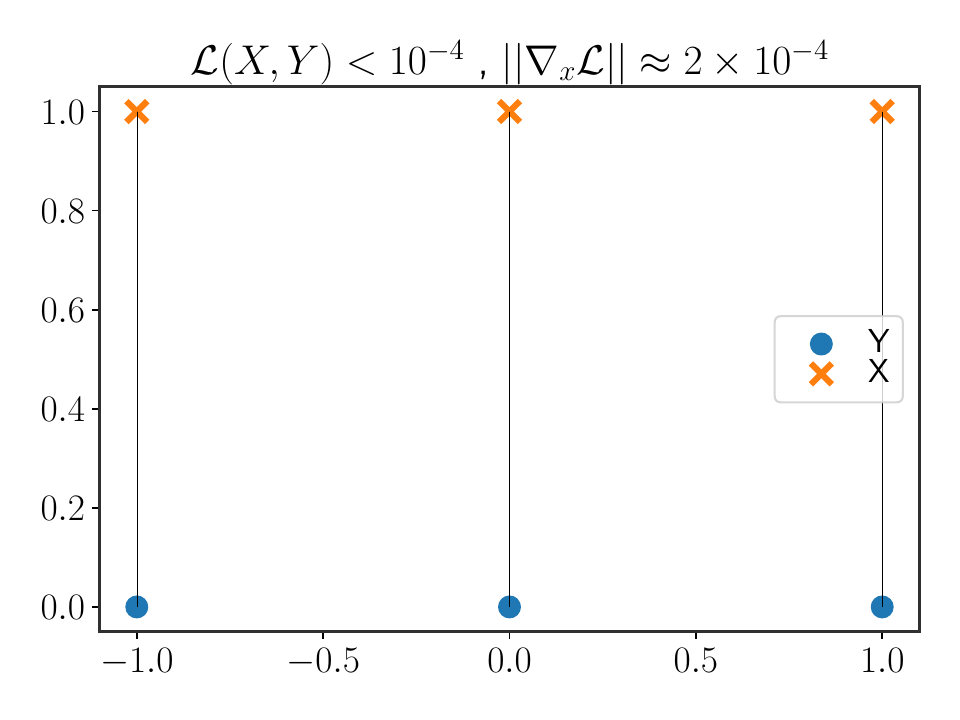}
    \end{tabular}
    \caption{Three points in each modality in  $\R^2$ and the corresponding multi-modal contrastive loss along with the  magnitude of the gradient of the loss. Lines connect  true pairs. As long as the points satisfy relative alignment - the true pair of any point is also its nearest neighbor - the loss and the gradient magnitude are close to zero, even when there exists a gap.}
    \label{fig:theory1} 
\end{figure}

\subsection{Why Should a Global Gap Exist?}\label{sec:gapexistence}

\begin{figure*}
    \centering
    {\setlength{\tabcolsep}{0pt}
    \begin{tabular}{@{}l@{}l@{}l@{}l@{}l@{}}
        \includegraphics[width=0.2\linewidth]{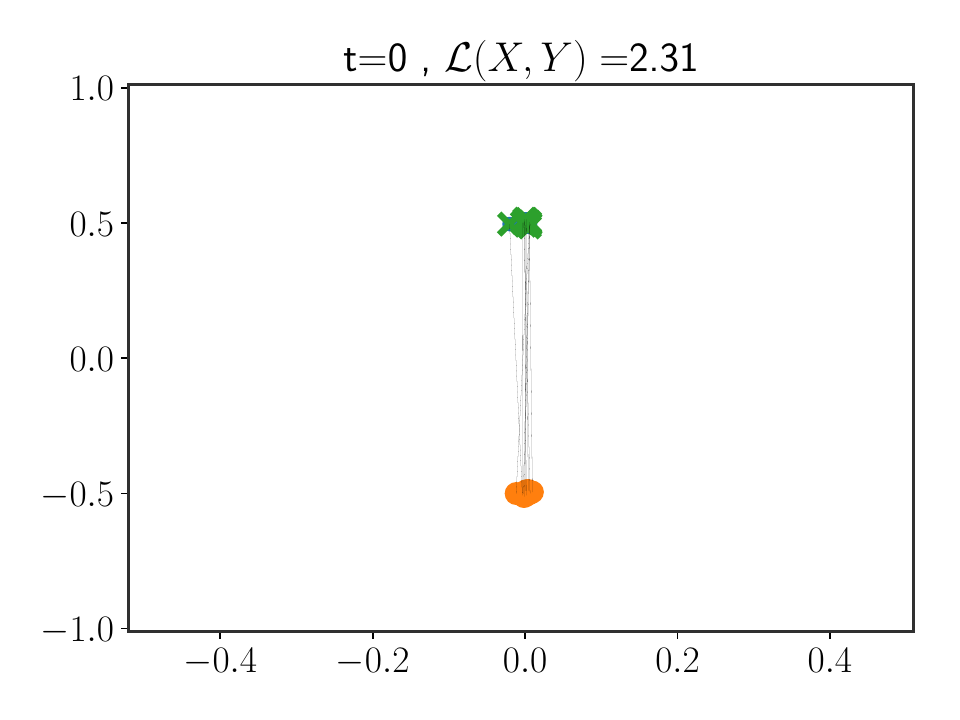} &
        \includegraphics[width=0.2\linewidth]{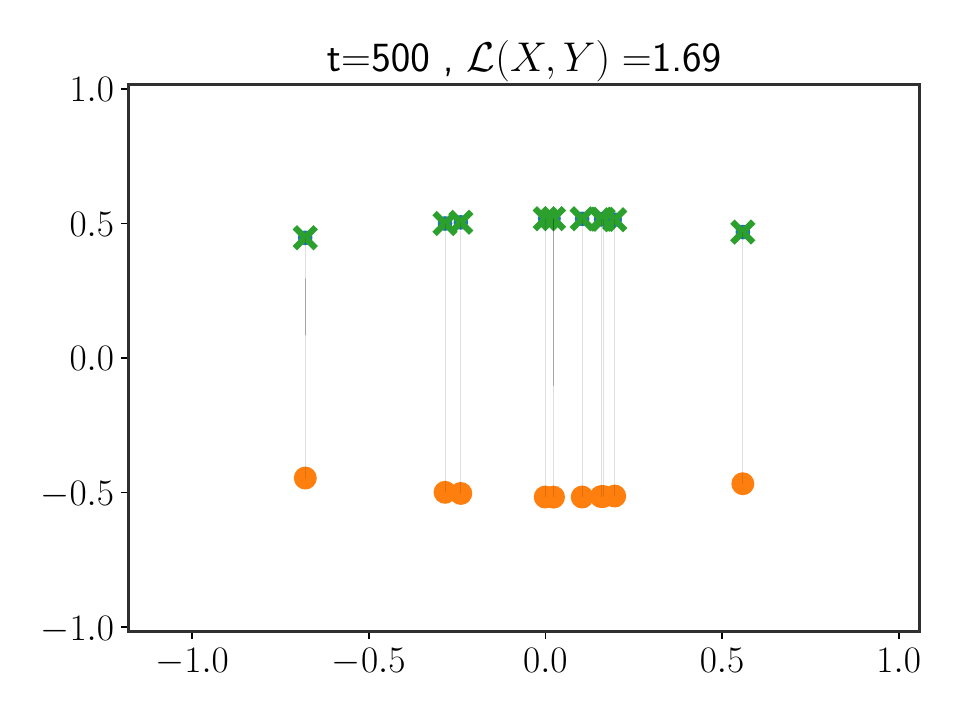} &\includegraphics[width=0.2\linewidth]{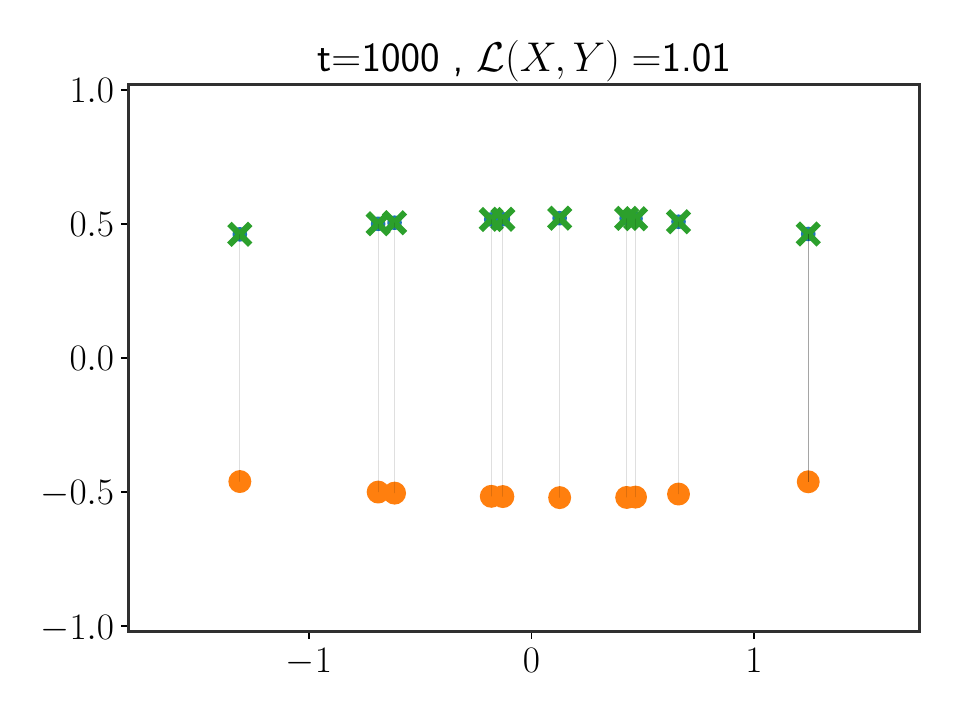} &
        \includegraphics[width=0.2\linewidth]{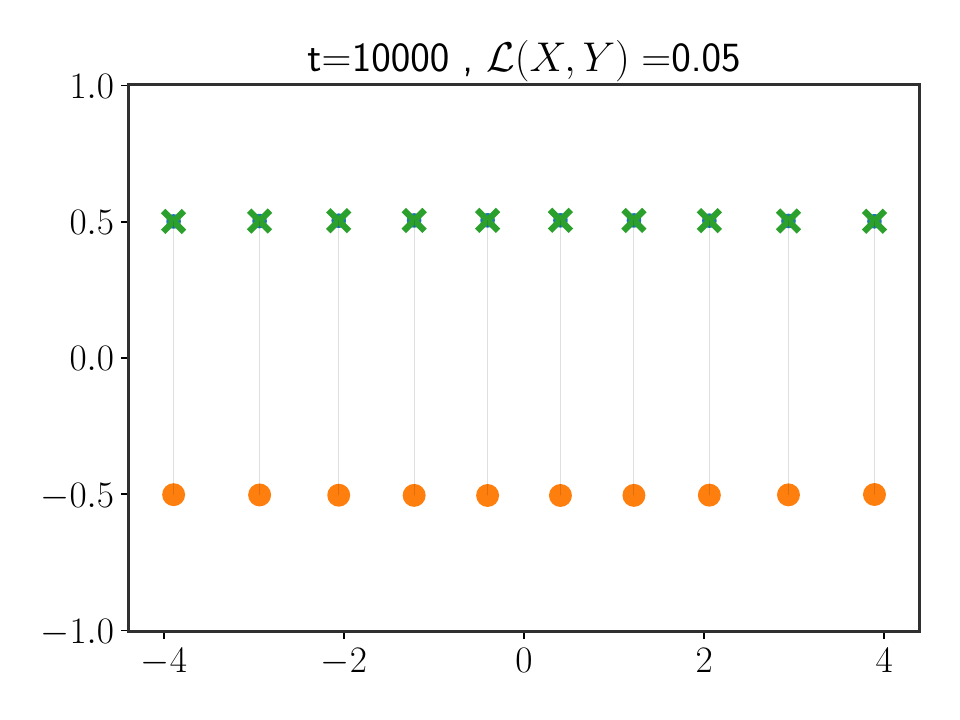} &
        \includegraphics[width=0.2\linewidth]{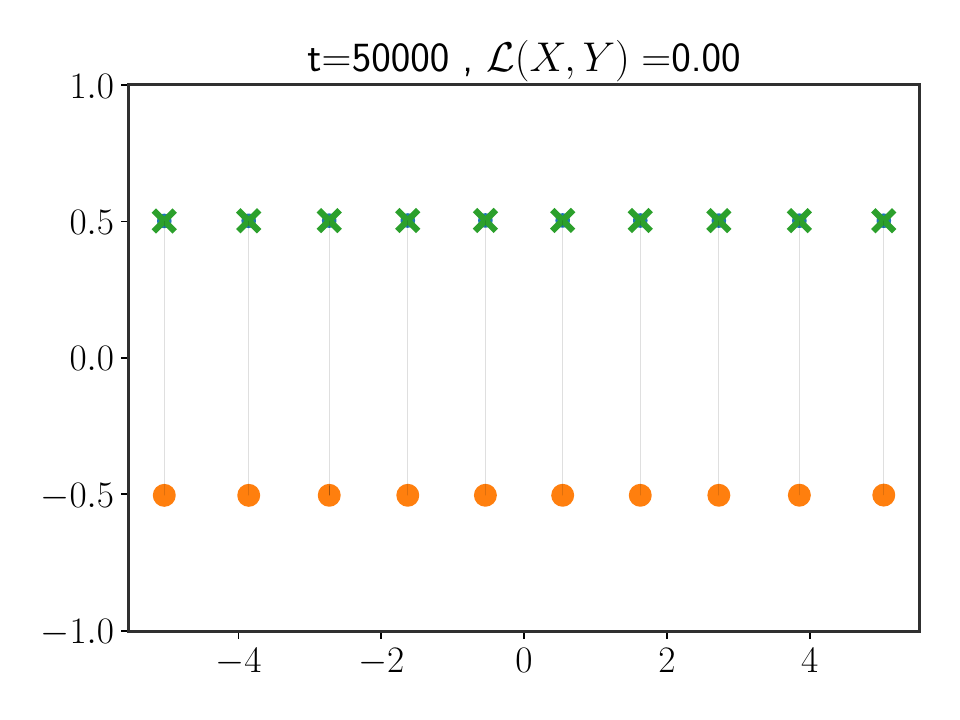}  \\
        \includegraphics[width=0.2\linewidth]{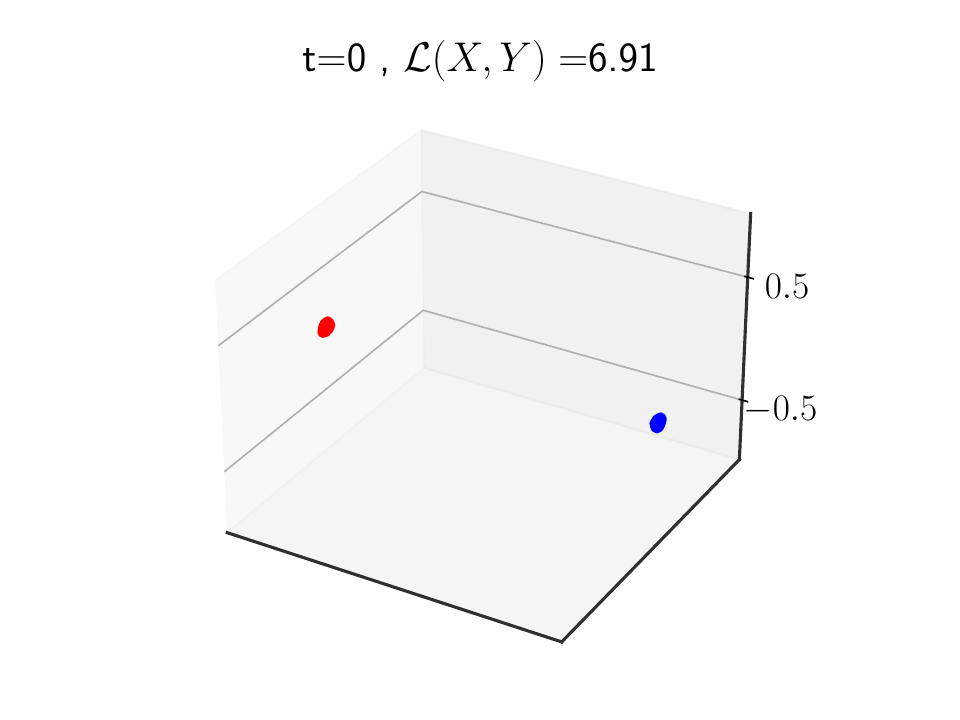} &
        \includegraphics[width=0.2\linewidth]{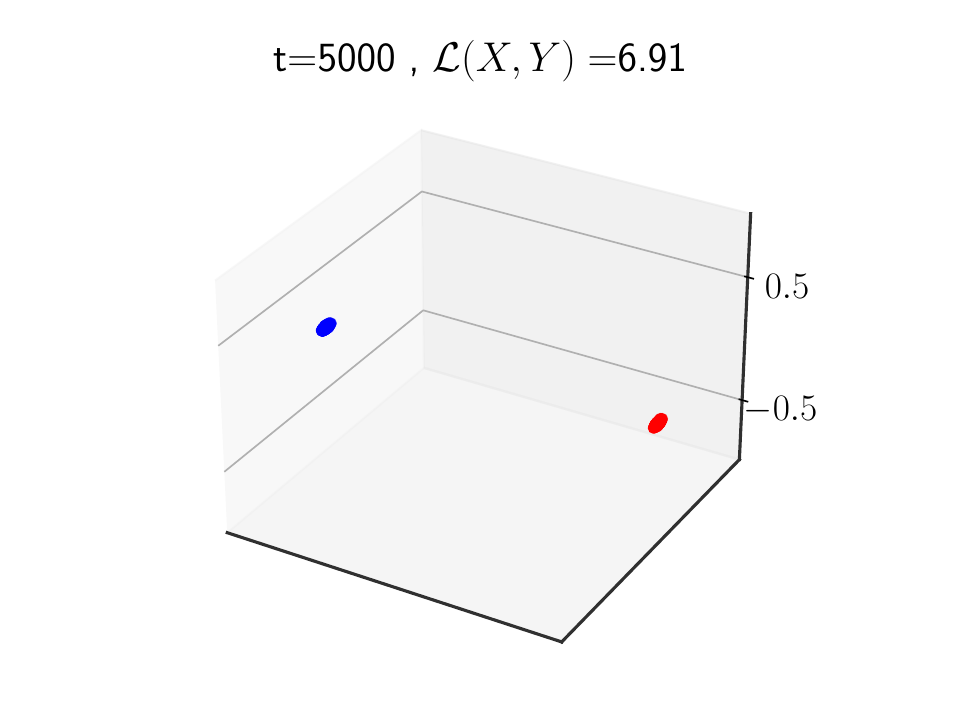} &
        \includegraphics[width=0.2\linewidth]{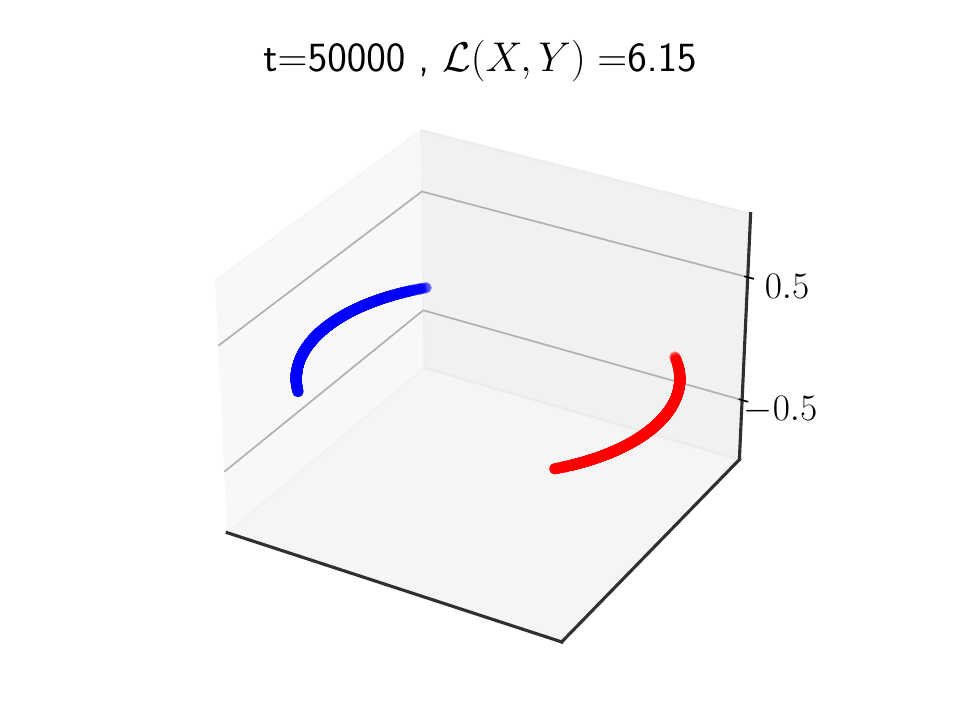} &
        \includegraphics[width=0.2\linewidth]{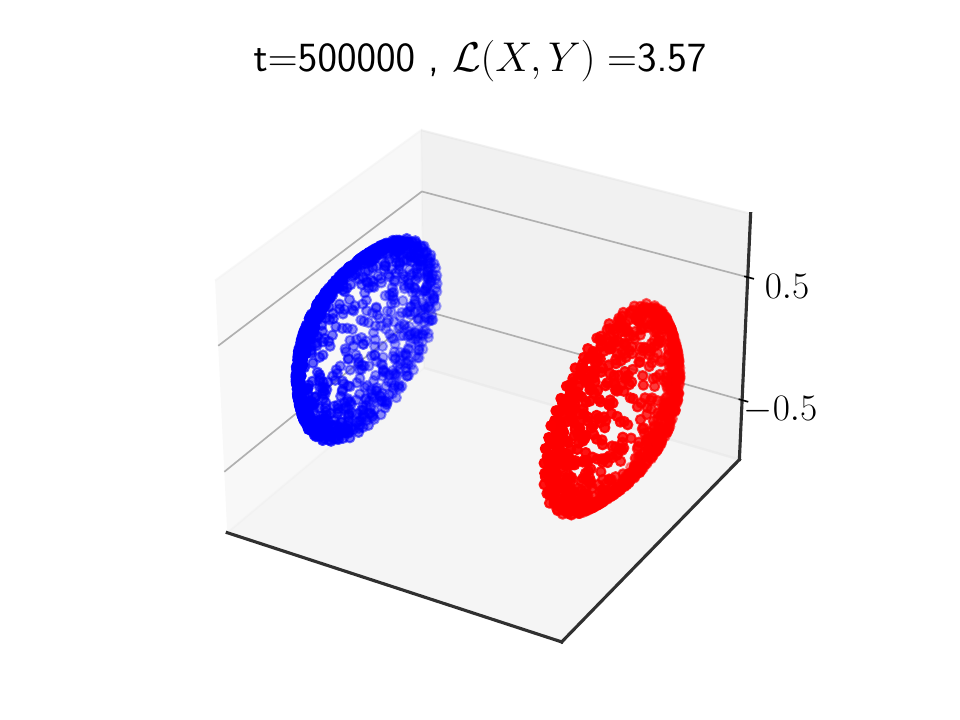} &
        \includegraphics[width=0.2\linewidth]{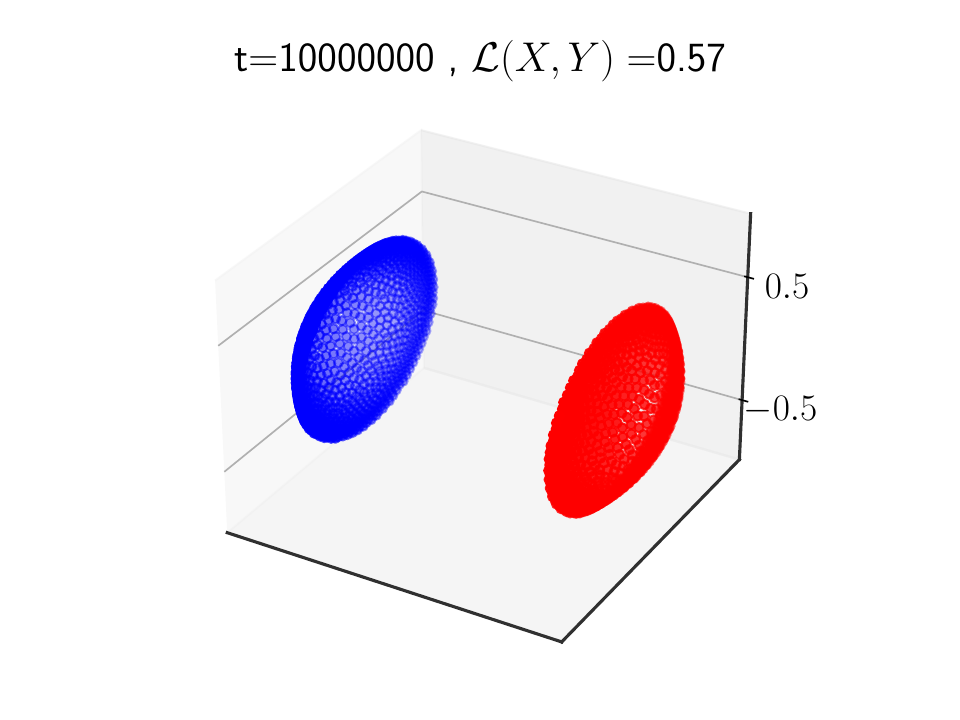}
        \end{tabular}
        }
    \caption{The evolution of embeddings using gradient descent on the contrastive loss (\cref{mmloss}) starting with two tight clusters. Both when using an unnormalized embedding space (top) and when constricting the embeddings to the sphere (bottom) they converge to a solution that has almost zero loss and for which a global gap vector exists and the gap vector is orthogonal to both modalities. We prove that minimizing the contrastive loss will lead to such a solution under certain assumptions. For full training details see \cref{appendix:simulations}.}
    \label{fig:theory2}
\end{figure*}

Since the contrastive loss rewards embeddings in which the representations in the two modalities are aligned \cite{wang2020hypersphere}, the presence of a global gap in embeddings learned using such a loss is highly surprising. Indeed, it is easy to show that the global minimum of the contrastive loss is obtained when the representations are maximally aligned, i.e. $\forall i: x_i=y_i$, entailing no local or global gap, i.e. $\vec{g}_i=0$. But it is also easy to see that there exist many other embeddings that achieve almost the same loss as the globally aligned embedding. \Cref{fig:theory1} shows an example. The lines are the local gap vectors $\vec{g}_i$ which show the correspondence between the two modalities: ideally we want circles and crosses that are connected to lie on top of each other in embedding space, meaning $\vec{g}_i=0$. As can be seen in the figure, the loss and the gradient can be made arbitrarily close to zero without perfect alignment: as long as the corresponding embeddings are closer than all other cross-modal pairs, the loss and the gradient will approach $0$. In particular, the solution in the bottom right (where there is a modality gap) achieves practically the same loss and gradient magnitude as the desired, perfectly aligned solution.  Intriguingly, gradient descent sometimes seems to prefer these low-loss solutions over the global optimum. \Cref{fig:theory2} shows dynamics of points in low dimensions, in which we minimize the contrastive loss with respect to the embeddings coordinates. In both of these cases, training converges to a solution that has almost zero loss but for which a global gap vector remains. 

Zhang et al. \cite{zhang2024connect} postulate that dimensional collapse, i.e. when the embeddings' variance within each modality is concentrated in relatively few dimensions, is the cause for the modality gap. But why does the modality gap happen when there is equal variance in all dimensions? \Cref{fig:theory2,fig:varshrink} show such cases - both embeddings are initialized by sampling from isotropic Gaussians, therefore maintaining equal variance in all dimensions even after normalization to the sphere. Still learning yields a representation with a global gap, even without dimensional collapse. Therefore, further investigation is needed.

We begin by defining:
\begin{align}
Q^x(i,j) &= \frac{e^{-\|x_i-y_j\|^2/\tau}}{\sum_{j'} e^{-\|x_i-y_{j'}\|^2/\tau}}   \notag \\
Q^y(i,j) &= \frac{e^{-\|x_i-y_j\|^2/\tau}}{\sum_{i'} e^{-\|x_{i'}-y_{j}\|^2/\tau}}   
\end{align}

Which are known as the softmax of the logit matrix \cite{clip}. $Q^x$ can be thought of as a soft assignment matrix: for each $x_i$ it measures the normalized affinity between $x_i$ and all points $y_j$ in the other modality while $Q^y$ does the same for each point $y_i$.  By construction, the sum of the rows of $Q^x$ are equal to one, and the sum of the columns of $Q^y$ are equal to one.
Using these matrices the loss becomes:
\begin{equation}\label{mmloss}
    \mathcal{L}(\X,\Y) = -\parent{
    \frac{1}{N}\sum_{i=1}^N \log Q^x(i,i) + \log Q^y(i,i) 
    }
\end{equation}

The gradient of the loss can be written:

\begin{equation}\label{eq:grad}
\frac{\partial \mathcal{L}}{\partial y_i}
\propto (x_i - y_i)
- \sum_{k=1}^N \frac{Q^x(k,i) + Q^y(k,i)}{2}
  (x_k - y_i)
\end{equation}


The gradient has two opposing effects: an attractive force $(x_i - y_i)$ pushes $y_i$ towards it matching $x_i$ and repulsive forces which push $y_i$ away from $x_k$. The strength of the repulsive force is determined by $\sum_{k=1}^N \frac{Q^x(k,i) + Q^y(k,i)}{2}= \frac{1}{2}(1+S_i^y)$ with $S_i^y := \sum_k Q^x(k,i)$. Intuitively, this means that points for which $S_i^y>1$ will be pushed away from the other modality (their repulsive force are stronger than the attractive force) while those for which $S_i<1$ will be pushed towards the other modality. Interestingly, points for which $S^y_i>1$ are those that are close to the other modality (\cref{fig:varshrink} top), so gradient descent will reduce the variance in the direction of the initial gap vector (\cref{fig:varshrink} bottom). The following theorem formalizes this intuition.

\begin{theorem}\label{theorem:newtheory}
    Let $\mu_x$ and $\mu_y$ be the modalities' empirical means.
     Assume that $\forall i: \norm{y_i-\mu_y} \leq \epsilon, \norm{x_i-\mu_x} \leq \epsilon$,
    and $\norm{\mu_x - \mu_y}\gg\epsilon$.
    Then the gradient of the contrastive loss with respect to the  embeddings 
    is approximately in the direction that shrinks the variance of each modality in the direction of the gap vector $\vec{g}=\mu_x-\mu_y$.
\end{theorem}
\begin{proofsketch}
We compute the Taylor expansion of $k(x)=e^{-\|x-y\|^2/\tau^2}$ around $\mu_x$ and use it to simplify  \cref{eq:grad}, yielding:
\begin{equation}
\frac{\partial \mathcal{L} }{\partial y_i} \approx \frac{-2}{\tau} 
[\vec{g} \cdot (y_i - \mu_y) ] \vec{g} + O(\epsilon) 
\end{equation}
Thus $y_i$ moved along the gap in the direction of $\mu_y$. See full proof in \cref{appendix:proofs}.
\end{proofsketch}

\begin{figure}[hb!]
    \centering
    \includegraphics[width=0.5\linewidth]{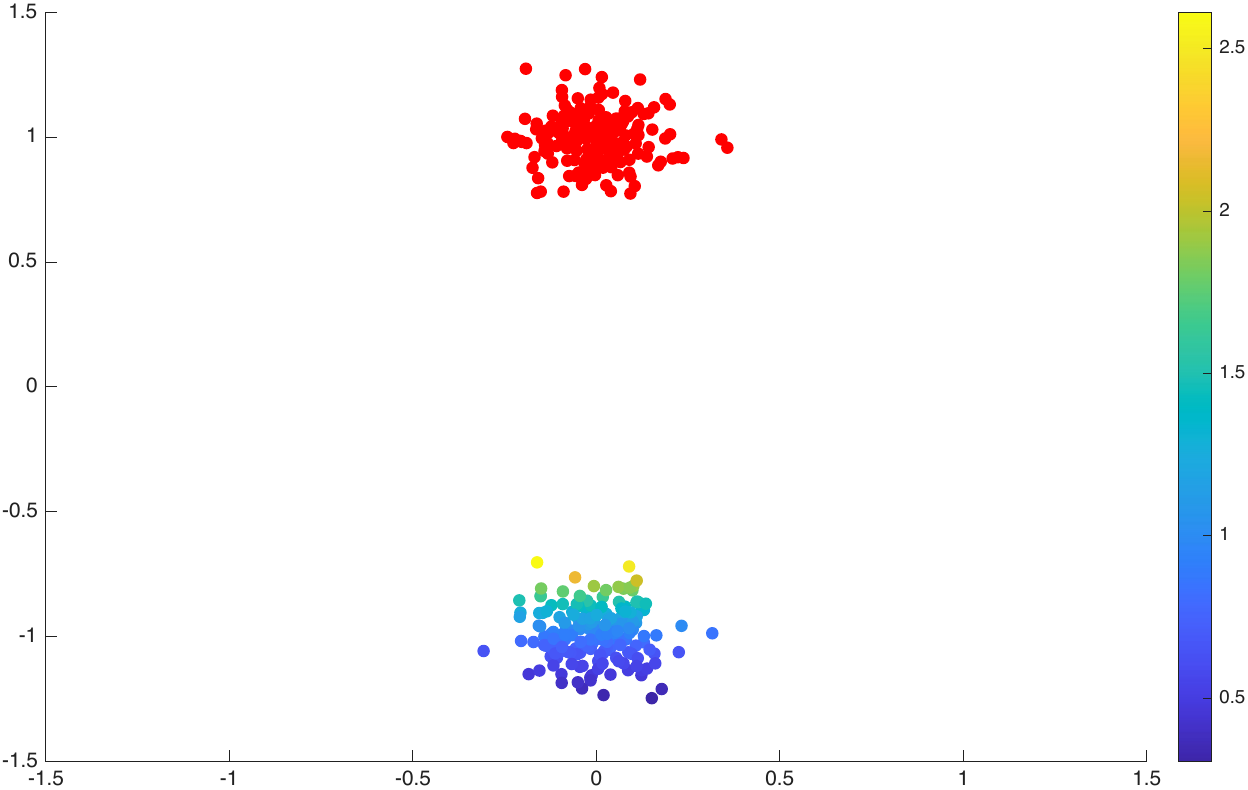}
    \begin{tabular}{ccc}
        \small Init. &
        \small $t=25$ &
        \small $t=100$  \\
        \includegraphics[width=0.3\linewidth]{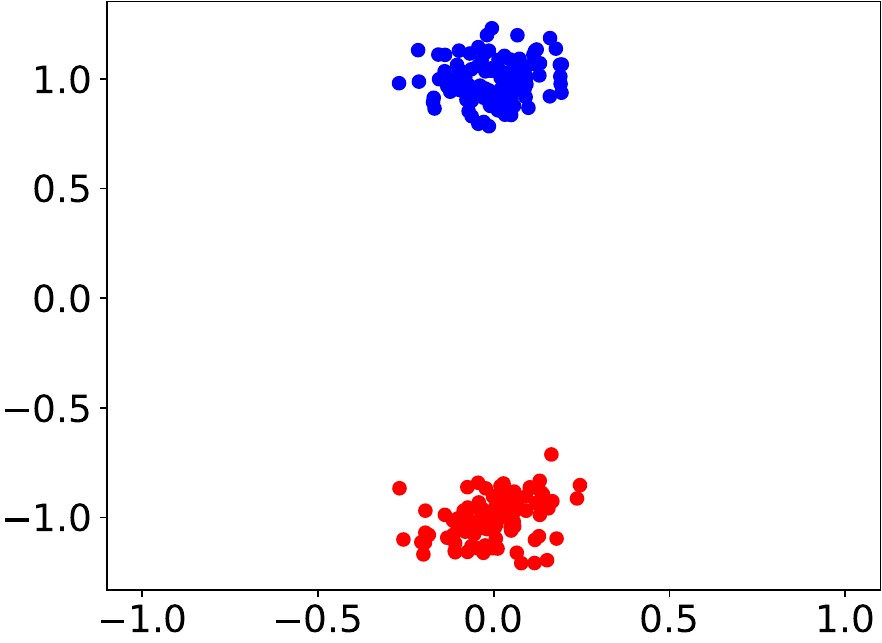} &
        \includegraphics[width=0.3\linewidth]{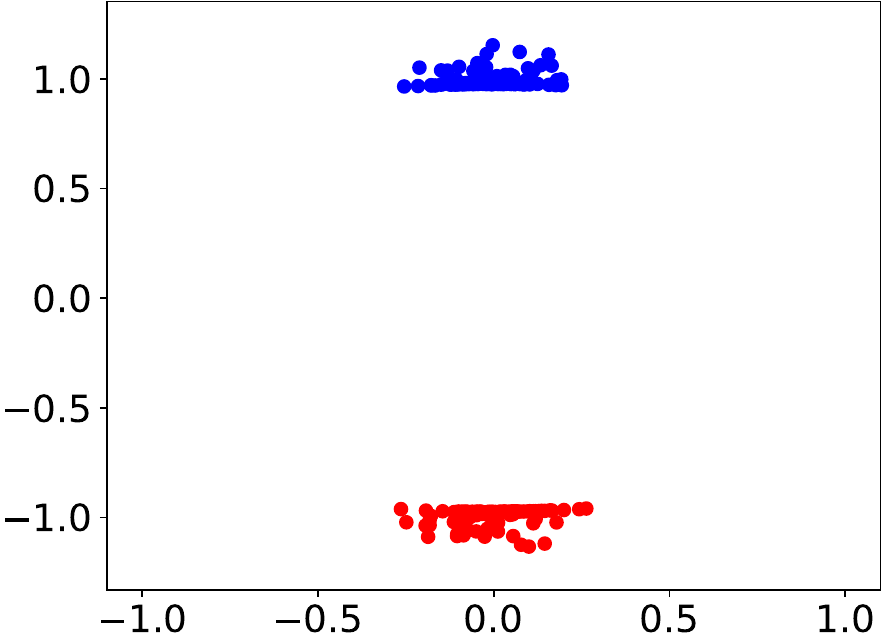} &
        \includegraphics[width=0.3\linewidth]{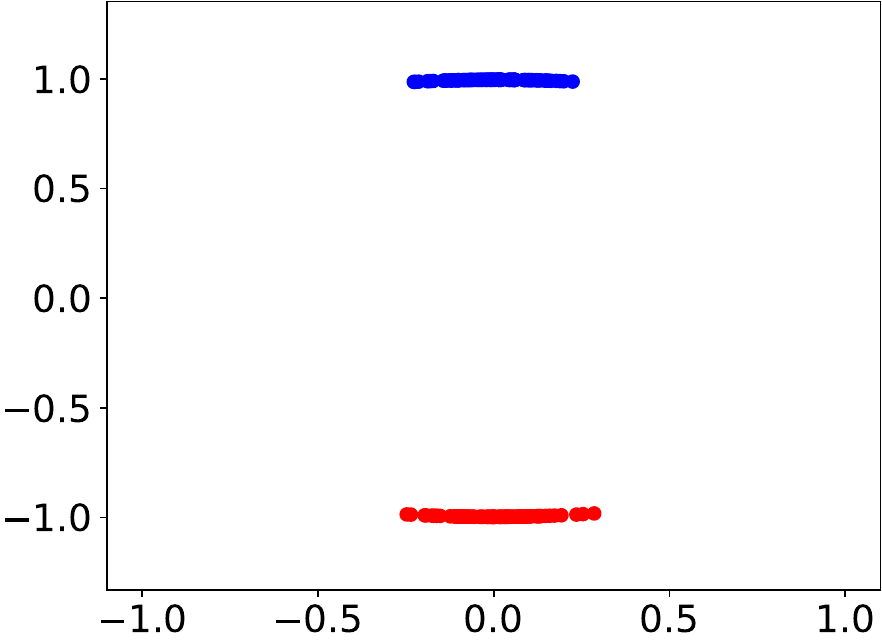}
    \end{tabular}

    \caption{ {\bf (Top:)} Two initial embeddings where the bottom embedding is color-coded based on $S^y_i$. $S^y_i$ decreases with distance to the other modality. {(\bf Bottom:}) The training dynamics of a toy model initialized with isotropic Gaussians. Training starts by shrinking variance in the direction of the gap according to \cref{theorem:newtheory}.}
    \label{fig:varshrink}
\end{figure}

\Cref{fig:varshrink} illustrates
\cref{theorem:newtheory}. The embeddings begin with equal variance in all directions, but training first shrink the variance in each modality along the initial gap direction $\vec{g}$, resulting in no variance in the gap.


We now prove that once there exists a direction along which both modalities are constant,   gradient descent will converge to a solution where both the local gap vectors $g_i$ and the global gap vector $\vec{g}$ are orthogonal to both modalities.   Our result similar to that of \cite{zhang2024connect}  but it highlights the crucial role of double-stochasticity: convergence to a nonzero, orthogonal gap vectors requires that the two matrices $Q^x,Q^y$ are {\em doubly-stochastic}, i.e. $S_j^x=S_j^y=1$.   While both matrices $Q^x,Q^y$ are singly stochastic by definition, only in certain point configurations will they also be doubly stochastic. One such case is when the nearest neighbor matching between the two sets is a perfect match so each $x_i$ is assigned to exactly one $y_j$ and vice versa. 
Empirically, we have found that in the later stages of gradient descent, double stochasticity does indeed hold approximately but this is not true in the initial iterations (\cref{fig:boxplots} left).  

\begin{theorem}\label{theorem:mmloss}
    Assume that at a given iteration $t$ 
 there exists a direction $\vec{v}$ such that $\forall \vec{x}_i\in \X: \vec{v}^T \vec{x}_i=a$ and $\forall \vec{y}_i\in \Y: \vec{v}^T \vec{y}_i=b$. Assume that for all iterations after $t$ the matrices $Q^x,Q^y$ are doubly stochastic. Then gradient descent on the contrastive loss (\cref{mmloss}) will converge to a solution where all the local gap vectors are the same $\vec{g}_i=\vec{g}$ and furthermore, $\vec{g}_i$ and $\vec{g}$ are orthogonal to both $X$ and $Y$.
\end{theorem}

\begin{proofsketch}
By the assumption that $v^T x_i=a$ for all $x_i$ and $v^T y_j=b$ for all $y_j$ we see that using \cref{eq:grad}:
\begin{eqnarray*}
v^T \frac{\partial \mathcal{L}}{\partial y_i}
&\propto  &v^T \left( x_i - y_i \right) \\
&& - v^T \sum_{k=1}^N \frac{Q^x(k,i) + Q^y(k,i)}{2}
  (x_k - y_i) \\
  &=& 0
\end{eqnarray*}

where the last equality used the assumption of double stochasticity. 
This means that gradient descent will not change the values of either modality in direction $\vec{v}$ and hence will converge to a solution in which this direction is unchanged and in all other directions $\vec{u}$, the modalities will be perfectly aligned $\vec{u}^T x_i = \vec{u}^T y_i$. 
\end{proofsketch}

Together, \cref{theorem:newtheory} and \cref{theorem:mmloss} explain the dynamics seen in \cref{fig:theory2}. At first, according to \cref{theorem:newtheory}, learning begins by shrinking the variance in the direction of the gap. Once this happens and there is no variance in the direction of the gap, then according to \cref{theorem:mmloss} the gradient in the vertical dimension of the top figure remains zero and the loss continues to decrease by moving the points only in the horizontal direction until the loss reaches a value close to zero with a constant, orthogonal gap vector. In the bottom figure, the initial iterations reduces the variance in the gap direction and subsequent iterations decrease  while moving in the other  two directions.

This phenomenon of an orthogonal gap has previously been noticed to hold empirically for multi-modal models trained on various datatypes \cite{zhang2023diagnosing}. We term this the \textit{global orthogonality assumption} and formally define it as:

\begin{assumption} \label{assump}
    The gap vector $\vec{g}$ is orthogonal to the affine subspaces defined by $\X$ and $\Y$:
    \begin{align} \label{eq:ortho}
        \forall x\in \X, y\in \Y: \cos \parent{x-\mu_x , \vec{g}  } \notag \\
        = \cos \parent{y-\mu_y , \vec{g} }=0
    \end{align}
\end{assumption}

Our theorems provide an explanation for the origin of assumption~\ref{assump} provided that the initialization of the two modalities takes the form of two, separated and concentrated clusters. \Cref{fig:boxplots} (right) shows that this is indeed the case - the distance between the two modalities is much larger than the inner modality distance for random initializations of CLIP on real data. This has already been shown to be the case for multi-modal models, a phenomenon termed the "cone effect" \cite{liang2022mind}.

To summarize, multi-modal models are commonly initialized in non-intersecting tight clusters with an existing gap. According to our analysis,  under these conditions initial training dynamics decrease the variance of these clusters in the direction of the initial gap,  while subsequent iterations maintain the gap while aligning the two modalities only in directions orthogonal to the gap, thus creating global orthogonality at convergence.   
For completeness, we provide more empirical support for our assumptions in \cref{appendix:assumptions} and further discussion of the dimensionality collapse assumption in \cref{appendix:dimcollapse}.

\begin{figure}[ht!]
    \centering
    \begin{tabular}{cc}        \includegraphics[width=0.45\linewidth]{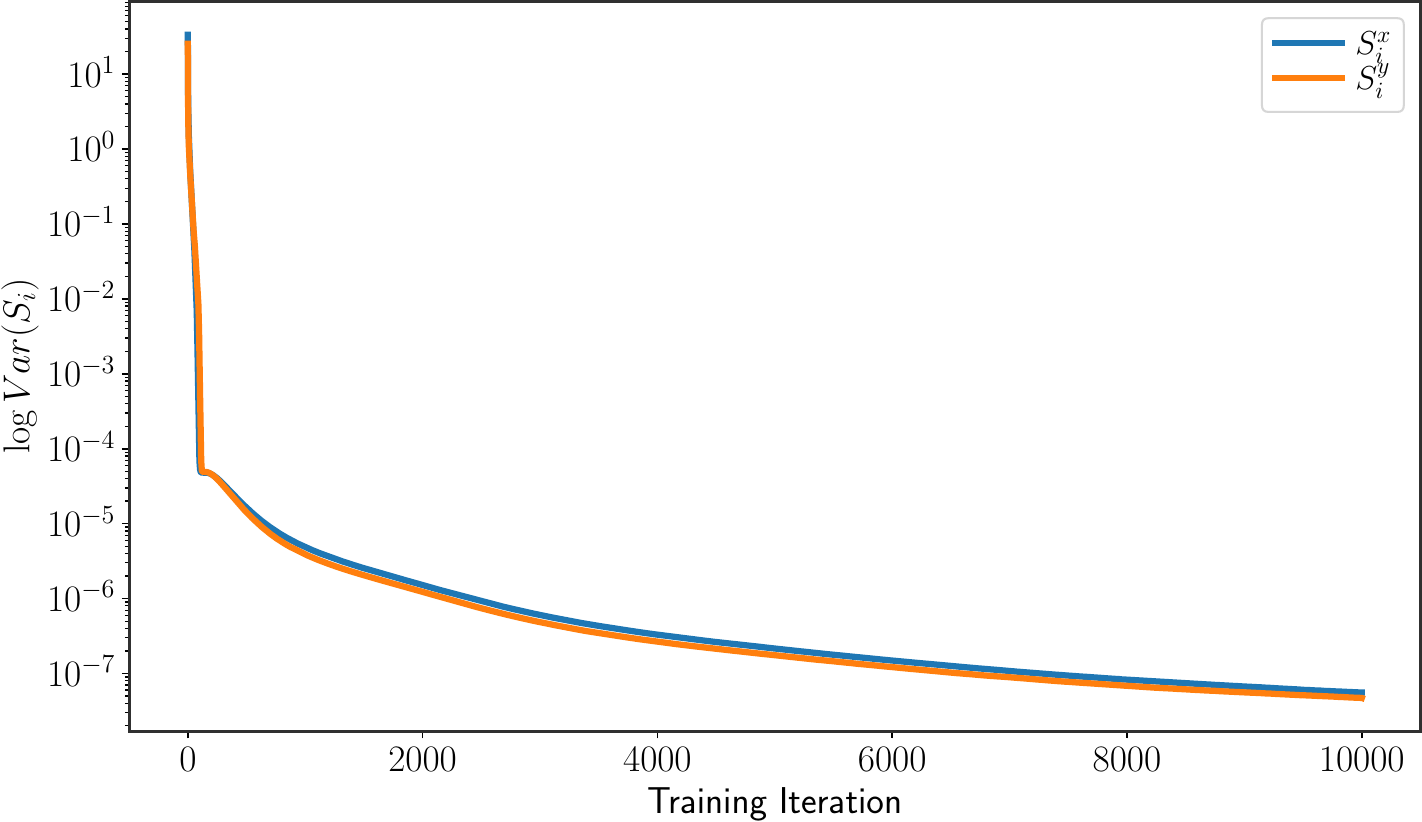}\label{fig:si} \hfill &
        \includegraphics[width=0.45\linewidth]{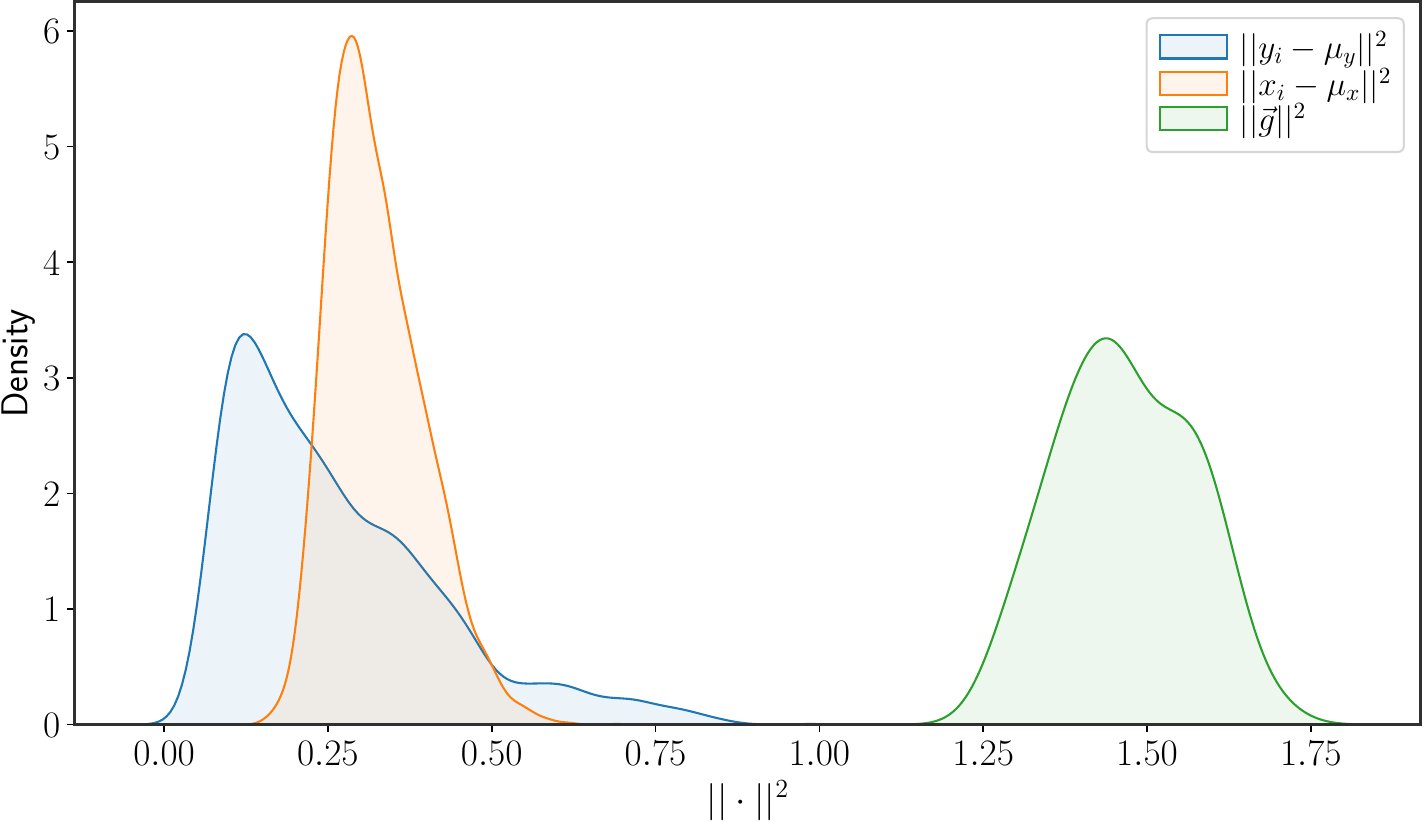} \label{fig:kdeplot}%
    \end{tabular}
    \caption{Empirical evidence for the assumptions of our thoeorems on training dynamics \textbf{Left}: The variance of $S_i^x, S_i^y$ in log scale throughout training of \cref{fig:varshrink}. The values of $S_i^x$ and $S_i^y$ approach their means as training progresses, which is $1$ by definition, meaning that the matrices $Q^x,Q^y$ are approximately doubly-stochastic.  \textbf{Right}: KDE plot of the gap's norm and the distances between the embeddings and their means calculated over 50 random initializations of CLIP (ViT-B-32) on 5000 randomly sampled images and texts from the Flickr30k dataset \cite{flickr}. The within modality distance is much smaller than the distance between the modalities. }
    \label{fig:boxplots}
\end{figure}

\subsection{The Modality Gap Decreases Robustness} \label{section:modalityrobust}

We define robustness ("Rob") as the probability that when adding noise sampled from some distribution $\mathcal{P}$ to the embeddings of modality $\X$, the nearest neighbor of some $y\in \Y$ will not change:
\begin{equation}
\text{Rob}(\X,\Y,\mathcal{P}) = \E_{y\sim\Y,\epsilon\sim \mathcal{P}} \sqrbr{\mathds{1}_{\text{NN} y,\X) = \text{NN}(y,\X + \epsilon)}}
\end{equation}
Notice that in the zero shot classification setting, this definition of robustness captures semantically meaningful changes - any change of nearest neighbor from image embedding to text, necessarily changes the classification of that image.

We now ask how does the existence of the gap relate to the robustness of the models to noise in the embedding space. \Cref{robustness_illustration} provides intuition. In this figure, there are two captions and the space of images that are close to one caption or another is defined by a linear decision boundary.  Any noise added to the text embeddings changes the decision boundary between the two classes. In the case of a gap between the cluster of image embeddings and text embeddings, as images are further away from the texts, their sensitivity to change in the classification decision boundary grows. We prove this formally in the following theorem.

\begin{theorem}\label{theorem:robust}
    Let $\{\vec{y}_i\}$ be a set of points in one modality (e.g. an image embedding) and $\vec{x}_1,\vec{x}_2\in \X$ two vectors in the second modality (e.g. captions) with mean $\mu_x = \frac{\vec{x}_1+\vec{x}_2}{2}$. Assume $\vec{x}_1$ is the nearest neighbor of some $\vec{y}$ in $\X$. Under the orthogonality assumption~\ref{assump}, moving any point $\vec{y}$ towards the other modality by a global translation $\vec{g}=\mu_y-\mu_x$ increases robustness: the probability that the nearest neighbor of $\vec{y}$ in $\X$ does not change after adding noise to $\X$, with covariance $\sigma^2 I$ and zero empirical mean.
\end{theorem}
\begin{proofsketch}
    The proof follows the intuition presented in \cref{robustness_illustration} - as the query vector $\vec{y}$ is closer to modality $\X$, a larger rotation of the separating hyperplane is required to decrease robustness. Therefore the theorem holds as subtracting the global gap vector decreases the distance between $\vec{y}$ and $\X$. 
\end{proofsketch}

For full proof and generalization for the cases of $|\X|>2$ and general noise we refer the reader to \cref{appendix:proofs}. \Cref{theorem:robust} proves that under the orthogonality assumption, when moving one modality towards the other by translating it with the global gap vector then the embedding spaces become more robust and less sensitive to noise. 
This provides us with clear motivation to close the gap.

We note that while the theorem is phrased in terms of the global gap vector, any vector that is orthogonal to modality $\X$ will suffice, as long as translating $\vec{y}$ with it will decrease the distance between $\vec{y}$ and the mean of modality $\X$.

\begin{figure}[t]
  \centerline{
    \includegraphics[width=0.8\columnwidth]{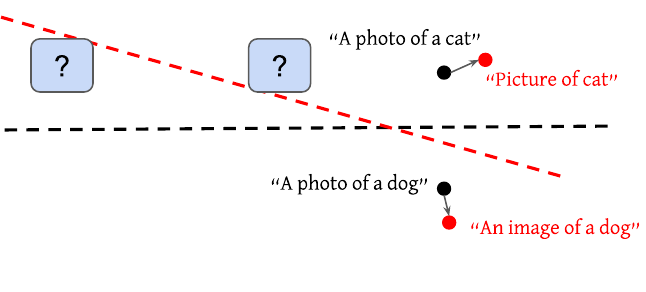}}
  \caption[]{Illustration of the relationship between robustness and the modality gap. Two image embeddings (blue squares) are classified identically when using the black texts. However, once slight noise is added to the texts (red), only the image with a smaller "gap" will maintain its classification. }
\label{robustness_illustration}
\end{figure}



In real applications, robustness is not the only desirable property of a system and a method that increases robustness while significantly decreasing the clean accuracy, will be of limited interest. 
We therefore strive to understand when does closing the gap have no effect on downstream performance, therefore solely improving robustness. Since most downstream tasks such as zero shot classification or text to image retrieval are based on cross modal nearest neighbors, we ask when can we guarantee that  the nearest neighbors will not change when varying the gap.

\begin{theorem} \label{theorem:loss_invariance}
    As in \cref{theorem:robust}, let $\vec{v}$ be some vector that is orthogonal to the affine subspace containing the modality $\X$. 
    Then translating the embeddings of modality $\X$ by $\alpha\cdot \vec{v}$ for any $\alpha\in \R$, such that $\forall x_i\in \X: x_i \assign x_i + \alpha\cdot \vec{v} $, does not change the nearest neighbor of  $\vec{y}$ in $\X$ for any $\vec{y}\in\Y$.
\end{theorem}
\begin{proofsketch}
    Since the vector $\vec{v}$ is orthogonal to $\X$, the variance of modality $\X$ in $\vec{v}$ is zero, therefore moving $\X$ by $\alpha\cdot\vec{v}$ changes the distance from any $\vec{y}\in\Y$ to all $\vec{x}\in \X$ by the same constant, meaning cross modal nearest neighbors are preserved. For full proof see \cref{appendix:proofs}. 
\end{proofsketch}

Under the orthogonality assumption~\ref{assump} such a nearest neighbor preserving direction is simply the gap vector $\vec{g}$. Therefore \cref{theorem:loss_invariance} ensures us that narrowing the gap along this direction would not change performance of the model on downstream tasks such as zero shot classification. 

\section{Algorithm}\label{section:algo}

From \cref{theory} we conclude that although the contrastive loss maintains the gap created at initialization throughout the training, it is desirable to close the gap in order to improve robustness. Theorems~\ref{theorem:robust} and~\ref{theorem:loss_invariance} provide us with a tool to do so - move the modalities towards each other along the global gap vector, as under assumption~\ref{assump} this provably does not affect performance on downstream tasks when there is no noise at all. 

Since in practice the gap vector is  not completely orthogonal to the modality subspaces (see \cref{appendix:assumptions}), we suggest projecting the gap vector to be exactly orthogonal to the affine subspace of the modality being retrieved before closing the gap. This is done by computing the principal components $V$ of the modality being retrieved (texts in case of classification) and projecting the gap $\vec{g}$ to the orthogonal complement of the components:
\begin{equation}\label{algo}
    \vec{g}' \assign \vec{g} - VV^T\vec{g}
\end{equation}
Afterwards, the modality being retrieved is moved towards the mean over the other modality by translating $\Y \assign \Y - \vec{g}'$, therefore decreasing the distance of any query point to the retrieved modality, as per \cref{theorem:robust}. 
\Cref{theorem:loss_invariance} assures us that closing the gap in the direction of $\vec{g}'$ would not affect performance on any downstream tasks that relies on computing nearest neighbors. In the appendix, we also provide an extension to this algorithm that allows moving in directions that are not perfectly orthogonal to each modality (see \cref{appendix:approx_algo}).

Contrary to previous attempts at closing the gap, our method requires no retraining or finetuning of the models \citep{oh2023geodesic,mistretta2025cross}, or the training of new ones (such as the prior network in \citep{dalle2,Patel_2024_CVPR}). Our method is also general and does not adhere to a specific setting \citep{liao2025multimodal} - any task using cross modal nearest neighbors could benefit from it.


\section{Experiments}\label{section:experiments}
\begin{figure}[t!]
     \centering
   \begin{subfigure}[t]{\figwidth\linewidth}
    \includegraphics[width=\linewidth]{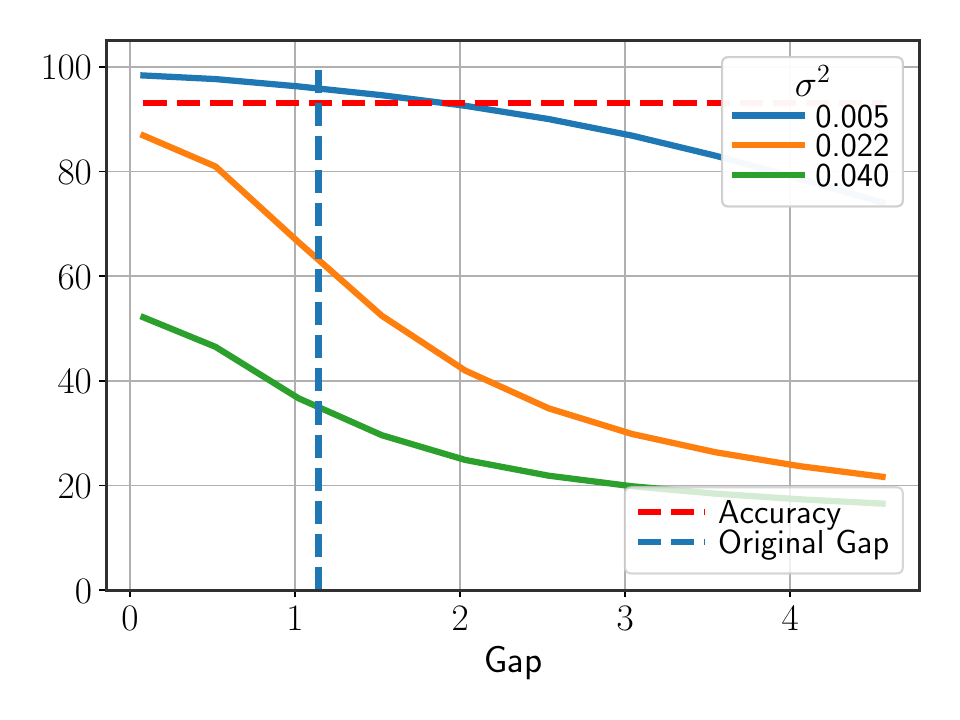} 
     \caption{CLIP (ViT L-14) on CIFAR10}
     \label{fig:siglipcifar}
 \end{subfigure}
 \hfill
 \begin{subfigure}[t]{\figwidth\linewidth}
  \includegraphics[width=\linewidth]{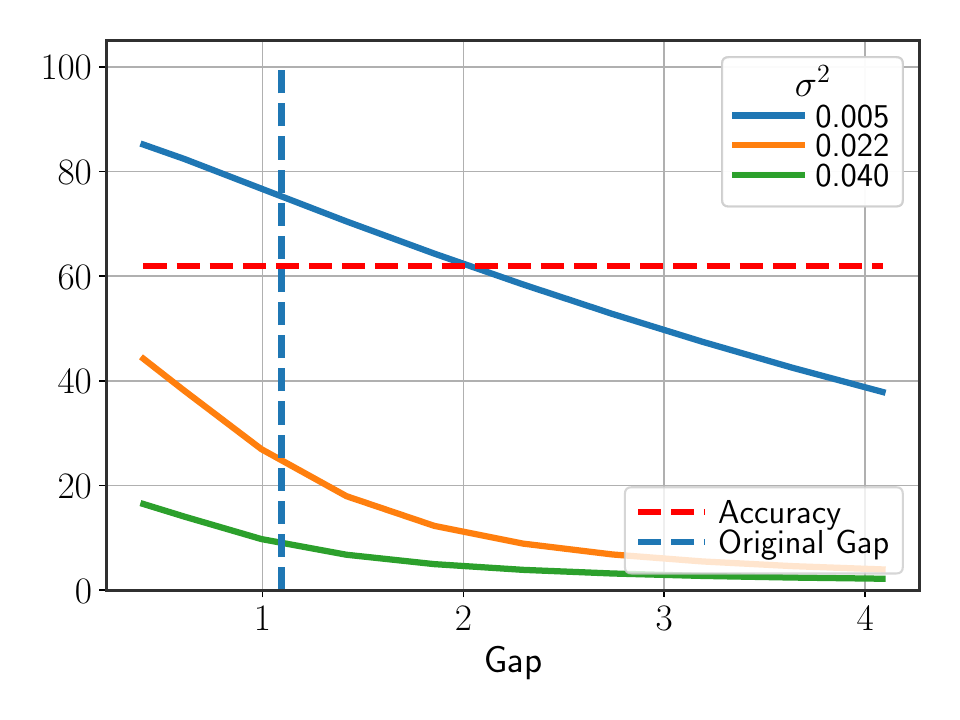} 
  \caption{CLIP (ViT B-16) on CIFAR100}
  \label{fig:clipcifar100}
  \end{subfigure}
   \ifthree{
 \hfill
 \begin{subfigure}[t]{\figwidth\linewidth}
 {{\label{fig:cliportho}\includegraphics[width=\linewidth]{Figs/rob_acc_exact/svhn_robustness_ViT-B-16-SigLIP-
     webli.pdf} }
 \caption{SigLIP on SVHN}
 }
 \end{subfigure}
 }
 \ifthree{
      \hfill
      \begin{subfigure}[t]{\figwidth\linewidth}
       {\label{fig:vig14vqa}\includegraphics[width=\linewidth]{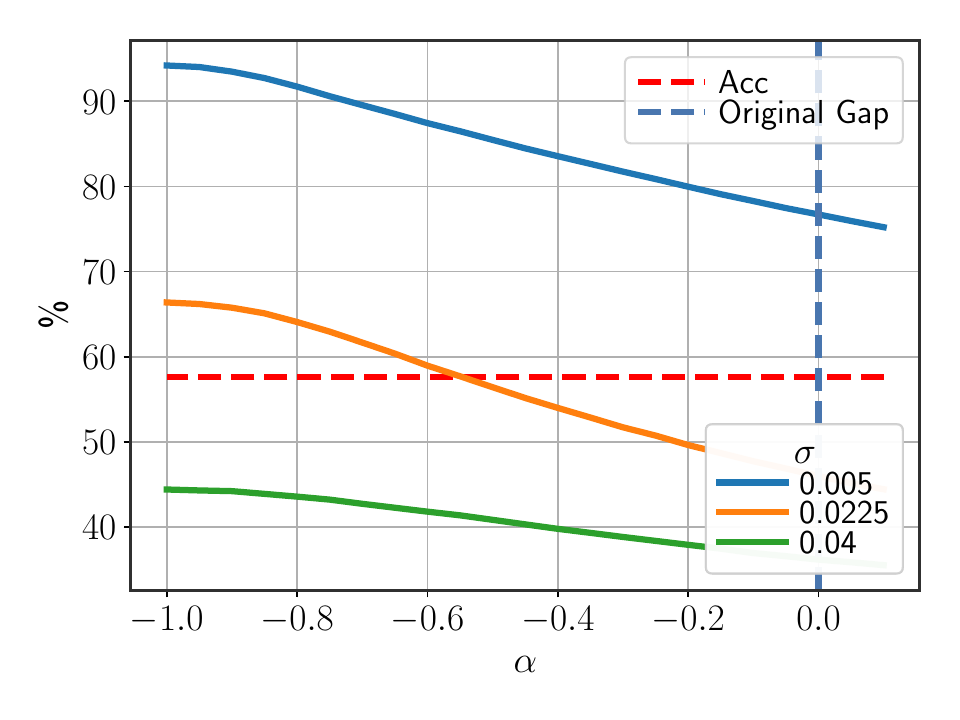} }%
       \caption{CLIP (ViT L-14) on A-OKVQA}
       \end{subfigure}
  }
 \hfill
 \begin{subfigure}[t]{\figwidth\linewidth}
  {\label{fig:vit16vqa}\includegraphics[width=\linewidth]{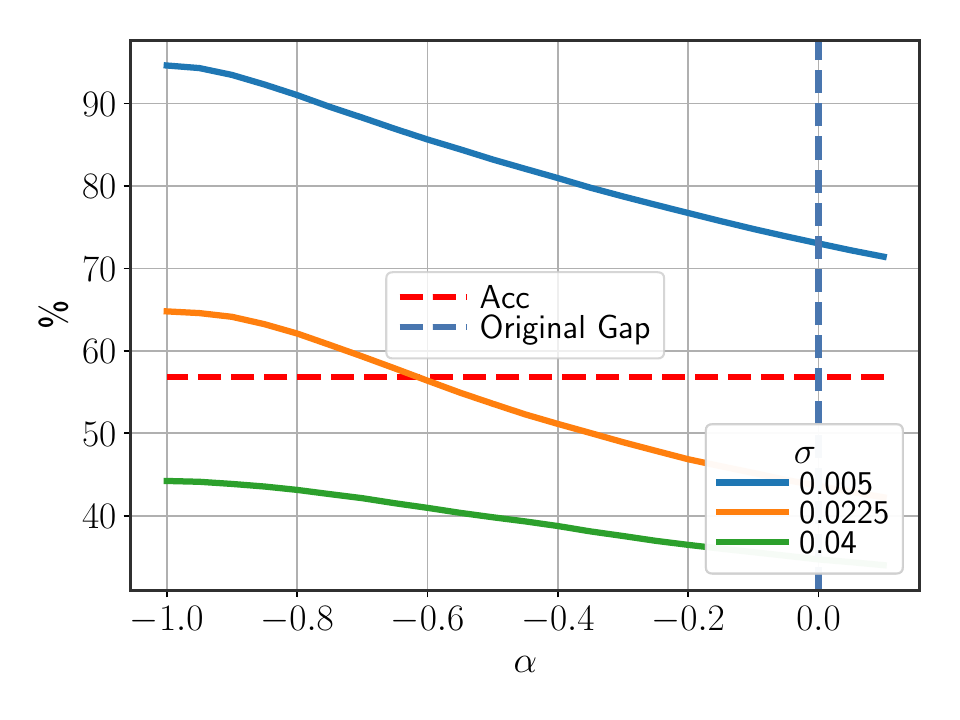}}
  \caption{CLIP (ViT B-16) on A-OKVQA}
   \end{subfigure}
  \hfill
  \begin{subfigure}[t]{\figwidth\linewidth}
  {\label{fig:siglipvqa}\includegraphics[width=\linewidth]{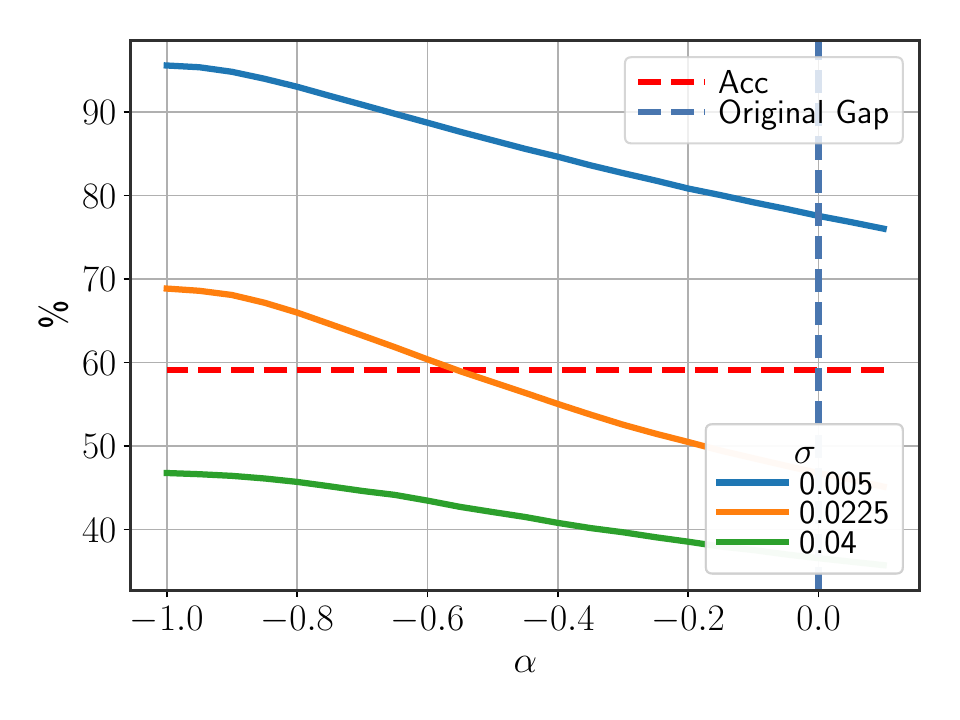}}
\caption{SigLIP on A-OKVQA}
 \end{subfigure}

    \caption[]{The zero shot classification accuracy, multiple choice VQA accuracy and robustness under noise $\eta\sim \mathcal{N}(0,\sigma^2 I)$, on various models and datasets. As predicted by our theory, robustness monotonically increases when the gap is closed while accuracy is maintained.  }
    \label{fig:robustness}
\end{figure}
We turn our attention to understanding how well our theory and algorithm work in practice for various real world models and datasets. We follow \citep{liang2022mind} and experiment on standard zero shot classification and image-text retrieval tasks. We create embeddings of commonly used, pretrained, multi-modal models using  using the openclip library \citep{openclip}.

To measure robustness, we empirically estimate:

\begin{align}
\label{eq:empirical_rob}
    R(\alpha)= \frac{1}{K}\sum_{i=1}^K \frac{1}{|\Y|}\sum_{y\in \Y} \mathds{1}_{\text{NN}(y,\X+\alpha\cdot\vec{g}) = \text{NN}(y,(\X + \epsilon_i)+\alpha\cdot\vec{g}}
\end{align}

Where $K$ is the number of samples of noise, $\epsilon_i \sim \mathcal{P}$ is a sample from the noise model $\mathcal{P}$ (such as quantization or gaussian noise) and $\alpha \in \R$ is a scalar controlling the amount of the gap that we close. We set $K=1$ (e.g. when there is a single deterministic quantization) unless stated otherwise.

\subsection{Robustness Under Controlled Noise}\label{section:gaussian_experiments}

Our first set of experiments attempts to verify our theory and algorithm in the presence of controlled noise added to the embedding space. We add varying degrees of independent Gaussian noise with zero mean and variance $\sigma^2$ to each text embedding $\vec{x}\in\X$ and translate the text embeddings by $\alpha\cdot \vec{g}'$ according to our algorithm (\cref{section:algo}). We do this for different values of $\alpha$ therefore expanding or closing the gap. For each change in the gap, we measure both accuracy and  robustness (\cref{eq:empirical_rob}) averaging over $K=100$ samples of noise.


We test our algorithm on various models, in both zero shot classification and visual question answering (VQA) on the A-OKVQA dataset \citep{okvqa}, following the evaluation protocol of \citet{clipvqa} for CLIP-like models. Results are displayed in \cref{fig:robustness} - when applying our algorithm under Gaussian noise robustness greatly increases when closing the gap across all models and tasks, and greatly decreases when expanding the gap. As noise increases, robustness is lower but still improves when the gap is closed. See \cref{appendix:other_noises} for similar results on other noise distributions.

As predicted by \cref{theorem:loss_invariance}, the zero shot and VQA accuracy do not change when closing the gap in the direction orthogonal to the affine subspace of texts.

\subsection{Robustness to Quantization}

\begin{figure}[t!]
 \centering
 \begin{subfigure}[t]{\figwidth\linewidth}
 {\label{fig:siglipcifarquant}\includegraphics[width=\linewidth]{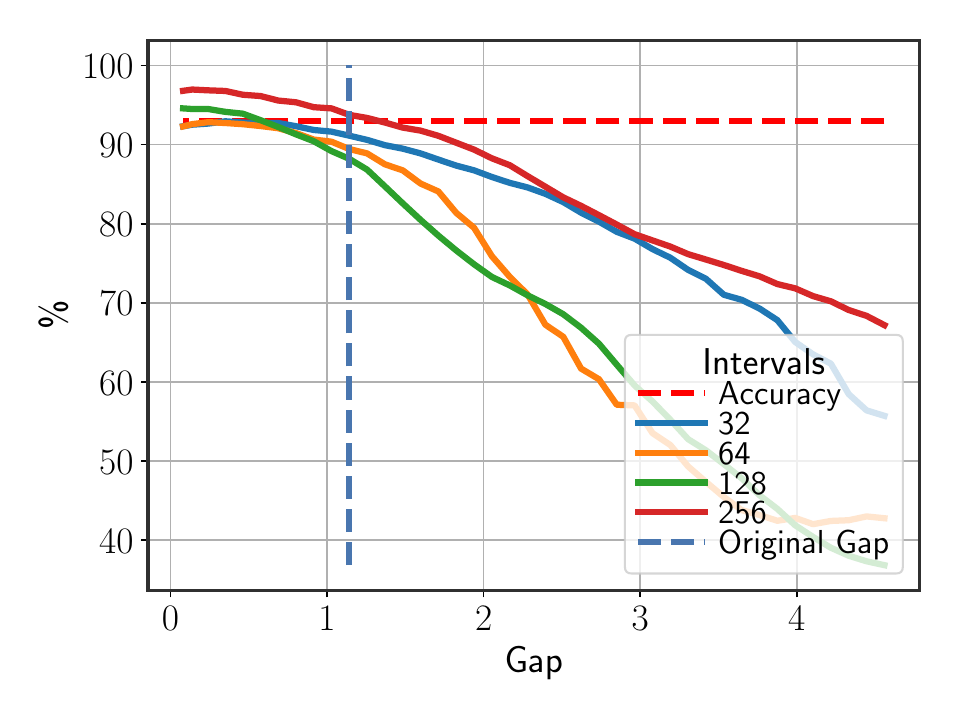} }%
  \caption{CLIP (ViT L-14) on CIFAR10}
        \end{subfigure}
 \hfill
 \begin{subfigure}[t]{\figwidth\linewidth}
  {\label{fig:clipcifar100quant}\includegraphics[width=\linewidth]{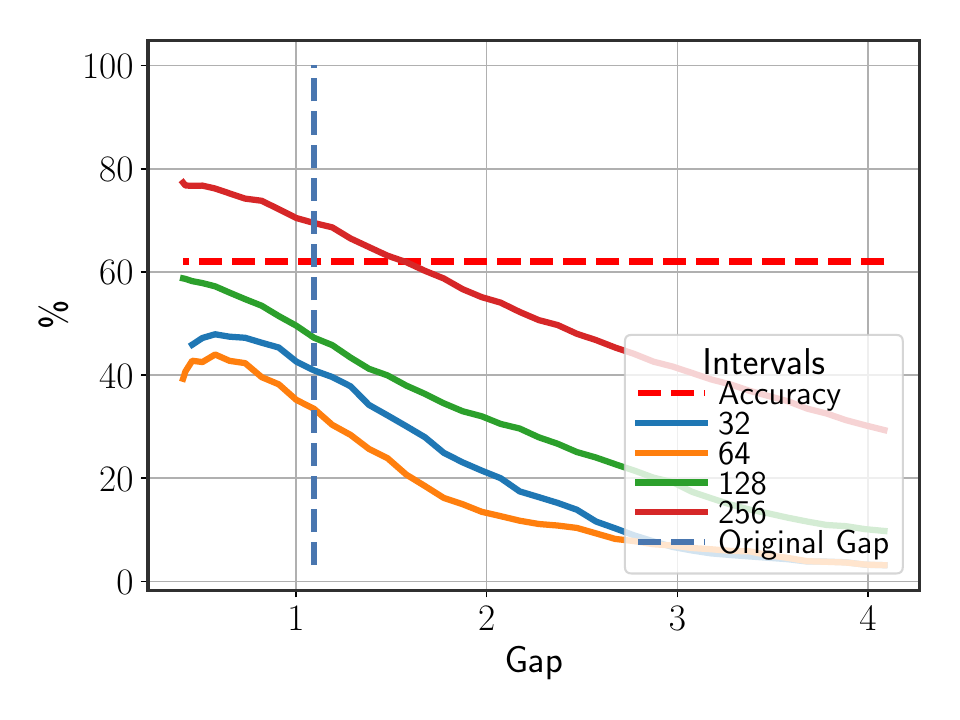}}
  \caption{CLIP (ViT B-16) on CIFAR100}
         \end{subfigure}
  \ifthree{
      \hfill
      \begin{subfigure}[t]{\figwidth\linewidth}
      {\label{fig:clipsvhbquant}\includegraphics[width=\linewidth]{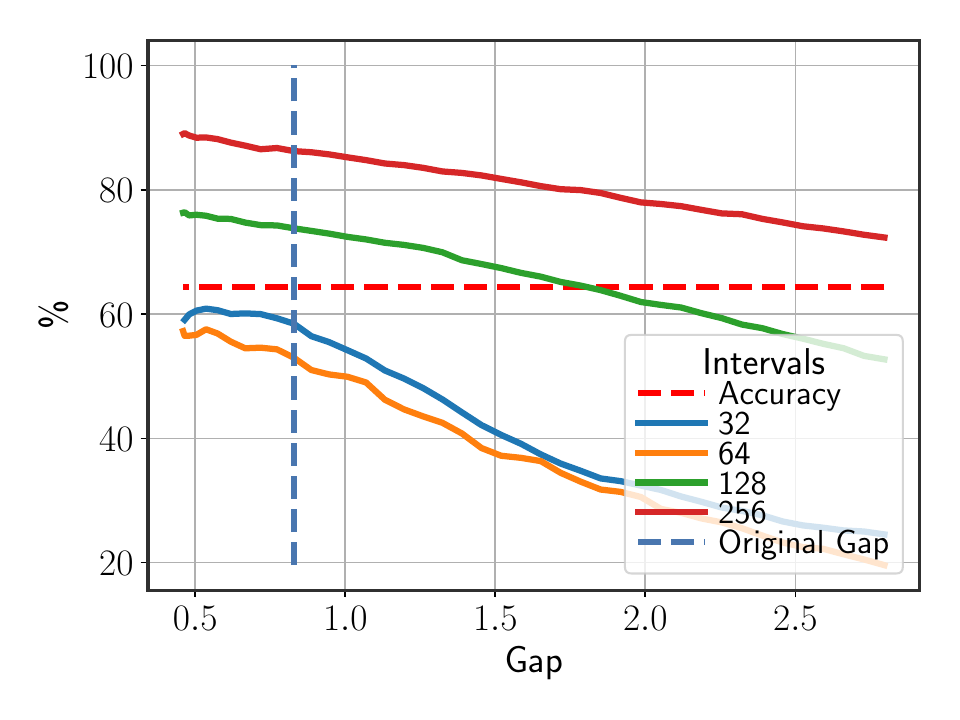}}
      \caption{CLIP (ConvNeXt) on CIFAR100}
             \end{subfigure}
  }
  \caption[]{CLIP variants are increasingly more robust to quantization as the gap is closed. The gap norm is measured in the non-quantized space.}
    \label{fig:quant}
\end{figure}

We continue and test our theory and algorithm on noise distributions that apply to real world applications. A setting of uncorrelated noise is that of quantization \citep{quantization_uncorrelated}. While many forms of model quantizations exist, we limit ourselves to the simple case of post-training quantizing of the embedding vectors alone, a setting relevant to tasks in which the embeddings were calculated in advance e.g. retrieval augmented generation (RAG) \citep{llmrag}. We close the gap to varying degrees and then quantize the text and image embedding vectors to different numbers of intervals between a constant range of $\sqrbr{-3,3}$  (as all embeddings are on the unit hypersphere), and measure robustness.

Empirically, we find that since quantization noise does not always have zero mean, it is better to close the gap until the gap is minimized in the \textit{quantized} embedding space rather than the original one. \Cref{fig:quant} displays results for these experiments. As can be seen, robustness is greatly improved when closing the gap before quantizing the embedding vectors.

\subsection{Robustness to Rephrasing} \label{subsection:rephrasing}
We now turn our attention to text rephrasings - noise that is added in the \textit{input} space rather than embedding space - which we've investigated thus far. We find that in this setting our theoretical assumptions do not hold - specifically, the noise in the embedding space that results from the rephrasing has non-zero mean and tends to lie entirely within the text subspace, with little to no variance in the direction of the gap (and see \cref{appendix:rephrasing}). This is not surprising as models that are trained on large datasets are expected to be invariant to various input noises, such as text rephrasing. 

We find that our algorithm can indeed be used for such cases. For each question in the A-OKVQA dataset, we simulate a sample of random noise by sampling 3 random wrong answers and a random correct answer, exact details are provided in \cref{appendix:rephrasing}. Since in this setting the answers are randomly selected, we focus on measuring accuracy under noise instead of robustness, while closing the gap using our algorithm. Results are presented in \cref{fig:vqa_rephrasing} -  accuracy under rephrasings greatly increases when the gap is closed using our algorithm, while clean accuracy is unaffected.

\begin{figure}
  \centering
  \begin{subfigure}[t]{\figwidth\linewidth}
    {\includegraphics[width=\linewidth]{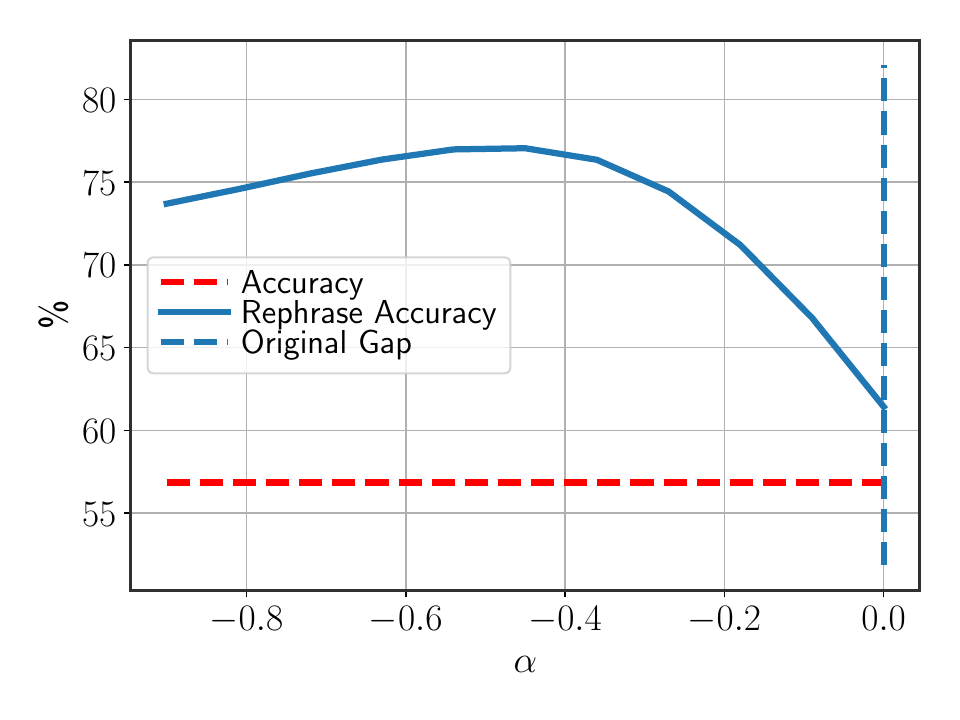} }
    \caption{CLIP (ViT B-16)}
  \end{subfigure}
  \ifthree{
 \hfill
 \begin{subfigure}[t]{\figwidth\linewidth}
    {\includegraphics[width=\linewidth]{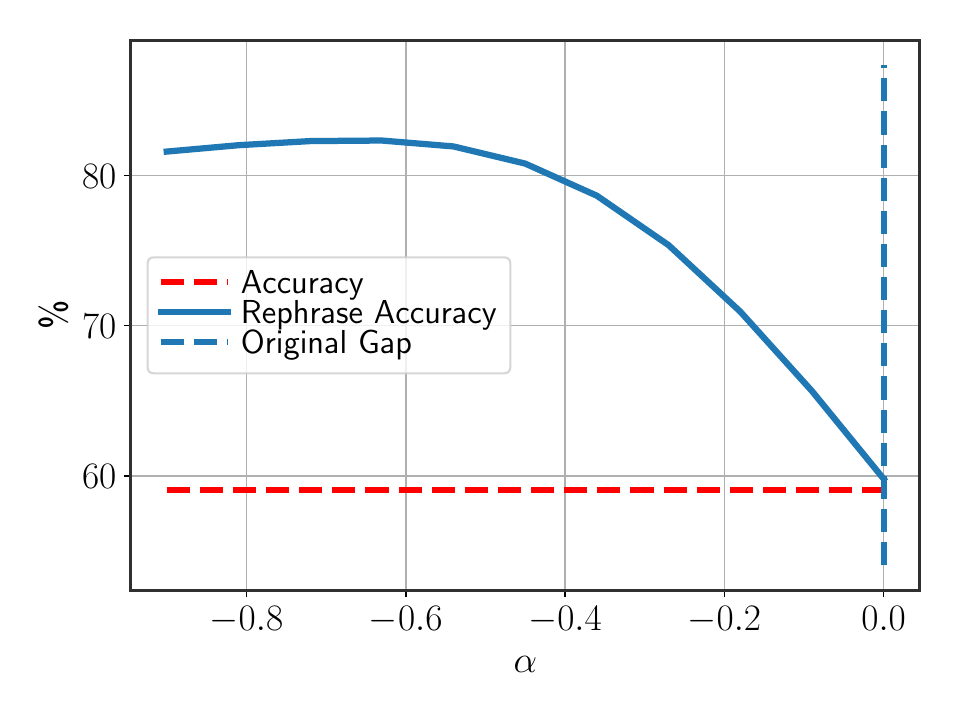} }
    \caption{SigLIP}
  \end{subfigure}
  }
 \hfill
 \begin{subfigure}[t]{\figwidth\linewidth}
    {\includegraphics[width=\linewidth]{Figs/finalzqa/ViT-B-16-SigLIP-webli_Avg_no_rob.pdf} }
    \caption{SigLIP}
  \end{subfigure}
  
  \caption[]{We apply rephrasing noise to textual inputs on the A-OKVQA dataset, and measure average accuracy and rephrase accuracy in the noisy setting over $K=500$ samples of wrong and correct answers for each question. Models are increasingly more accurate once the gap is closed, without change in clean accuracy, even though this setting is far from our theoretical assumptions.}
    \label{fig:vqa_rephrasing}
\end{figure}

\section{Discussion}
In this paper, we have attempted to provide a comprehensive theory for the formation of an orthogonal gap in multi-modal models. Specifically we showed that when initializing with tight clusters, minimizing the contrastive loss will lead to solutions where a global  gap vector separates the two modalities and that this vector will be orthogonal to both modalities. We continued to examine the effect of such a gap  in the context of robustness. We prove that in this context, the gap is indeed a bug - as the modality gap increases, the consistency of the retrieval when applying subtle variations to the embedding space is decreased.

Following this theoretical insight, we presented a simple and computationally inexpensive algorithm that improves robustness with hardly any reduction in clean accuracy. Our algorithm could easily be implemented as a post processing step when using multi-modal model embeddings.

Finally, we began examining the effect of our theory and algorithm in the context of real noise, such as quantization and text rephrasing.  We view our experiments as a hint to the potential of our algorithm and encourage further investigation as to settings of noise in which our simple algorithm could improve downstream performance.

\newpage
\section*{Acknowledgments}
This research was supported by a joint Hebrew University IBM Research grant. YW thanks the ISF and the Gatsby Foundation for their financial support.
{
    \small
    \bibliographystyle{ieeenat_fullname}
    \bibliography{bibli}
}
\clearpage
\appendix
\clearpage
\setcounter{page}{1}
\maketitlesupplementary

\section{Theorem Proofs}
\label{appendix:proofs}
\subsection{Proof of \Cref{theorem:newtheory}}

\begin{proof}
Calculating the gradient of the loss using the notation in \cref{mmloss}:
\begin{align}\label{oggrad}
    \nabla_{y_i} \mathcal{L} &= \frac{2}{\tau} [ 2(x_i - y_i) \notag \\
    &- \sum_{k=1}^N (Q^x(k, i) + Q^y(k, i))(x_k - y_i) ]
\end{align}

Defining $S_i = \sum_k Q^x(k, i)$ and $\tilde{x}_i$:
\begin{equation}
    \tilde{x}_i = \frac{\sum_k (Q^x(k, i) + Q^y(k, i))x_k}{S_i + 1}
\end{equation}
Inputting into \cref{oggrad}:
\begin{equation}
    \nabla_{y_i} \mathcal{L} = \frac{2}{\tau} \left[ 2(x_i - y_i) - (S_i + 1)(\tilde{x}_i - y_i) \right]
    \label{gradient-eq}
\end{equation}

Assume $\forall j: \norm{x_j-\mu_x}\leq \epsilon$ and $\forall i: \norm{y_i-\mu_y}\leq \epsilon$ . 
Substituting into \cref{gradient-eq}:
\begin{align}
    \nabla_{y_i} \mathcal{L} &= \frac{2}{\tau} [ 2(x_i - \mu_x +\mu_x - y_i) \notag \\
    &- (S_i + 1)(\tilde{x}_i -\mu_x + \mu_x - y_i) ] \notag \\
    &= \frac{2}{\tau} [ 2(x_i - \mu_x )+2(\mu_x - y_i) \notag \\
    &- (S_i + 1)(\tilde{x}_i -\mu_x) -(S_i+1) (\mu_x - y_i)  ]  \notag \\
    &=   \frac{2}{\tau} \sqrbr{(1-S_i) (\mu_x - y_i) + O(\epsilon)}
\end{align}
From our assumption $y_i=\mu_y+O(\epsilon)$ and that  $\norm{\mu_x - \mu_y} \gg \epsilon$:
\begin{equation}
    \nabla_{y_i} \mathcal{L} = \frac{2}{\tau} \sqrbr{(1-S_i) (\mu_x - \mu_y) + O(\epsilon)}  
\end{equation}
Meaning that the gradient of $y_i$ is approximately along the direction of the gap $\vec{g}=\mu_x - \mu_y$. 

The last thing left to prove is that the specific direction along the gap shrinks the variance in the gap. We'll show that:
\begin{equation}
 \label{eq:linear-approximation}
    S_i \approx 1 + \frac{2}{\tau}(\mu_x - \mu_y) \cdot (y_i - \mu_y)
\end{equation}
Meaning that each $y_i$ moves towards $\mu_y$ along the gap, thus shrinking the variance. 

Let $k(x,y)=e^{-\norm{x-y}^2/\tau}$ be the Gaussian kernel. Under our assumption, for $\epsilon\rightarrow 0$ we get that $x_i\rightarrow \mu_x$ and $y_j\rightarrow \mu_y$. Thus we can use a (linear) Taylor expansion of the kernel around the values of the means:
\begin{equation}
    k(x_k, y_i) \approx k(x_k, \mu_y) + \nabla_y k(x_k, y)\big|_{y=\mu_y} \cdot (y_i - \mu_y)
\end{equation}
The gradient of the kernel is $\nabla_y k(x, y) = \frac{2}{\tau}(x - y)k(x, y)$. Substituting this into the expansion:
\begin{equation}
    k(x_k, y_i) \approx k(x_k, \mu_y) \left[ 1 + \frac{2}{\tau}(x_k - \mu_y) \cdot (y_i - \mu_y) \right]
\end{equation}
Substituting the linearized kernel into the expression for $Q^x(k, i)$, the common factor $k(x_k, \mu_y)$ cancels from the numerator and denominator:
\begin{equation}
    Q^x(k, i) \approx \frac{1 + \frac{2}{\tau}(x_k - \mu_y) \cdot (y_i - \mu_y)}{\sum_{j=1}^N \left[ 1 + \frac{2}{\tau}(x_k - \mu_y) \cdot (y_j - \mu_y) \right]}
\end{equation}
The denominator sum simplifies significantly:
\begin{eqnarray}
    \sum_{j=1}^N \left[ 1 + \frac{2}{\tau}(x_k - \mu_y) \cdot (y_j -\mu_y) \right] \notag \\
    = N + \frac{2}{\tau}(x_k - \mu_y) \cdot \sum_{j=1}^N (y_j - \mu_y)
\end{eqnarray}
Under the assumption that the cluster $\Y$ is centered at its mean ($\sum (y_j - \mu_y) = 0$), the denominator becomes exactly $N$.

Summing $Q^x(k, i)$ over all $k$:
\begin{equation}
    S_i \approx \frac{1}{N} \sum_{k=1}^N \left[ 1 + \frac{2}{\tau}(x_k - \mu_y) \cdot (y_i - \mu_y) \right]
\end{equation}
Applying the centering identity for $\X$ ($\sum x_k = N\mu_x$):
\begin{equation}
    S_i \approx 1 + \frac{2}{\tau}(\mu_x - \mu_y) \cdot (y_i - \mu_y)
\end{equation}

\end{proof}

\subsection{Proof of \Cref{theorem:mmloss}}
\begin{proof}

    As stated in the theorem, assume that at some iteration $t$ of gradient descent $Q^x,Q^y$ are doubly stochastic and that exists $\vec{v}$ in which both modalities have zero variance. Assume that $Q^x,Q^y$ stay doubly stochastic throughout the training. We'll begin by showing that iteration $t+1$ maintains the zero variance in direction $\vec{v}$:
By the assumption that $v^T x_i=a$ for all $x_i$ and $v^T y_j=b$ for all $y_j$ we see that using \cref{eq:grad}:
\begin{eqnarray*}
v^T \frac{\partial \mathcal{L}}{\partial y_i}
&\propto  &v^T \left( x_i - y_i \right) \\
&& - v^T \sum_{k=1}^N \frac{Q^x(k,i) + Q^y(k,i)}{2}
  (x_k - y_i) \\
  &=& 0
\end{eqnarray*}
where the last equality follows from double stochasticity.

To show that training will converge when global and local orthogonality hold, we'll change the coordinate system s.t.:
\begin{eqnarray}
    \vec{x_i} = (a,\tilde{x}_i) \notag \\
    \vec{y_i} = (b,\tilde{y}_i) \notag 
\end{eqnarray}
and minimize the loss w.r.t. $\tilde{x}_i, \tilde{y}_i$. From standard uniformity and alignment arguments, the loss will be minimal when $\tilde{x}_i = \tilde{y}_i$ and the points are uniformly distributed. Note that at such a solution all the local gap vectors will be of the form
\begin{equation}
    \forall i: \vec{g}_i = \vec{x}_i-\vec{y}_i =  (a-b,0,0,...)
\end{equation}

meaning that by definition the global gap vector will also be equal to the above (as it is the mean). Orthogonality will also hold since the gap is solely in the direction of $\vec{v}$ which is orthogonal to both modalities.
\end{proof}

\subsection{Proof of \Cref{theorem:robust}}

\begin{lemma} \label{lemma:dist}
    Under the orthogonality assumption, for every point $y\in \Y$ and $\mu_x$ the mean of modality $\X$:
    \begin{equation}
        \norm{\mu_x - (y-g)} < \norm{\mu_x - y}
    \end{equation}
    with $\vec{g}=\mu_y-\mu_x$ the global gap vector.
\end{lemma}
\begin{proof}
    Assume coordinate frame s.t. $\mu_x=0$. Therefore the claim reduces to:
    \begin{equation}
        \norm{y-g} < \norm{y}
    \end{equation}
    Under the orthogonality assumption, the points $y,g,\mu_x$ form a right angled triangle with $\measuredangle y,g,\mu_x = \pi/2$. 
    Therefore, from the Pythagorean theorem $\norm{y}^2 = \norm{g}^2 + \norm{y-g}^2$. Since $\norm{g}^2 > 0$ then:
    \begin{equation}
        \norm{y}^2 > \norm{y-g}^2
    \end{equation}
    as required.
\end{proof}
Now to prove the theorem:
\begin{proof} \label{mainproof}    
The separating hyperplane between $x_1$ and $x_2$ is characterized by the normal to the plane $w=\frac{x_1-x_2}{\norm{x_1-x_2}}$. Denote $\tilde{w}$ the normal to the separating hyperplane between the noisy versions $X$ i.e $\tilde{x}_1$ and $\tilde{x}_2$.

Since the nearest neighbor of $y$ is $x_1$ then $w^Ty > 0$. 
We wish to show that the probability of $\tilde{w}^T(y+v) > 0$ is greater than the probability that $\tilde{w}^Ty > 0$. 

Since by our assumptions the noise only rotates the hyperplane, then $\tilde{w} = R(\theta)w$ with $R(\theta)$ some rotation matrix. 

Therefore: 
\begin{equation}
    P(\tilde{w}^Ty > 0) = P(w^T (R(-\theta)y) > 0) 
\end{equation}
Notice that $w^T (R(-\theta)y) > 0$ only if $\frac{\pi}{2} - \cos^{-1} (\frac{w^Ty}{\norm{y}}) > \theta$ where $\cos^{-1} (\frac{w^Ty}{\norm{y}})$ is the angle between $w$ and $y$.

From \cref{lemma:dist}, $y-g$ has smaller norm than $y$ and due to the orthogonality assumption, $g^Tw=0$. Therefore:
\begin{gather}    
    \frac{w^T(y+g)}{\norm{y+g}} = \frac{w^Ty}{\norm{y+g}} > \frac{w^Ty}{\norm{y}} \notag \\
    \Rightarrow \cos^{-1} (\frac{w^T(y+g)}{\norm{y+g}}) < \cos^{-1} (\frac{w^Ty}{\norm{y}}) \\
    \Rightarrow \frac{\pi}{2} - \cos^{-1} (\frac{w^T(y+g)}{\norm{y+g}}) > \frac{\pi}{2} - \cos^{-1} (\frac{w^Ty}{\norm{y}}) 
\end{gather}
Therefore the event that $w^Ty > \theta$ is a strict subset of the event that $w^T(y+g) > \theta$, meaning that:
\begin{equation}
    P\parent{w^T(y+g) > \theta} > P\parent{w^T(y+g) > \theta}
\end{equation}
\end{proof}
Another way to see this is to note that we can write $\tilde{w}=w+\eta$ with $\eta$ a zero mean r.v. with covariance $2\sigma^2 I$. The retrieval is robust if $\tilde{w}^Ty > 0$. Substituting:
    \begin{equation}
        \tilde{w}^Ty = \parent{w+\eta}^T y = w^T y +\eta^Ty
    \end{equation}
    From our assumption, $ w^T y>0$, so robustness will be maintained if $ w^T y > -\eta^Ty$. From our assumptions:
    \begin{equation}
        \text{Var}(\eta^Ty) = 2\sigma^2 \norm{y}^2
    \end{equation}
    Again, according to our assumptions, replacing $y\rightarrow y+g$ maintains:
    \begin{equation}
        w^Ty = w^T(y+g)
    \end{equation}
    since $g$ is orthogonal to $w$, and from \cref{lemma:dist}:
    \begin{equation}
        \norm{y} > \norm{y+g}
    \end{equation}
    Therefore:
    \begin{gather}
         \text{Var}(\eta^T(y+g)) = 2\sigma^2 \norm{y+g}^2 < 2\sigma^2 \norm{y}^2 \notag \\
         = \text{Var}(\eta^Ty)
    \end{gather}
    meaning that:
    \begin{equation}
        P\parent{w^T (y+g) < -\eta^T(y+g)} < P\parent{w^T y < -\eta^Ty}
    \end{equation}
    Since we decreased the variance of a zero mean r.v..

\subsection{Extensions to \Cref{theorem:robust}}

For completeness, we provide two extensions for the proof, one in the case of perturbation with non-zero mean, and another in the case of a general covariance structure. 
\subsubsection{Extension to Non-Zero Mean}
Suppose that the noise can change the mean of the perturbed points. Following all notations of proof~\ref{mainproof}, the requirement for robustness is that:
\begin{equation}
    \tilde{w}^Ty > b
\end{equation}
where $2b=\tilde{x}_1^T\tilde{x}_1-\tilde{x}_2^T\tilde{x}_2$, so robustness is maintained if:
\begin{equation}
    -\eta^T y \leq w^Ty-b
\end{equation}
If we assume that $w^Ty>b$, i.e. that the margin between $y$ and the decision boundary between $x_1,x_2$ is at least $b$, then the same proof from proof~\ref{mainproof} holds.

\subsubsection{Extension to General Covariance}
Suppose that the noise has covariance $C$. Then:
\begin{equation}
    \text{Var}(\eta^Ty) = y^T(2C)y
\end{equation}
In this case we would like to chose $v$ s.t. it is orthogonal to $x_1,x_2$ (maintaining $w^Ty=w^T(y+v)$), but also satisfies:
\begin{align}
    (y+v)^TC(y+v)\leq y^TCy \notag \\
    \Rightarrow 2v^TCy+v^TCv \leq 0
\end{align}
Any direction maintaining the above will increase robustness. Notice that this direction does not necessarily point from $y$ to the origin (which in our case is the mean of modality $\X$), therefore the gap is not necessarily closed.

\subsection{Proof of \Cref{theorem:loss_invariance}}
\begin{proof}
 Observe the relative distance between some $y\in Y$ and two points in the other modality when translating modality $\X$ along the gap by $\alpha\cdot v$:
    \begin{eqnarray}\label{eq:nnloss}
        \forall x_i, x_j \in \X: \notag \\
        \norm{y-(x_i + \alpha \cdot v)}^2 - \norm{y-(x_j + \alpha \cdot v)}^2 = \notag \\
         \norm{x_i}^2  -2x_i^Ty  - 2x_i^Tv -\norm{x_j}^2  +2x_j^Ty  + 2x_j^Tv 
    \end{eqnarray}
    Since $x_i,x_j$ are embedded on the unit sphere $\norm{x_i}=\norm{x_j}=1$. Since $y$ is also on the unit sphere, $-2x^Ty = \norm{x-y}^2 - 2$. Additionally, under the orthogonality assumption $\forall x\in \X: x^Tv=c$ for some constant $c\in \R$. Substituting into equation~\ref{eq:nnloss}:
    \begin{eqnarray}
         \norm{y-(x_i + \alpha \cdot v)}^2 - \norm{y-(x_j + \alpha \cdot v)}^2 =\notag  \\
         1+\norm{x_i-y}^2  - 2c - 1 + 2c - \norm{x_j-y}^2 = \notag \\
         \norm{x_i-y}^2- \norm{x_j-y}^2
    \end{eqnarray}
    Therefore:
    \begin{eqnarray}
         \norm{y-x_i}^2 < \norm{y-x_j}^2 \Rightarrow \notag \\
         \norm{y-(x_i + \alpha \cdot v)}^2 < \norm{y-(x_j + \alpha \cdot v)}^2
    \end{eqnarray}
    Meaning the nearest neighbor structure is maintained after translating one modality along the gap.
\end{proof}

\section{Does Information Imbalance Cause the Gap?}\label{sec:appendixinformationimbalance}
It has previously been postulated that information imbalance between texts and images, i.e. when texts are much less informative that their image counterparts, plays a role in the creation of the modality gap \cite{trigger}. 

We observe a simplified setting in which a training set for a model trained with the contrastive loss contains only two images $x_1,x_2$ with the same caption $y$. We'll show that the minima of these dynamics isn't when a gap exists but when all points lie on top one another.

The loss is therefore:
\begin{align}
    \mathcal{L}(x_{1},x_{2},y)=&-\frac{1}{2}(\log\frac{e^{-d(x_{1},y)}}{e^{-d(x_{1},y)}+e^{-d(x_{1},y)}} \notag \\ 
    &+\log\frac{e^{-d(x_{2},y)}}{e^{-d(x_{2},y)}+e^{-d(x_{2},y)}} \notag \\ 
    &+\log\frac{e^{-d(x_{1},y)}}{e^{-d(x_{1},y)}+e^{-d(x_{2},y)}} \notag \\ 
    &+\log\frac{e^{-d(x_{2},y)}}{e^{-d(x_{2},y)}+e^{-d(x_{1},y)}}) \notag \\ 
=&-\log\frac{1}{2}+d(x_{1},y)+d(x_{2},y)\notag \\
&+\log(e^{-d(x_{1},y)}+e^{-d(x_{2},y)})
\end{align}
With $d = \ell_2$ distance. Taking the derivative:
\begin{equation}
    \frac{d\mathcal{L}}{dx_{i}}=2(y-x_{i})+\frac{e^{-d(x_{i},y)}}{e^{-d(x_{1},y)}+e^{-d(x_{2},y)}}\cdot2(x_{i}-y)
\end{equation}
Since $\frac{e^{-d(x_{i},y)}}{e^{-d(x_{1},y)}+e^{-d(x_{2},y)}}<1$ we get that the total gradient is $\frac{d\mathcal{L}}{dx_{i}}=c(y-x_{i})$ for some $c>0$. 
The same analysis can be done for $\frac{d\mathcal{L}}{dy}$ where we get:
\begin{align}
    \frac{d\mathcal{L}}{dy}&=2(x_{1}-y)+2(x_{2}-y)\notag \\
    &+\frac{e^{-d(x_{1},y)}\cdot2(y-x_{1})+e^{-d(x_{2},y)}\cdot2(y-x_{2})}{e^{-d(x_{1},y)}+e^{-d(x_{2},y)}}
\end{align}
Now assume by negation that the minima of these dynamics isn't when $x_1=x_2=y$. For $x_i$ it is obvious from the gradient that the only minima is $x_i=y$. For $y$, it is enough to see that $y=\frac{x_1+x_2}{2}$ is a minima by substituting into the gradient. Therefore contradicting that there is a different minima except $x_1=x_2=y$, as required. 

\section{Dimensionality Collapse and the Gap}\label{appendix:dimcollapse}
\begin{figure*}[t!]
    \centering
    \includegraphics[width=0.4\linewidth]{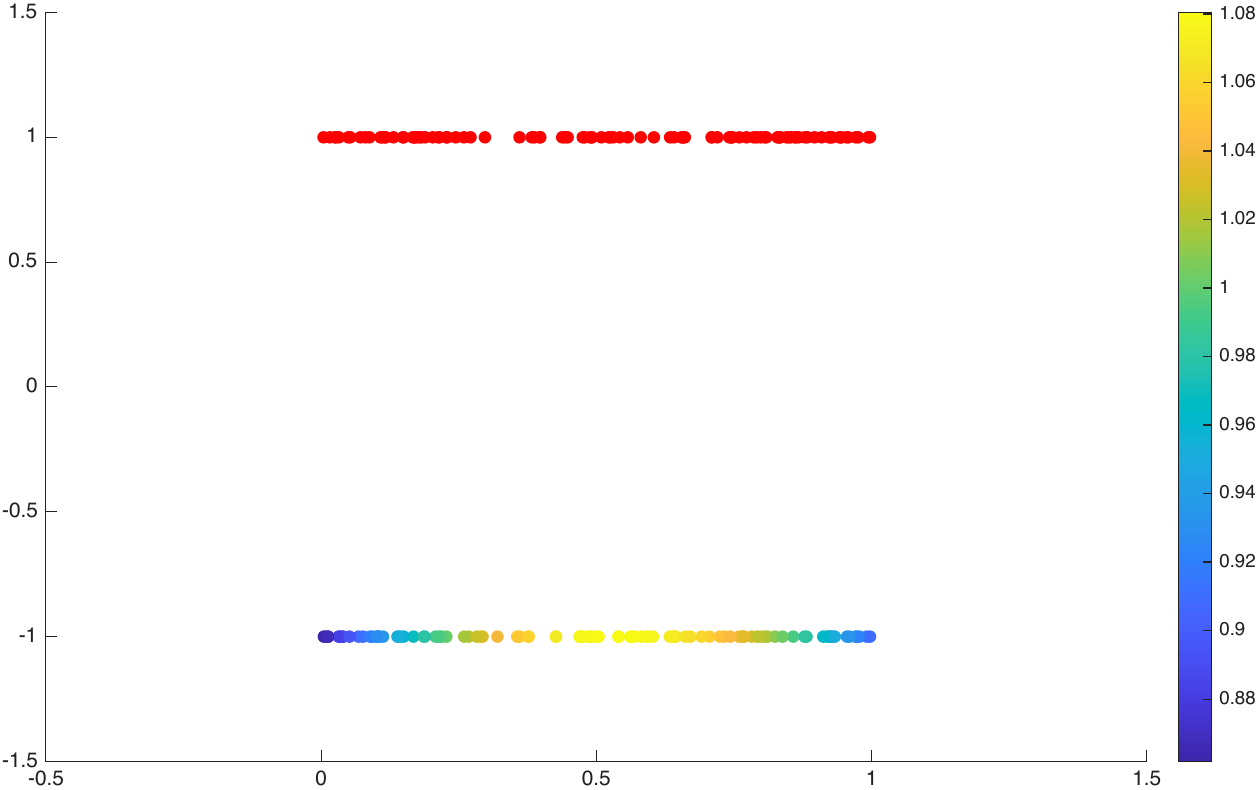}
    \begin{tabular}{ccccc}
        \small Initialization &
        \small Iteration 100 &
        \small Iteration 500 &
        \small Iteration 1000 &
        \small Iteration 5000  \\
        \includegraphics[width=0.17\linewidth]{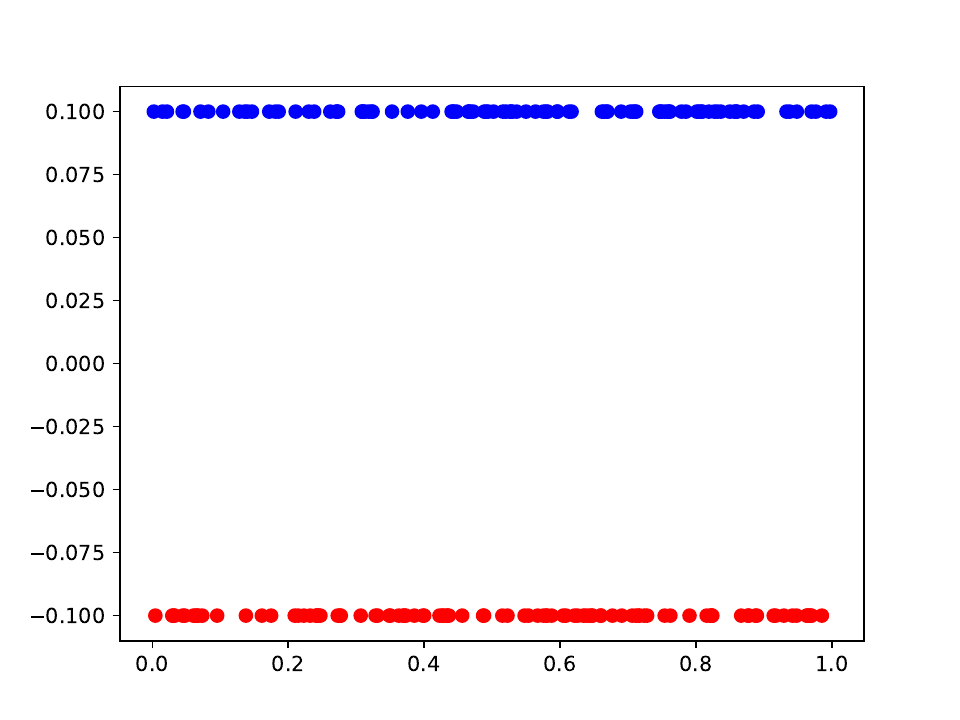} &
        \includegraphics[width=0.17\linewidth]{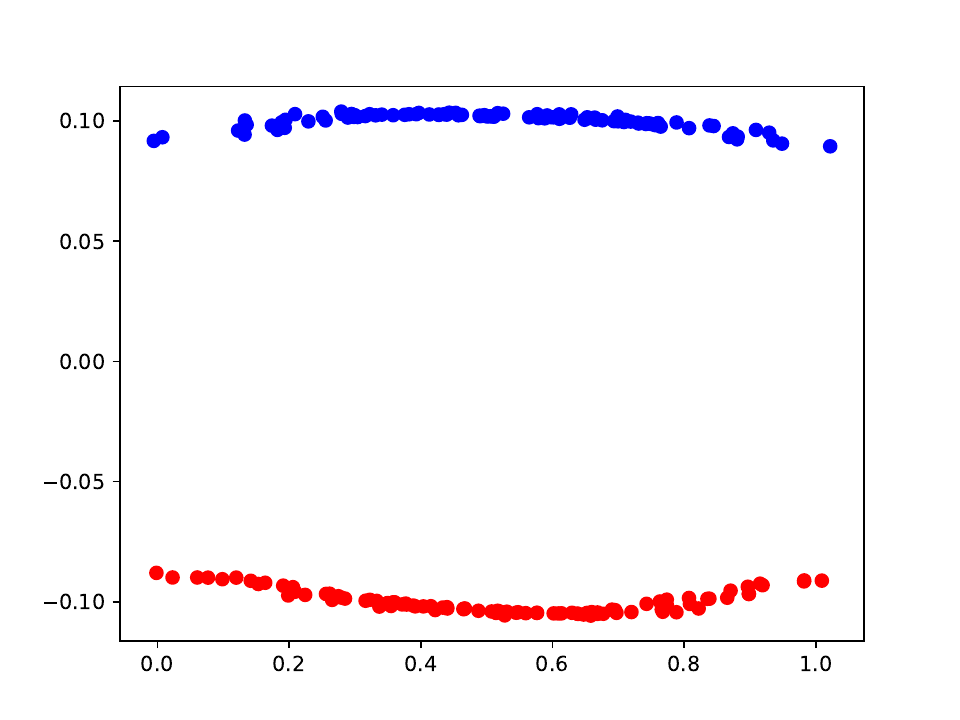} &
        \includegraphics[width=0.17\linewidth]{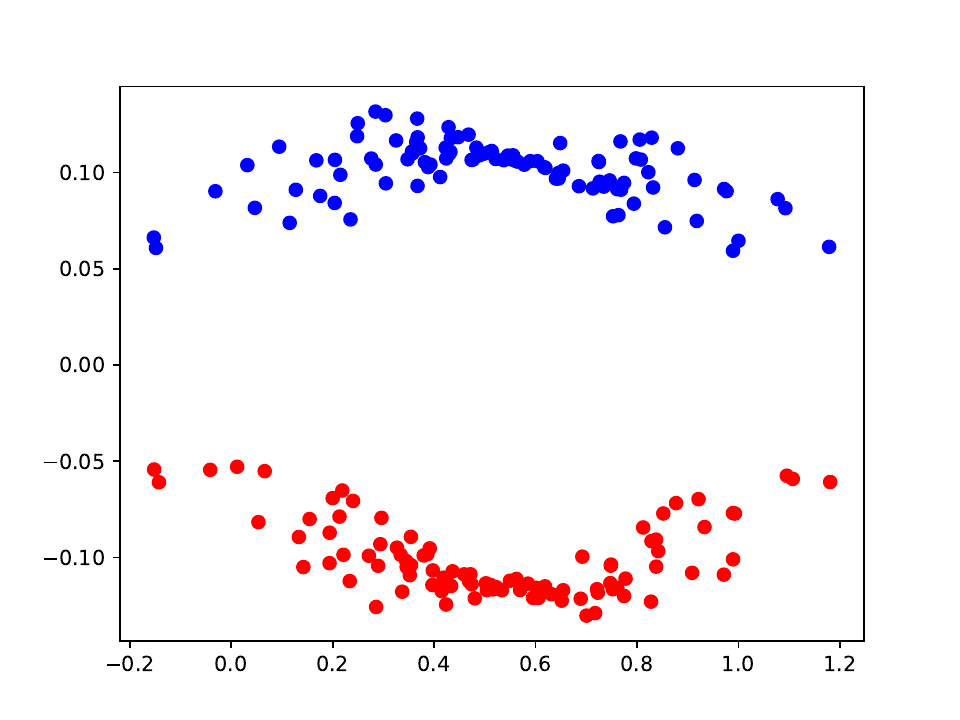}  &
        \includegraphics[width=0.17\linewidth]{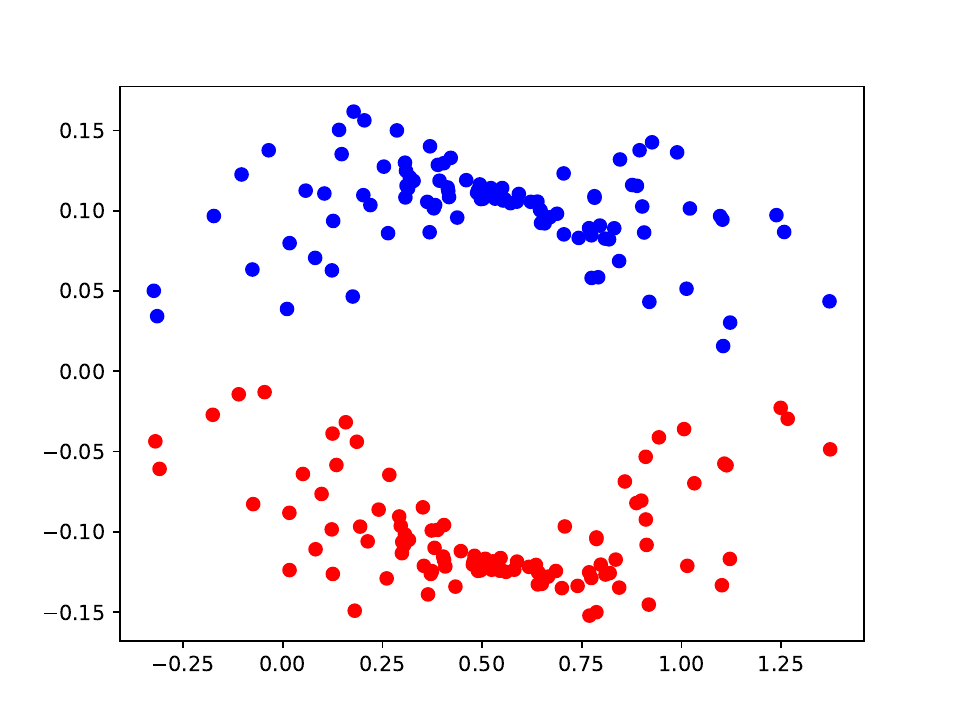} &
        \includegraphics[width=0.17\linewidth]{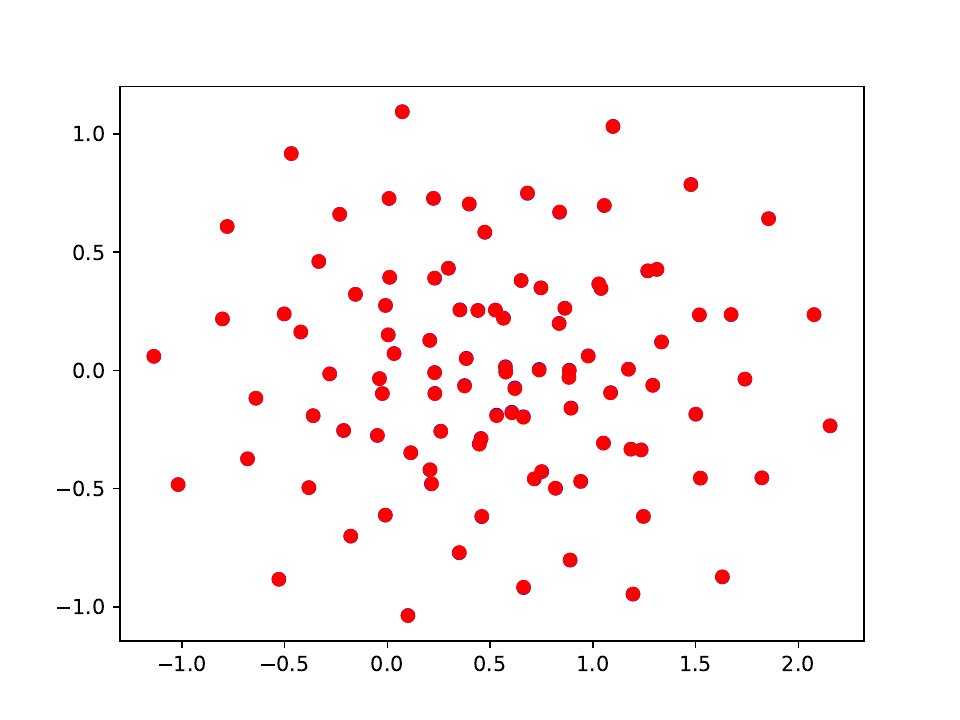}
    \end{tabular}

    \caption{{\bf (Top:} Two set of points that perfectly satisfy dimensionality collapse. The bottom points are color-coded by $S_i^y$. Note that $S_i>1$ for points near the center. {\bf (bottom:} The training dynamics. Despite no variance in the direction of the gap for both modalities, the solution converged to has no gap and is perfectly aligned. This is because in the initial iterations, points near the center are pushed away from the other modality hence destroying the original low dimensionality of each modality.}
    \label{fig:sufficient}
\end{figure*}
As mentioned in the paper, an orthogonal gap can emerge even without dimensionality collapse, as in \cref{fig:theory2} where initialization has equal variance in all dimensions. This makes it an unnecessary condition for the formation of an orthogonal gap.

But it the condition sufficient? \Cref{fig:sufficient} shows that the answer is no as well. We train again follow the simplified setting of experiment of \cref{fig:theory2}, only initializing the two modalities with complete dimensional collapse - all variance lies in a single dimension with the other containing none. In this case, training converges to a fully aligned solution with no gap, making the condition of dimensionality collapse insufficient to explain the formation of the gap. The behavior is consistent with our analysis because in the initial iterations double stochasticity does not hold. The top of \cref{fig:sufficient} color-codes the bottom points by $S_i^y$. Note that $S_i>1$ for points near the center, so at the initial iterations, points near the center are pushed away from the other modality more than points at the edges  hence destroying the original low dimensionality of each modality.

\section{Empirical Evidence for Assumptions}\label{appendix:assumptions}


\subsection{Empirical Evidence for Apprximate Double Stochasticity}\label{appendix:ds}

We repeat the experiment measuring  $S_j^x,S_j^y$ as in \cref{fig:boxplots} only for embeddings trained on the unit hypersphere in $\R^{64}$. Results are shown in \cref{fig:sihypersphere} - as training progresses, $S_j^x,S_j^y\rightarrow 1$ meaning that the matrices $Q^x,Q^y$ become more doubly-stochastic. 

\begin{figure}[h]
    \centering
  \includegraphics[width=\linewidth]{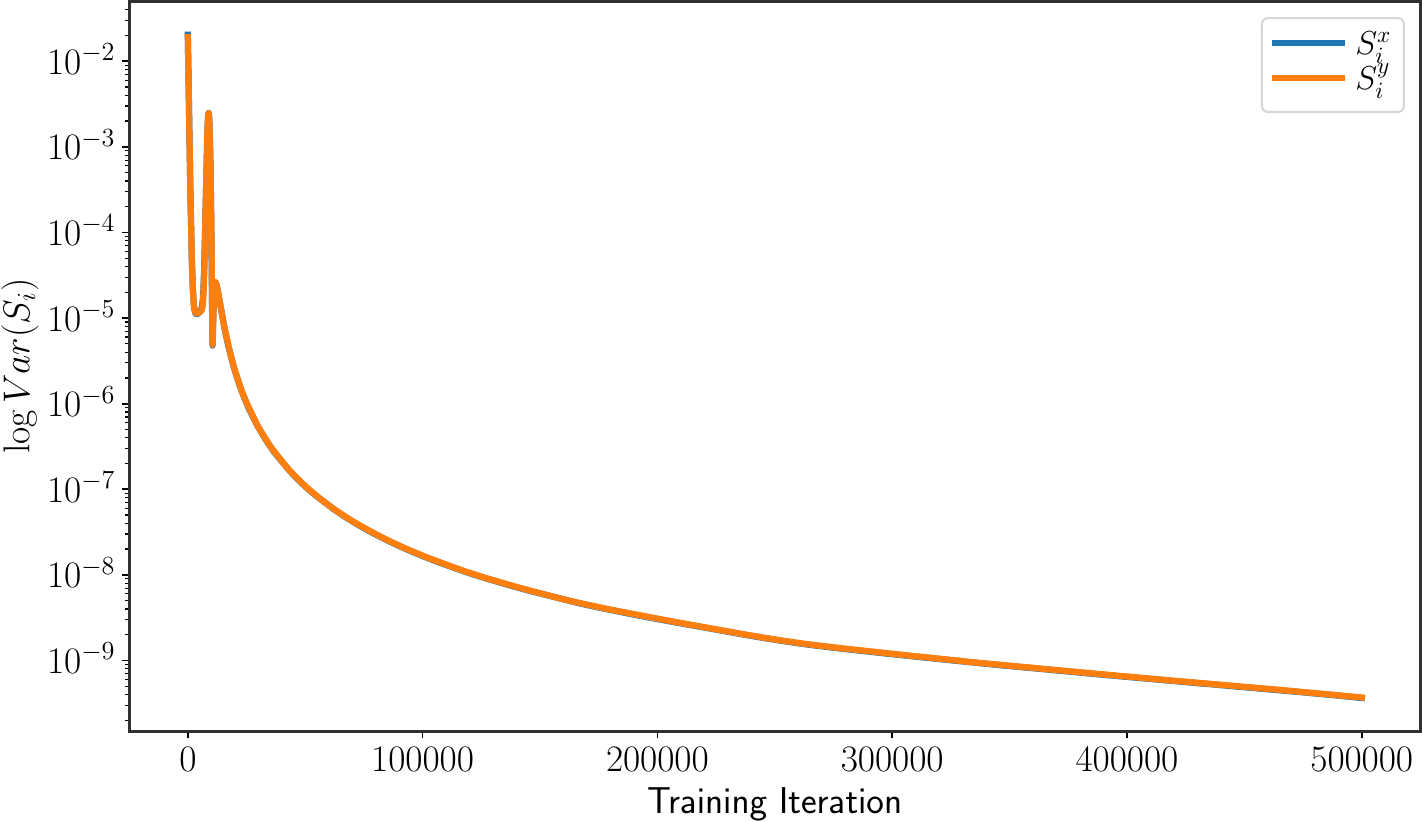} 
  \caption[]{We calculate $S_i^x,S_i^y$ throughout the training of $N=500$ embedding pairs initialized from a Gaussian distribution with variance $\sigma^2 = 0.01$ and gap of $\norm{\vec{g}} = 1.8$. As training progresses, the values of $S_i^x,S_i^y$ are concentrated around their means, which by definition equal $1$, making the matrices $Q^x,Q^y$ more doubly-stochastic.}
    \label{fig:sihypersphere}
\end{figure}

\subsection{Empirical Evidence for Orthogonality}\label{appendix:orthogonality}
For completeness we recompute results from \cite{zhang2023diagnosing} showcasing the orthogonality of the global gap vector for different models and datasets in \cref{fig:ortho}.

 \begin{figure}[h]
     \centering
   \begin{subfigure}[t]{\figwidth\linewidth}
    \includegraphics[width=\linewidth]{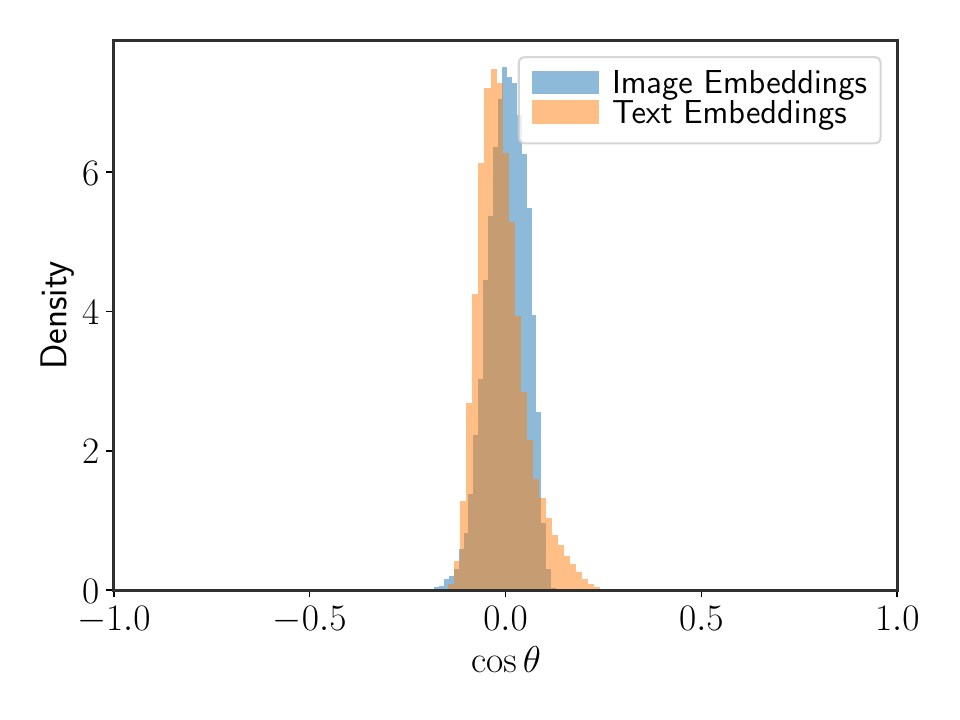} 
     \caption{SigLIP on Imagenet}
     \label{fig:cliportho}
 \end{subfigure}
 \hfill
 \begin{subfigure}[t]{\figwidth\linewidth}
  \includegraphics[width=\linewidth]{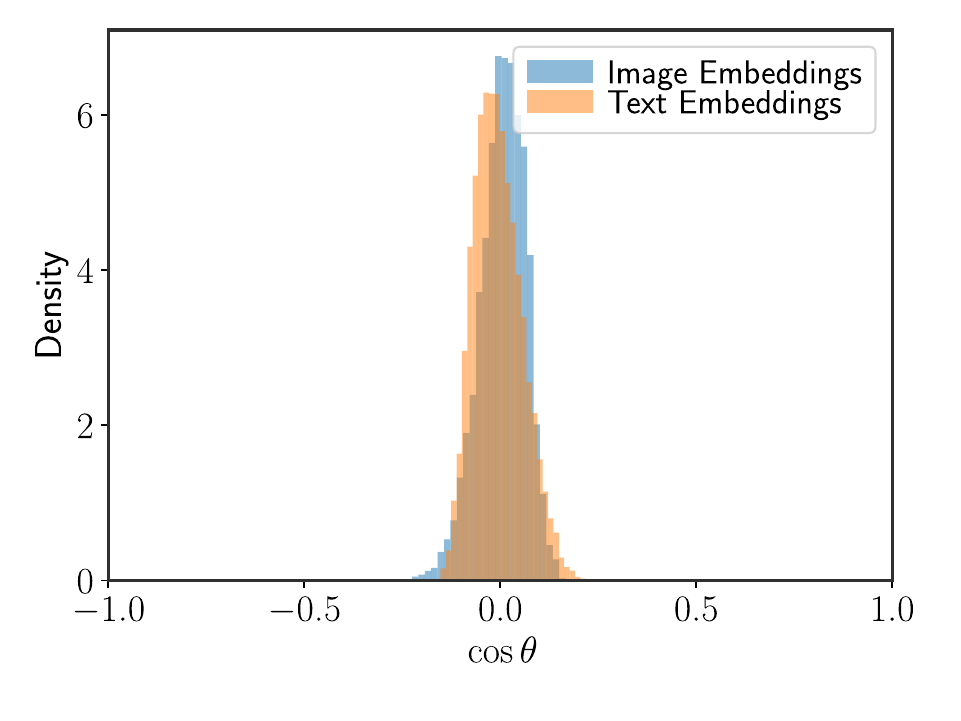} 
  \caption{CLIP on MS-COCO}
  \label{fig:cliportho}
  \end{subfigure}
   \ifthree{
 \hfill
 \begin{subfigure}[t]{\figwidth\linewidth}
 {{\label{fig:cliportho}\includegraphics[width=\linewidth]{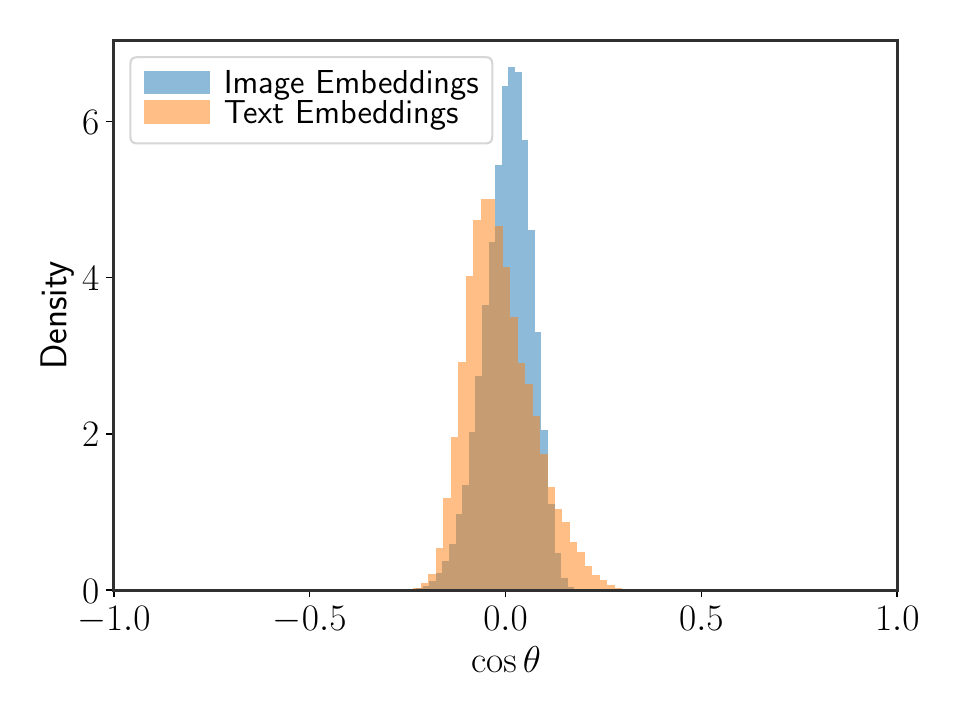} }
 \caption{CLIP (ViT-B16) on Imagenet}
 }
 \end{subfigure}
 }
  \caption[]{The orthogonality assumption - the cosine of angle $\theta$ between $\vec{g}$ and each embedding in each modality is nearly $0$ for different models and datasets, confirming assumption~\ref{assump}.} 
    \label{fig:ortho}
\end{figure}

\section{Robustness to Various Noise Distributions}\label{appendix:other_noises}
\begin{figure*}[ht!]
    \centering
    \begin{tabular}{ccc}
        \includegraphics[width=0.32\linewidth]{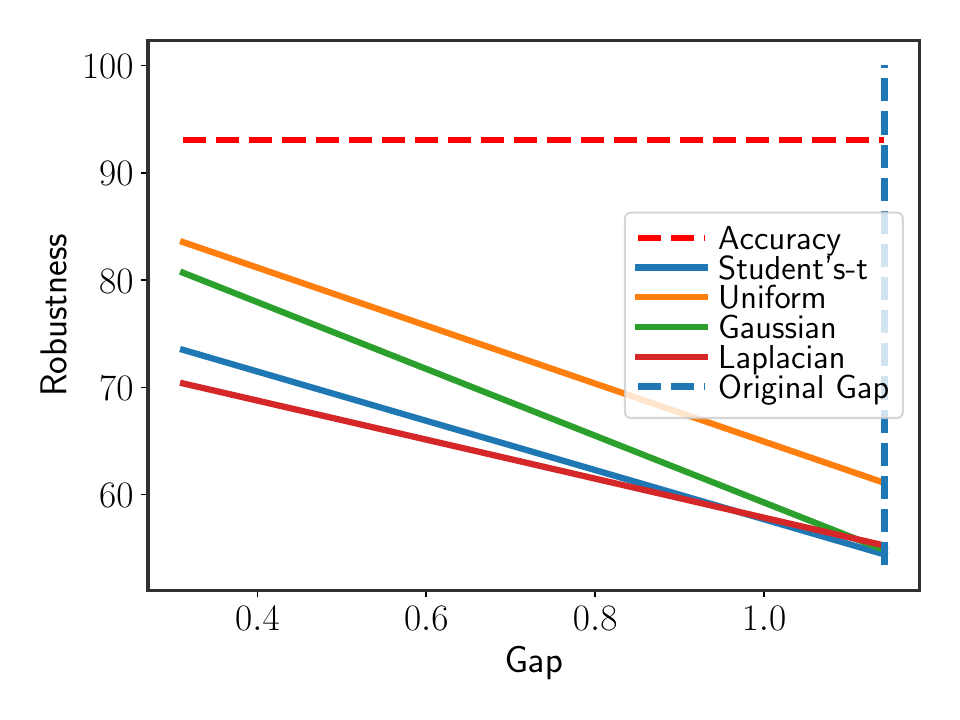} &
        \includegraphics[width=0.32\linewidth]{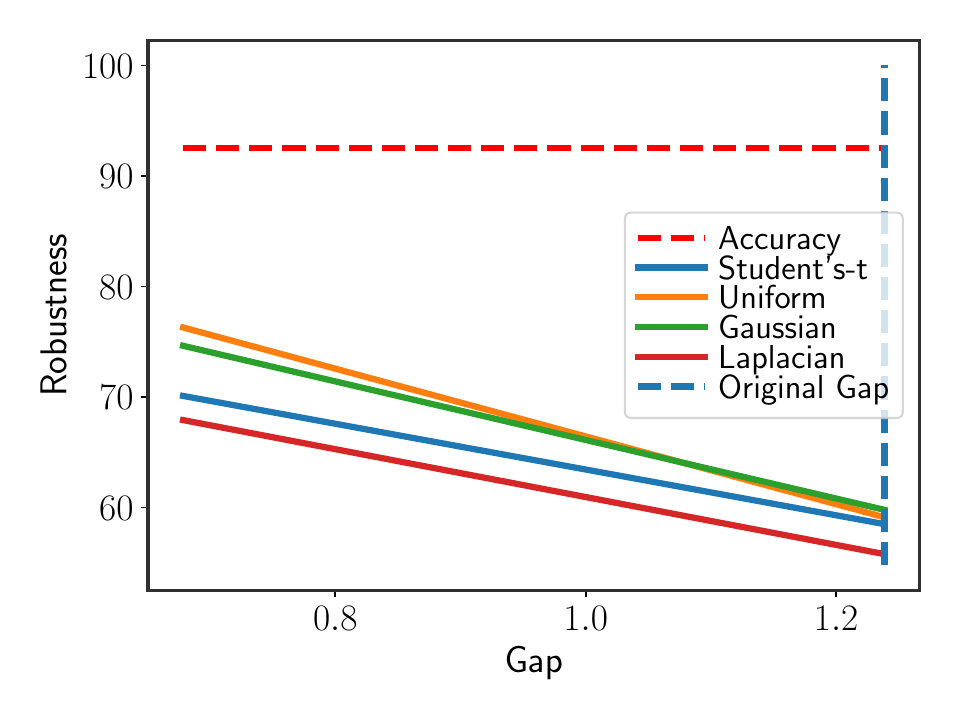} &
        \includegraphics[width=0.32\linewidth]{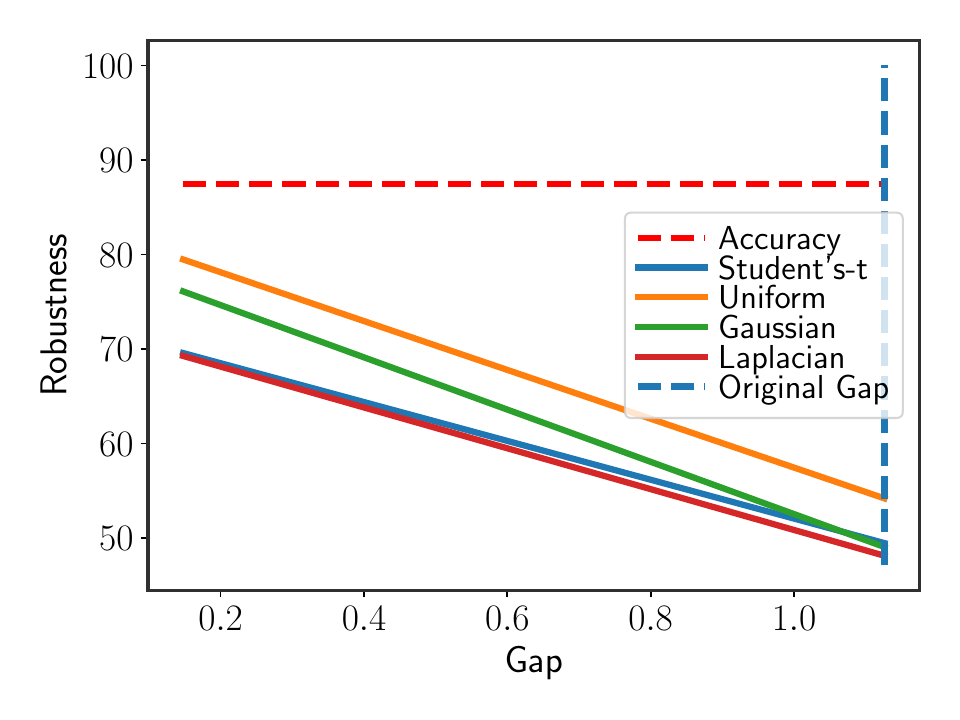} \\[-0.3em]

        \small (a) CLIP ViT-L-14 &
        \small (b) SigLIP &
        \small (c) CLIP ViT-B-16
    \end{tabular}

    \caption{Results for different noise distributions and models on CIFAR10. All are normalized to have variance $\sigma^2 = 0.025^2$. Robustness increases when the gap is closed.}
    \label{fig:noise_models}
\end{figure*}
Here we expand on the results presented in \cref{section:gaussian_experiments}. As \cref{theorem:robust} states, robustness should increase for any distribution of noise which has no correlation between the different dimensions and zero mean, not necessarily Gaussian noise. \Cref{fig:noise_models} expands the experiments in \cref{fig:robustness} to non-Gaussian distributions with zero mean and where all dimensions are sampled i.i.d. . As can be seen, as long as the noise's dimensions are uncorrelated, robustness increases when the gap is closed.

\section{Rephrasing Experiments - Further Details} \label{appendix:rephrasing}
This section provides details on the experiments presented in \cref{subsection:rephrasing}.
\subsection{Rephrasing Noise is Correlated}
Caption rephrasing, such as replacing the caption "a photo of a X" to "an image of X", tends to result in correlated noise in the embedding space. Intuitively, this is because in the input space, rephrasing can be seen as subtraction of a constant input - the string "a photo of a" followed by addition of the constant string "an image of". When these changes are applied to all inputs, it can be expected that the change in the embedding space would either not exist (if the model is completely robust to such changes) or be consistent.

To show this, we conduct the above experiment on Imagenet. Assume all $N$ class names (texts) are embedded with the prefix "a photo of a", resulting in the text embedding matrix $X$. We create a noise text embedding matrix, $\tilde{X}$ by embedding all class names with a different prefix (e.g. "an image of"). We repeat for 400 different caption templates, and calculate the empirical noise covariance. Define the noise to be $M = \tilde{X} - X$. The covariance $C$ is then:
\begin{equation}
    C = (M - \frac{1}{N} \vec{1}\vec{1}^T M)^T  (M - \frac{1}{N} \vec{1}\vec{1}^T M)
\end{equation}
To measure how much the noise is correlated, we simply measure the norm of all off-diagonal elements relative to the forbenius norm of the covariance matrix:
\begin{equation} \label{noisecor}
    d(C) = \frac{\norm{C - \text{diag}(C)}_F}{\norm{C}_F}
\end{equation}
This is of course equal to zero in the case that the noise is uncorrelated and $d(C)=1$ when completely correlated. \cref{fig:appendix:noisecov} measures $d(C)$ for different caption rephrasings for different models on Imagenet. As can be seen, in all cases $d(C)\approx 1$ meaning the noise is indeed extremely correlated. 

\begin{figure}[h]
   \begin{tabular}{cc}
        \includegraphics[width=0.45\linewidth]{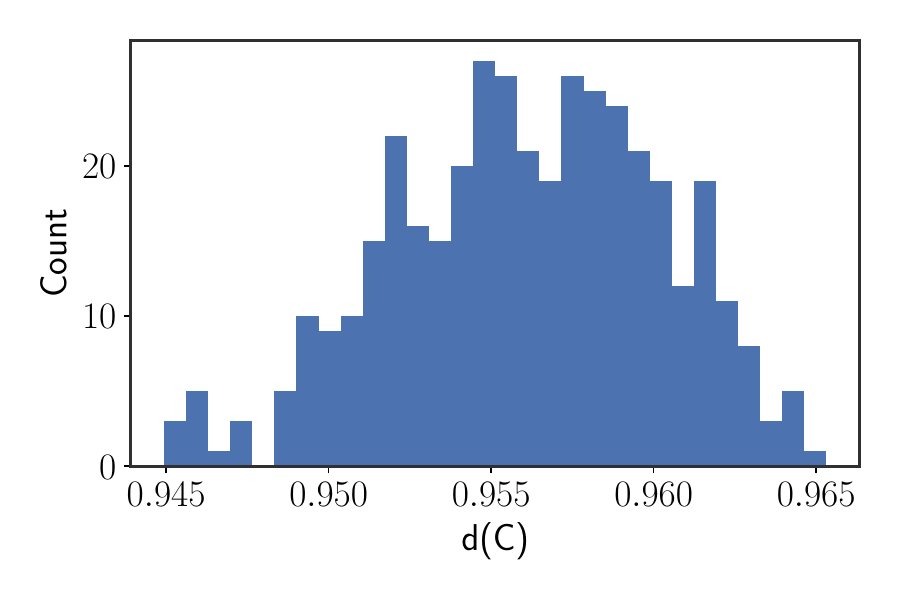} &
        \includegraphics[width=0.45\linewidth]{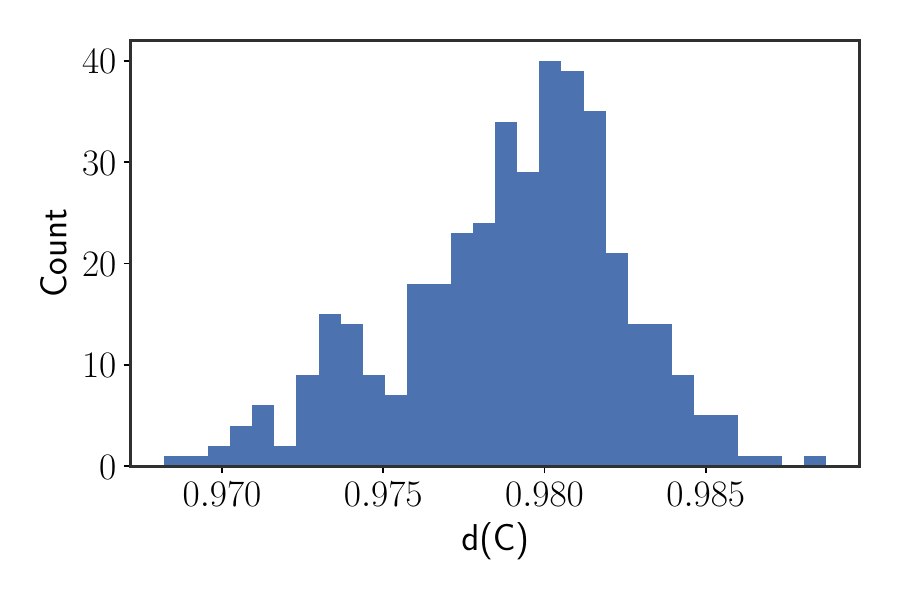}\\[-0.3em]
        \small (a) CLIP (ViT B-16) on Imagenet &
        \small (b) SigLIP on Imagenet  \\[-0.3em]
         
    \end{tabular}
  \caption[]{We compute $d(C)$ according to \cref{noisecor}. For all models we tested, with over 400 different captions, $d(C)\approx1$ suggesting the noise is extremely correlated in the embedding space. }
    \label{fig:appendix:noisecov}
\end{figure}

\subsection{Rephrasing A-OKVQA}
In the multiple choice VQA setting, we follow the protocol of \citet{clipvqa}: given a question $Q$, for each possible answer $A_i$ for $i\in \sqrbr{N_A}$ we embed the text that is the concatenation of $Q+A_i$ resulting in $N_A$ text embeddings. We chose the estimated answer via cosine similarity with the image embedding. We evaluate on the entire A-OKVQA validation set.

In order to rephrase a correct answer, we simply swap the correct answer $A_j$ with each of the "direct answers" options for that particular question in the dataset. Each question has up to 10 different possible correct answers, each one constitutes as a rephrasing. To replace the wrong answers we sample from the entire answer bank (not including the correct answers). 

Notice that in this case our measure of robustness isn't meaningful as the "noisy" wrong answers are completely different texts, therefore we shouldn't expect the model to consistently predict the same wrong answer when adding this type of noise. Therefore, under this setting we focus on consistency w.r.t. right answers, which is similar to measuring accuracy when adding noise. If accuracy increases when closing the gap, the model is more consistent on predicting the right answer, or in other words - robust w.r.t. that answer.

\section{Approximately Orthogonal Algorithm}\label{appendix:approx_algo}

As \cref{theorem:robust} proves, improvement of robustness is correlated with the amount of the gap that is closed, and is maximized when the two modality means overlap. In practice, the gap vector $g$ can almost be non-orthogonal to the subspace of $\X$, meaning that closing the gap in the direction of $g'=g - VV^Tg$ will produce a small increase in robustness since $\norm{g'} \ll \norm{g}$.

To solve this we rely on \cref{theorem:robust} that states that any direction which decreases the distance will increase robustness. We propose closing the gap in directions which are decreasingly orthogonal to the affine subspace of the modality being moved. Following from \cref{theorem:loss_invariance}, this procedure assures us that closing of the gap would result in minimal change to the clean nearest neighbor retrieval. We implement this idea by simply thresholding the number of components to which the gap vector is orthogonal, starting with those containing minimal variance. \Cref{alg:projection} presents a simple python pseudocode implementation.

\begin{algorithm}
\caption{Approximately Orthogonal Gap Vector}
\begin{algorithmic}[1]
\State \Comment{$\epsilon$ - variance threshold, $\vec{g}$ - gap vector, $X$ - modality to move}
\State $VSV^T \gets \texttt{PCA}(X)$
\State $\text{inds} \gets \{i : S_i > \epsilon\}$
\State $V' \gets V[:, \text{inds}]$
\State \Return $(I - V' V'^T) g$
\end{algorithmic}
\label{alg:projection}
\end{algorithm}

This algorithm produces a tradeoff between the increment in robustness and loss of accuracy controlled by the parameter $\epsilon$. An example of using the algorithm is displayed in \cref{fig:appendix:approx} and the tradeoff induces can be seen in \cref{fig:consistency_accuracy}.

\section{Simulations}\label{appendix:simulations}
\subsection{Details}
All simulations are done by optimizing randomly initialized embedding vectors using the \cref{mmloss}. We use full batch gradient descent with a learning rate of 0.01. 

In \cref{fig:theory2} (top) we train unnormalized embedding vectors, directly optimizing \cref{mmloss} without a normalization step. We initialize the embeddings sampled from two normal distributions $\mathcal{N}(\mu_{1,2}, 0.01^2 \cdot I)$ with $\mu_{1,2}=(0, \pm 0.5)$. We use a constant temperature of $\tau=0.1$ and train for $10^7$ iterations.

In \cref{fig:theory2} (bottom) and \cref{fig:training} we initialize 1000 embeddings per modality drawn from a normal distribution $\mathcal{N}(\pm\vec{e}_1, 0.01^2 \cdot I)$ with $\vec{e}_1$ being the first elementary basis vector. All embeddings are normalized as is done in training multi-modal contrastive models \cite{clip}. 

\subsection{Further Experiments}
 \begin{figure}[h]
     \centering
     \begin{tabular}{cc}
        \includegraphics[width=0.45\linewidth]{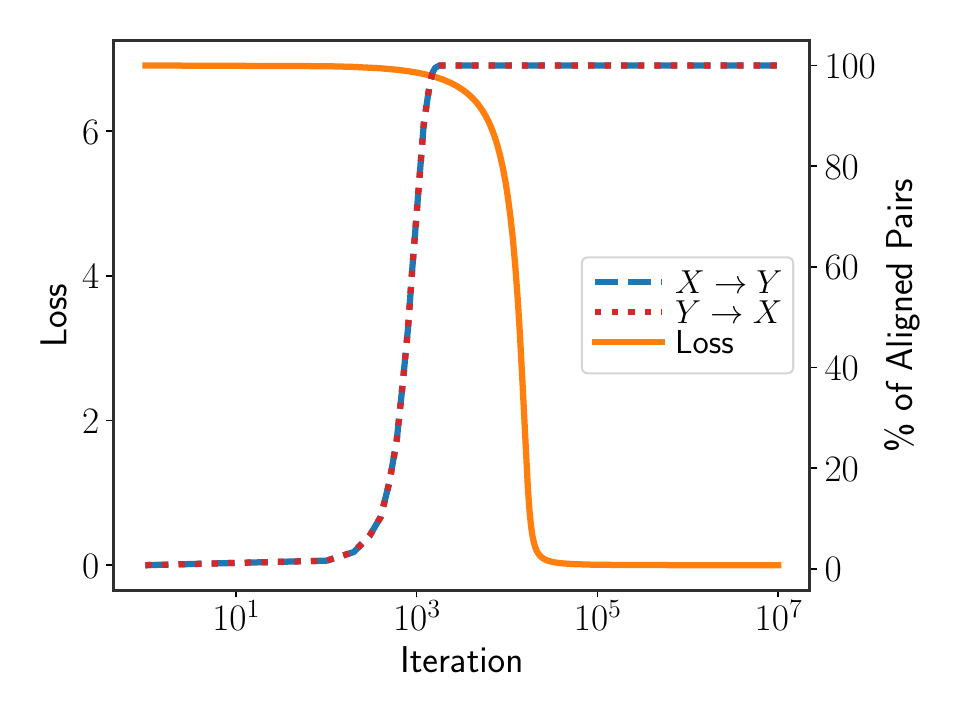}  &
        \includegraphics[width=0.45\linewidth]{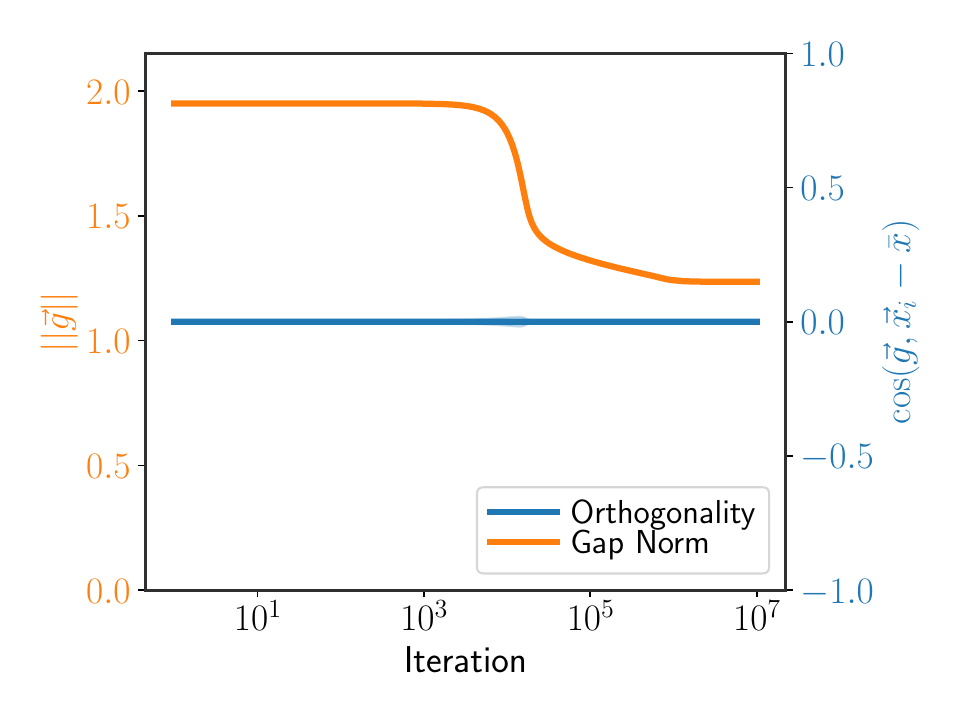} 
     \end{tabular}
    
  \caption[]{We follow \citet{shi2023towards} and learn embeddings $\X,\Y\in \R^{1000\times 512}$ on the hypersphere using the contrastive loss (\cref{mmloss}) and gradient descent. While the loss decreases and training converges (left), a major gap remains between the embeddings and orthogonality holds (right, measured using \cref{eq:ortho}). }\label{fig:training}
\end{figure}

We ablate both the dimension, initial distance between modality means and temperature. As stated in the main paper (and shown in \cite{shi2023towards} and \cite{trigger}), there exist cases in which training converges to completely aligned modalities. This can happen when double stochasticity does not hold throughout the iterations, and double stochasticity may depend subtly on the temperature $\tau$ and the number of iterations. We demonstrate this in \cref{fig:further_exs}.

\begin{figure*}[b!]
    \centering
    \begin{tabular}{cc}
        \includegraphics[width=0.45\linewidth]{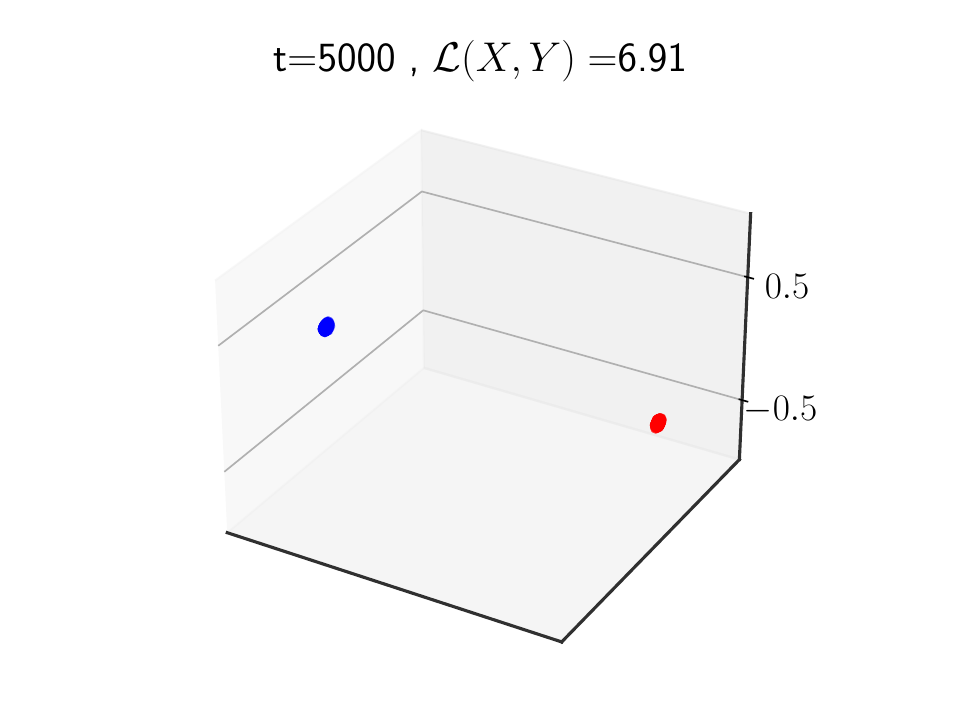} &
        \includegraphics[width=0.45\linewidth]{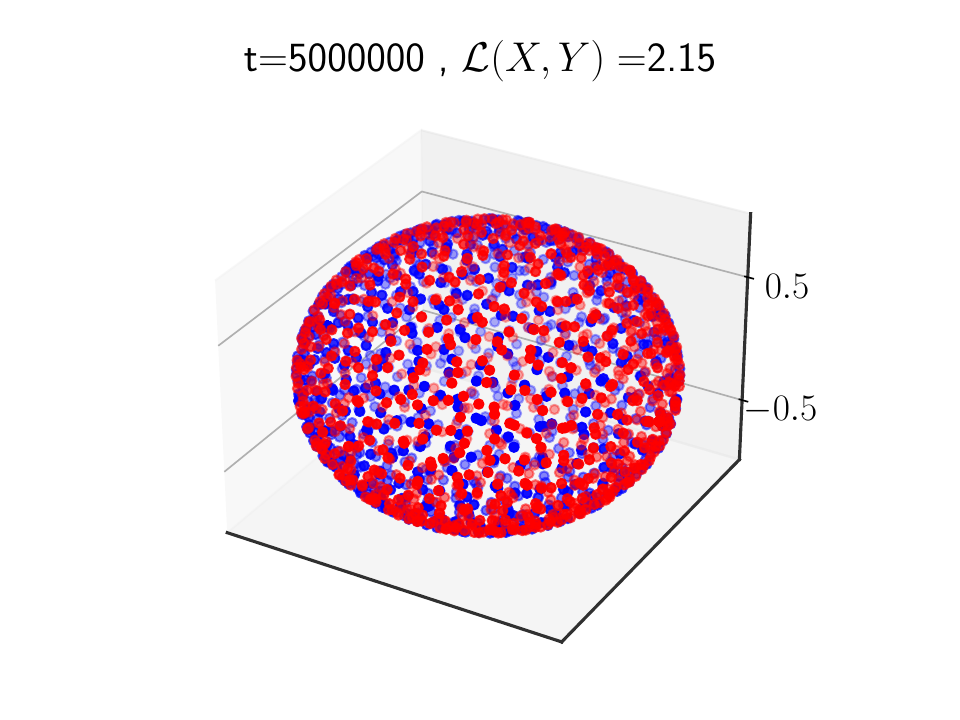} \\[-0.3em]
        \includegraphics[width=0.45\linewidth]{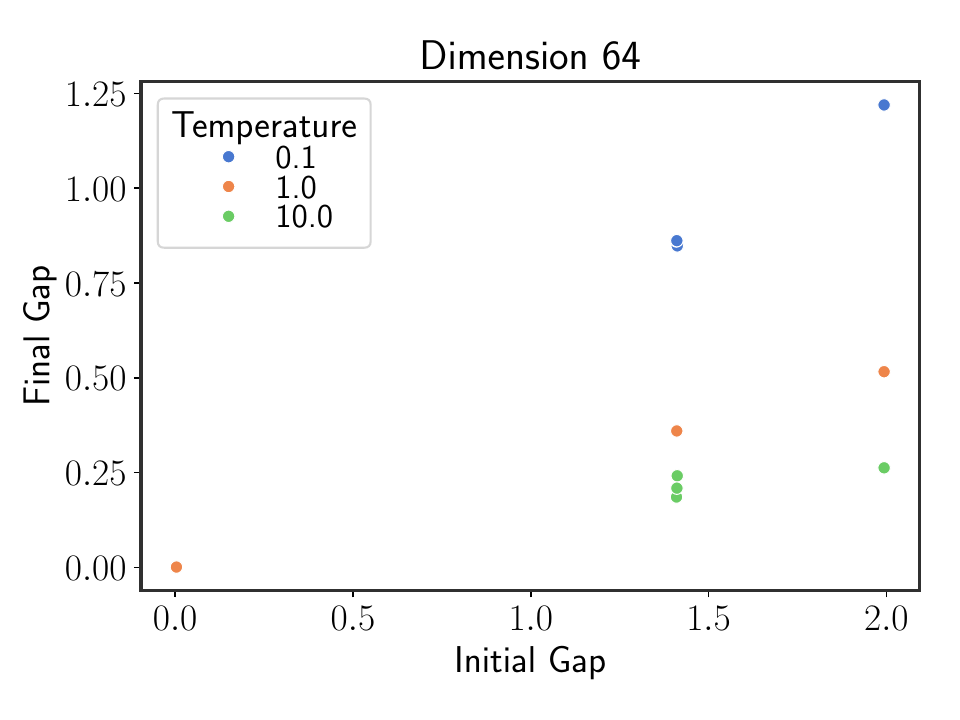} &
        \includegraphics[width=0.45\linewidth]{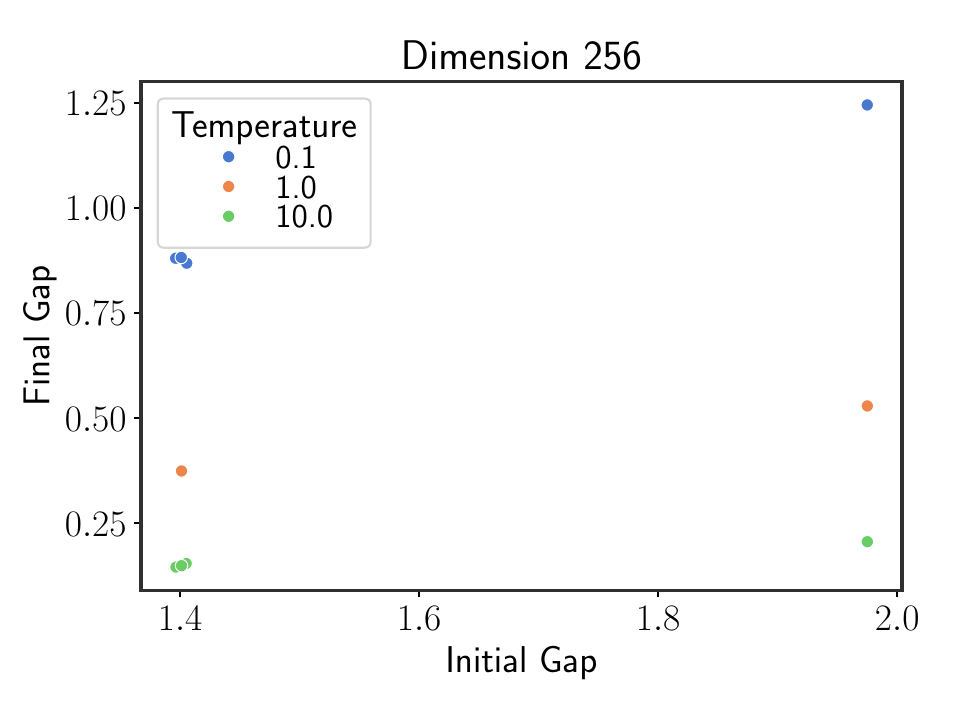} \\[-0.3em]
        \includegraphics[width=0.45\linewidth]{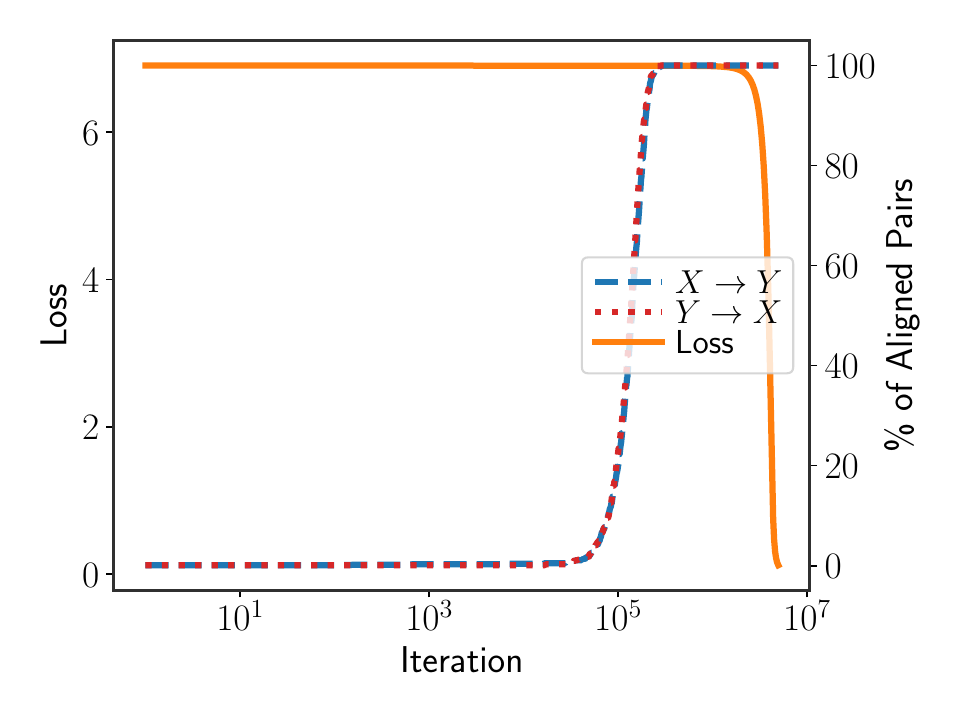} &
        \includegraphics[width=0.45\linewidth]{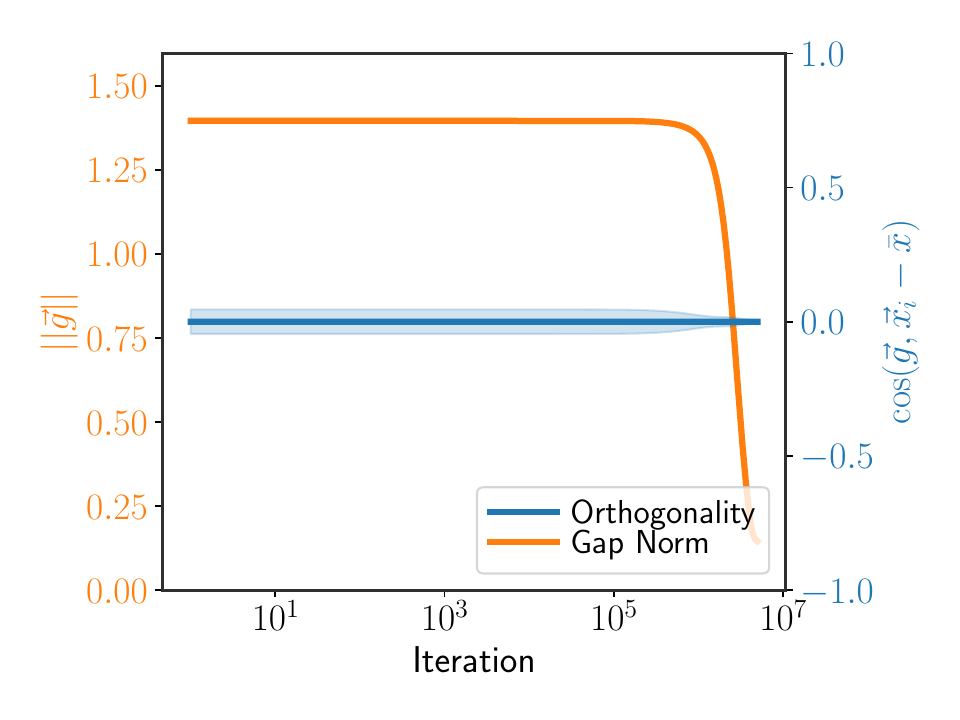}
    \end{tabular}

    \caption{Top: When training with $\tau=1$ training converges to a solution without a gap, despite existence of an initial gap and orthogonality.\\
    Middle: This is consistent for training in higher dimensions as well - different temperatures have different effects on how much of the gap is closed. When temperatures are $\geq 1$, the gap closes throughout training. In higher temperatures training hardly differs from initialization. When the gap is initialized at zero it remains so for all temperatures. 
    Bottom: Example of dynamics in $\R^{256}$ with $\tau=10$. The gap closes throughout training as it converges to perfect alignment. }
    \label{fig:further_exs}
\end{figure*}

\newpage

\begin{figure*}
   \begin{tabular}{ccc}
        \includegraphics[width=0.32\linewidth]{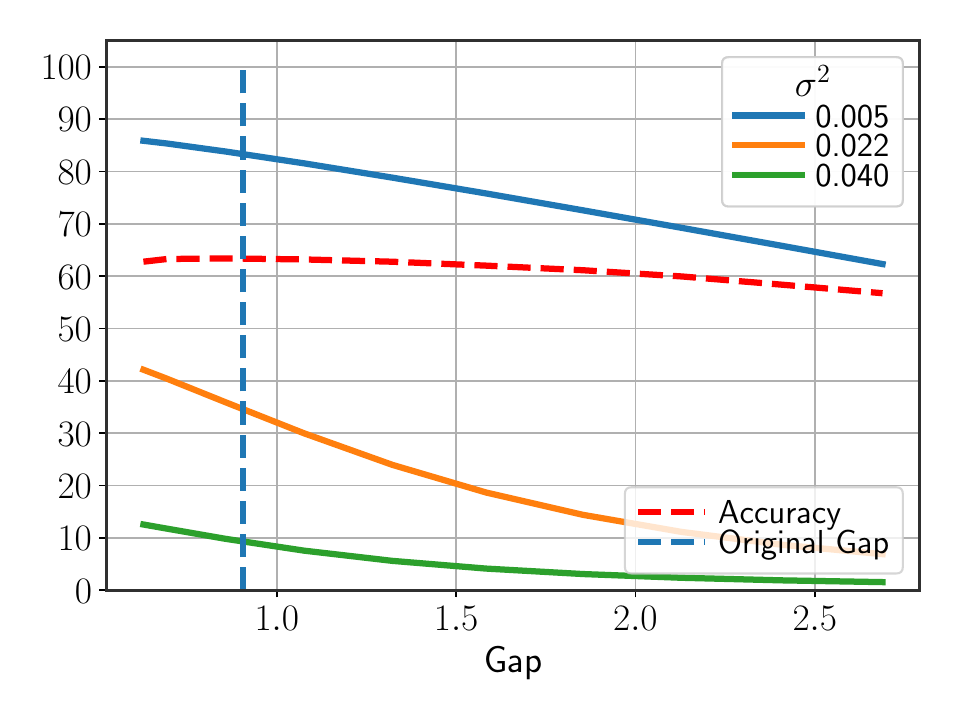} &
        \includegraphics[width=0.32\linewidth]{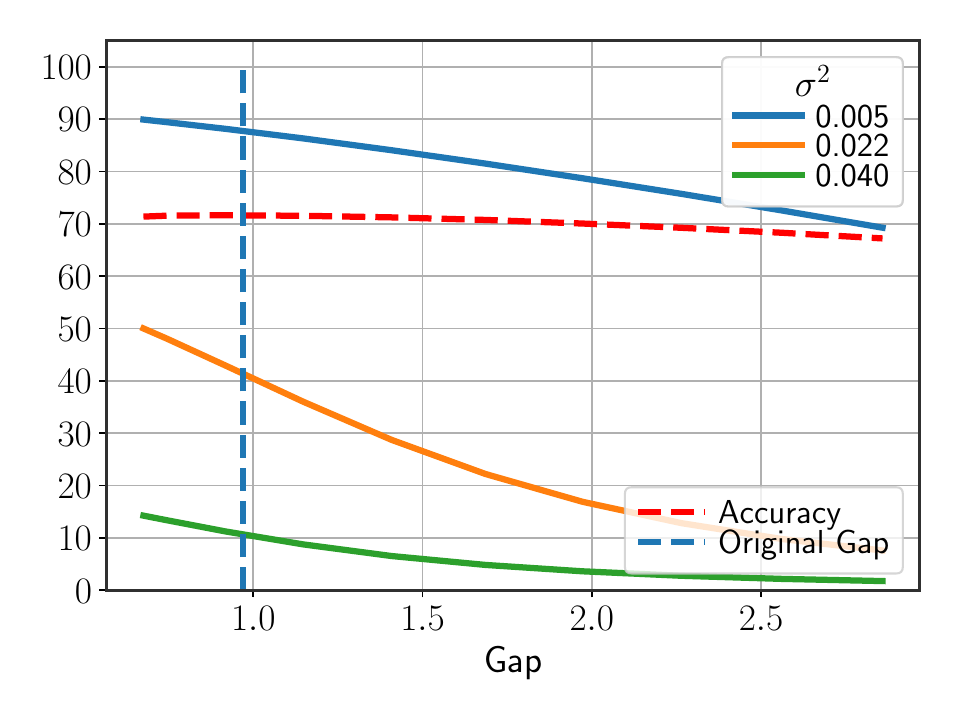} &
        \includegraphics[width=0.32\linewidth]{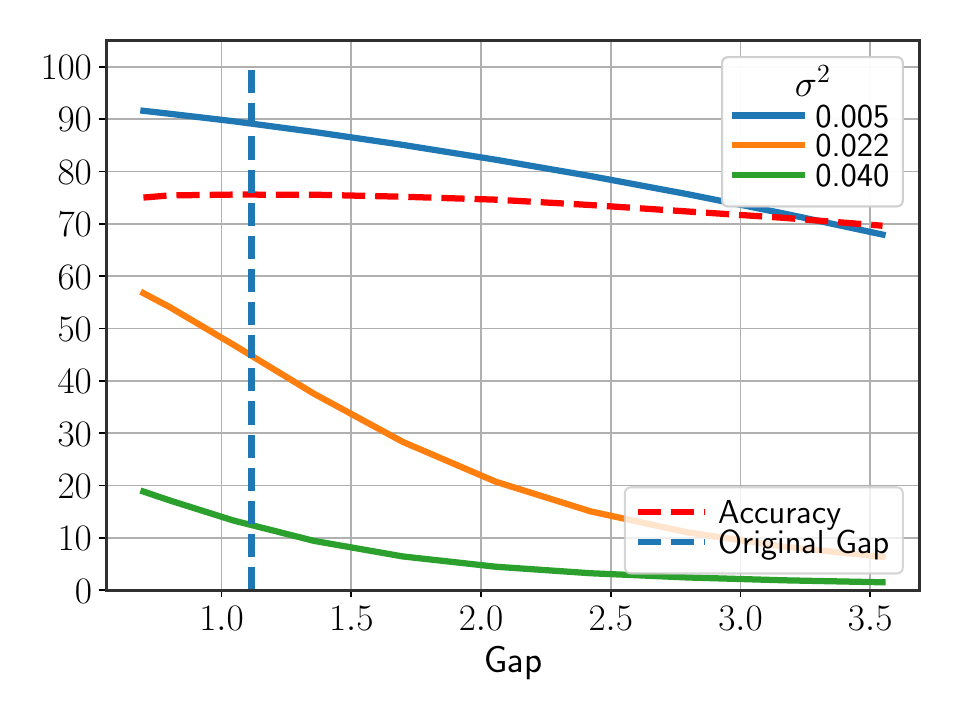}\\[-0.3em]
        \small (a) CLIP (ViT B-16) on Imagenet &
        \small (b) CLIP (ViT L-14) on Imagenet &
        \small (c) SigLIP on Imagenet  \\[-0.3em]
         
        \includegraphics[width=0.32\linewidth]{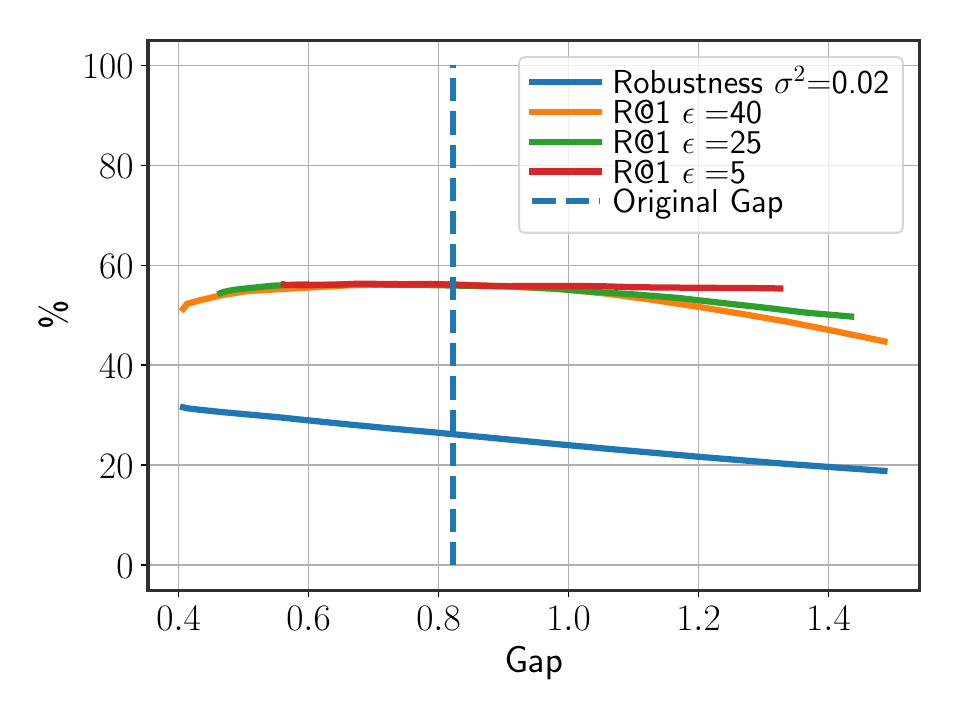} &
        \includegraphics[width=0.32\linewidth]{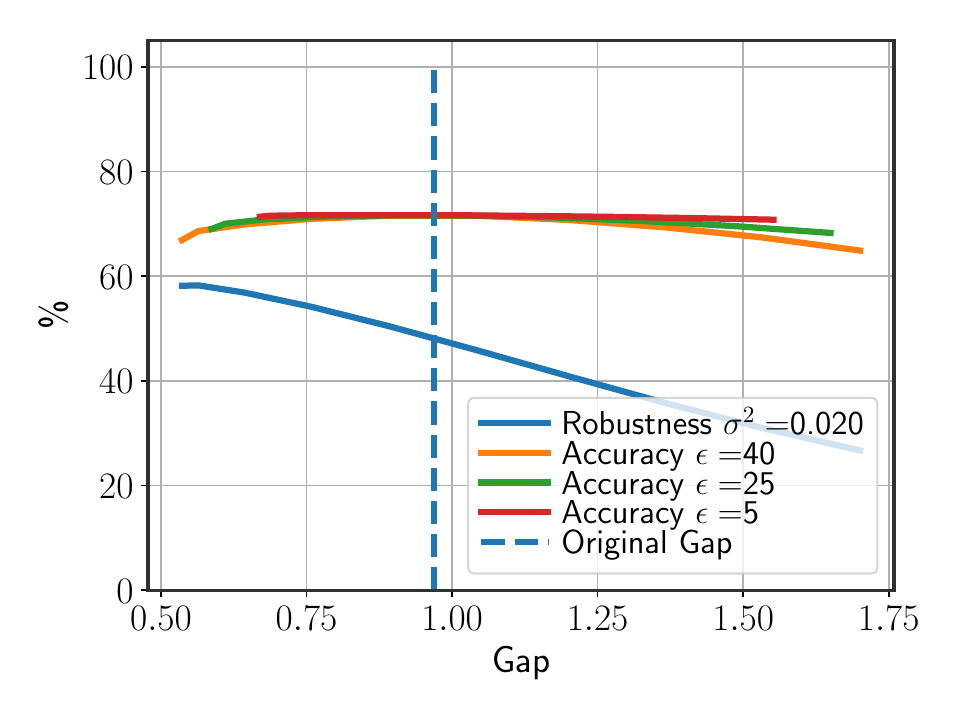} &
        \includegraphics[width=0.32\linewidth]{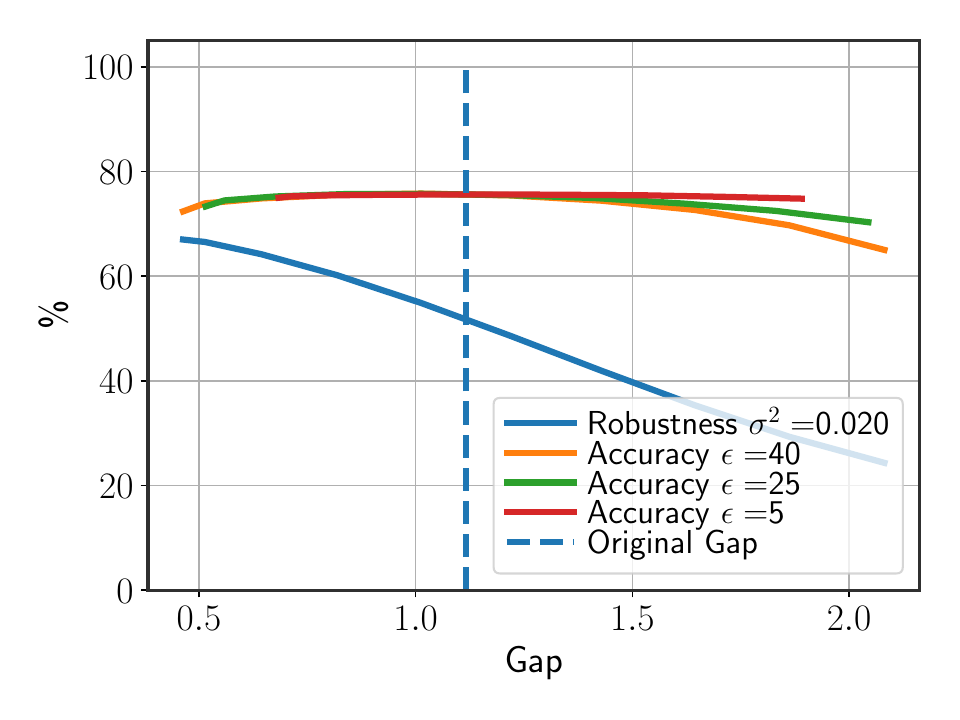} \\[-0.3em]
        \small (d) CLIP (ViT L-14), MS-COCO  &
        \small (e) CLIP (ViT L-14), ImageNet &
        \small (f) SigLIP, ImageNet  \\[-0.3em]
    \end{tabular}
  \caption[]{The zero shot classification accuracy and robustness under noise $\eta\sim \mathcal{N}(0,\sigma^2 I)$, on ImageNet and MS-COCO. \\
  Top row: When using \cref{alg:projection} with $\epsilon=5\%$ of the variance, decrease in accuracy is negligible ($<1\%$) relative to the robustness gained, which can be $\sim 10\%$. \\
  Bottom row: As the threshold $\epsilon$ grows larger, more of the gap is closed, greatly increasing robustness at a negligible cost of accuracy. 
  }
    \label{fig:appendix:approx}

\end{figure*}

\begin{figure*}
   \begin{tabular}{cc}
        {\label{fig:clipmscoco}\includegraphics[width=0.48\linewidth]{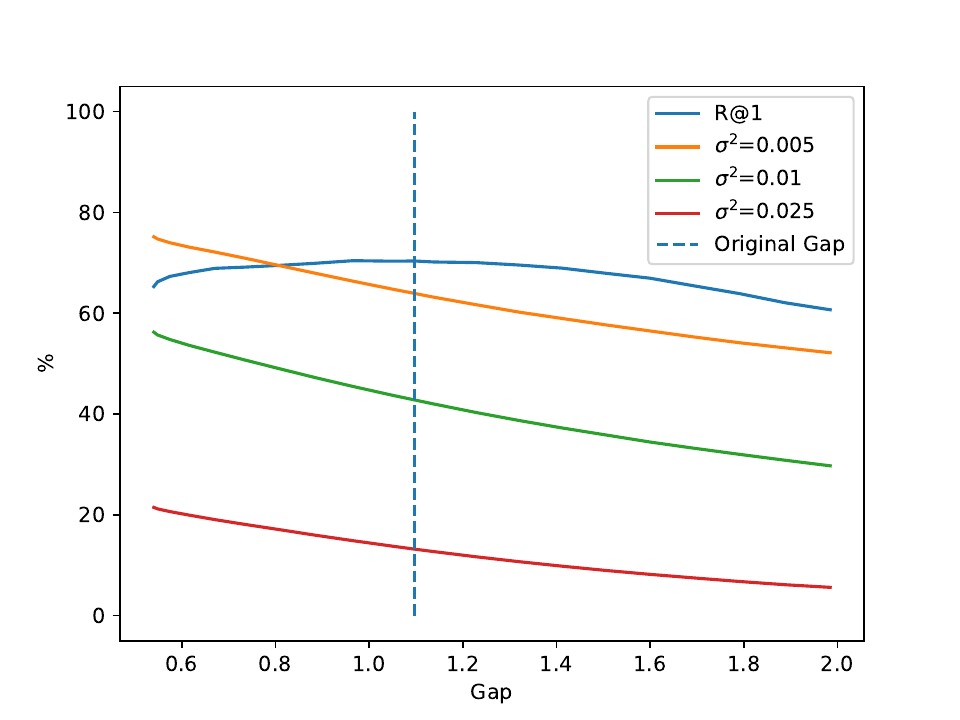} } &  {{\label{fig:cliptradeoff}\includegraphics[width=0.48\linewidth]{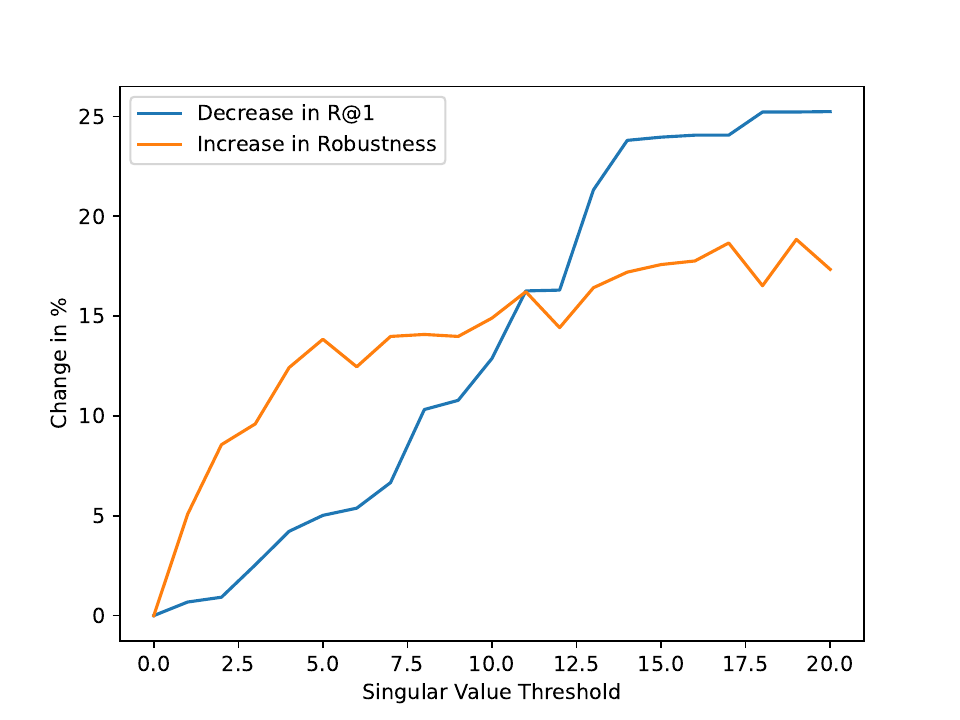} }}
        \\
        \small (a) SigLIP on MS-COCO  &
        \small (b) Tradeoff for $\sigma^2=0.01$ [-0.3em]
    \end{tabular}
   \caption{Even when using \cref{alg:projection} the drop in R@1 for SigLIP \citep{siglip} on image to text retrieval on MS-COCO dataset \citep{mscoco} is negligible relative to the improvement in robustness for different Gaussian noises (left). \cref{fig:cliptradeoff} shows the ranges of the singular value threshold $\epsilon$ for which the increment in robustness (for Gaussian noise with $\sigma^2 = 0.01$) is larger than the decrease in R@1.}
 \label{fig:consistency_accuracy}

\end{figure*}

\end{document}